\PassOptionsToPackage{hypertexnames=false}{hyperref}
\documentclass[a4paper,fleqn]{cas-sc}

\usepackage[authoryear,longnamesfirst]{natbib}

\usepackage{amsmath,amssymb,amsthm}
\usepackage{algorithm}
\usepackage{algpseudocode}
\usepackage{booktabs}
\usepackage{multirow}
\usepackage{array}
\usepackage{colortbl}
\usepackage{xcolor}
\usepackage{threeparttable}
\makeatletter
\def\TPT@opt@flushleft{%
  \def\TPTnoteSettings{\labelsep\z@ \leftmargin\z@ \labelwidth\z@ \itemsep1pt \topsep3pt}%
  \def\TPTnoteLabel##1{\if\relax\detokenize{##1}\relax\else\tnote{##1}\hspace{.2em}\fi}%
  \rightskip\z@skip \leftskip\z@skip}
\makeatother
\usepackage{siunitx}
\usepackage{graphicx}
\usepackage{tikz}

\newtheorem{theorem}{Theorem}[section]
\newtheorem{proposition}[theorem]{Proposition}
\newtheorem{lemma}[theorem]{Lemma}
\newtheorem{corollary}[theorem]{Corollary}
\theoremstyle{definition}
\newtheorem{definition}[theorem]{Definition}
\newtheorem{assumption}[theorem]{Assumption}

\theoremstyle{remark}
\newtheorem{remark}[theorem]{Remark}

\newcommand{\tagadapted}{\textnormal{\textit{adapted}}}
\newcommand{\tagknown}{\textnormal{\textit{established}}}
\newcommand{\tagspecialized}{\textnormal{\textit{specialized}}}
\newcommand{\tagderived}{\textnormal{\textit{derived}}}

\newcommand{\Y}{\mathcal{Y}}
\newcommand{\X}{\mathcal{X}}

\newcommand{\Tmap}{\mathcal{T}}
\newcommand{\Ical}{\mathcal{I}_{\mathrm{cal}}}
\newcommand{\Itr}{\mathcal{I}_{\mathrm{tr}}}
\newcommand{\yt}{\tilde{Y}}
\newcommand{\Ct}{\tilde{C}}
\newcommand{\Cop}{\tilde{C}^{\oplus}}
\newcommand{\dtv}{d_{\mathrm{TV}}}
\newcommand{\E}{\mathbb{E}}
\newcommand{\Prob}{\mathbb{P}}
\newcommand{\ind}{\mathbf{1}}

\definecolor{seqteal1}{HTML}{e6eeec}
\definecolor{seqteal2}{HTML}{a9d3cc}
\definecolor{seqteal3}{HTML}{5fb0a5}
\definecolor{seqteal4}{HTML}{11746c}
\definecolor{choircoral}{HTML}{c1443c}
\definecolor{choirink}{HTML}{1a1a1a}
\definecolor{choirslate}{HTML}{697784}

\newlength{\dbwidth}
\newcommand{\databar}[1]{%
  \makebox[0pt][l]{\textcolor{seqteal2}{\rule[-0.15ex]{#1\dbwidth}{1.15ex}}}}

\newcommand{\RUNTIMEONEM}{0.21\,s}
\newcommand{\RUNTIMEFIVEM}{1.09\,s}

\begin{document}
\let\WriteBookmarks\relax
\setcounter{topnumber}{1}
\setcounter{bottomnumber}{1}
\setcounter{totalnumber}{2}
\renewcommand{\topfraction}{.88}
\renewcommand{\bottomfraction}{.65}
\renewcommand{\textfraction}{.12}
\renewcommand{\floatpagefraction}{.82}
\setlength{\textfloatsep}{10pt plus 2pt minus 2pt}
\setlength{\floatsep}{8pt plus 2pt minus 2pt}

\shorttitle{CHOIR: conformal prediction for crash injury severity}
\shortauthors{A. Rafe and S. Das}

\title[mode = title]{CHOIR: heterogeneity-aware conformal prediction for crash injury severity across driver safety strata}

\author[1]{Amir Rafe}[orcid=0000-0002-4089-2088]
\cormark[1]
\ead{amir.rafe@txstate.edu}
\credit{Conceptualization, Methodology, Software, Formal analysis, Writing -- original draft, Visualization}

\author[1]{Subasish Das}[orcid=0000-0002-1671-2753]
\ead{subasish@txstate.edu}
\credit{Conceptualization, Methodology, Supervision, Writing -- review \& editing}

\affiliation[1]{organization={Civil Engineering, Texas State University},
    addressline={601 University Drive},
    city={San Marcos},
    postcode={78666 TX},
    country={USA}}

\cortext[1]{Corresponding author}

\begin{abstract}
Transportation agencies increasingly predict crash-injury severity with statistical and
machine-learning models, but these models do not state how often their output contains the
recorded injury level or for which groups of drivers it fails, a gap that matters most for
motorcyclists and unrestrained drivers. This study develops and evaluates a certification
layer that gives any fitted severity model a finite-sample, distribution-free coverage
guarantee within prespecified safety strata. The layer, CHOIR (Conformal
Heterogeneity-aware Ordinal Inference with Risk control), combines groupwise and weighted
conformal prediction with conformal risk control to return contiguous KABCO intervals, and
adds a declared sensitivity analysis for medically assessed injury and bounds on fatal
omission. It is evaluated on 4.04 million Texas crashes from 2017--2023, one sampled driver
per crash, with seven base models from the ordered logit to a tabular foundation model, and
on held-out counties and later years. Under one pooled threshold every model reaches 0.90
coverage overall but covers motorcyclists or unrestrained drivers at 0.868 or lower, and
class-balanced gradient boosting covers unrestrained drivers at only 0.374. Calibration
within four safety strata places all 28 model-by-stratum estimates between 0.898 and 0.907,
at the cost of sets spanning 3.6 to 4.6 of five categories for these groups, and the
certified ordered logit is within 0.03 categories of the narrowest model. Injury-model
coverage should therefore be certified within safety groups rather than on average,
calibration rather than model complexity determines validity, and a statewide threshold
should not be applied to small rural counties without local calibration data.
\end{abstract}


\maketitle

\section{Introduction}
\label{sec:intro}

Transportation agencies and researchers now predict crash-injury severity from police crash records with a wide range of models, from the ordered logit to gradient boosting, deep networks, and tabular foundation models. These models return a predicted KABCO category or a probability vector. They do not state how often a set of predicted severities will contain the recorded outcome, and they do not show for which groups of drivers that rate breaks down. The question matters most for the groups with the highest injury risk, such as motorcyclists and unrestrained drivers, because these groups are a small share of the records and can be poorly served by a model that looks accurate on average.

The application in this paper is retrospective and record based. A model trained on completed crash records assigns each sampled driver a set of plausible police-reported KABCO levels. An analyst can use the set to audit a severity model before it enters a safety program, for example to check whether the model's uncertainty is trustworthy for motorcyclists as well as for the statewide population, or to flag records for data-quality review. For these uses the analyst needs a statement of the form ``the set contains the recorded level at least 90\% of the time within this safety group,'' together with the price of that statement in set width. This study provides such statements for sampled-driver KABCO injury in the Texas Crash Records Information System (CRIS). The guarantee covers random holdouts of completed records. Use at the time of a crash report, or in other counties and years, falls outside it and is examined here only through descriptive stress tests.

Police-reported KABCO injury has four features that shape the required uncertainty method. The outcome is ordinal from no injury to fatal injury. The report is a structured field assessment that can differ from a medical assessment, often through disagreement between adjacent categories. Crash populations are heterogeneous, and the most safety-relevant groups can have score distributions that differ from the statewide mixture. The record distribution also changes across jurisdictions and years. A useful certificate must preserve the ordering, report performance within prespecified groups, account for a declared reporting relation, and keep observed-cell validity separate from transfer diagnostics.

Conformal prediction supplies finite-sample marginal coverage under exchangeability and
can wrap any fitted base model \citep{vovk2005,lei2018}. Groupwise calibration, weighted
conformal prediction, and conformal risk control are also established tools
\citep{tibshirani2019,angelopoulos2024}. Used directly, these tools do not give contiguous
ordinal sets, a statement about the medical injury behind a police-reported code, or an
audit that separates observed-cell validity from transfer evidence. The crash-injury
problem requires the established components to be connected through a common estimand,
partition, reporting structure, and audit design.

This study addresses three research questions. The first (RQ1) asks whether a single
certification layer can give any fitted crash-injury severity model finite-sample coverage of
the recorded KABCO level within prespecified driver safety strata, and what that guarantee
costs in set width. The second (RQ2) asks what can be stated about the medically referenced
injury behind a police-reported code when only a declared reporting relation is available,
and how much of the width of a valid set a declared reporting process alone forces. The third
(RQ3) asks how a certificate calibrated on one population behaves when it is applied to other
counties and later years, and how omissions of severe and fatal outcomes can be bounded when
the sets are used for record review.

To answer these questions, this paper develops CHOIR (Conformal Heterogeneity-aware
Ordinal Inference with Risk control), a certification layer for ordinal crash-injury
prediction.
CHOIR converts the fitted class probabilities of any base model into contiguous KABCO
intervals and calibrates them within prespecified safety strata that are fixed on the
training data before any calibration outcome is used. A declared compatibility map expands
reported-label intervals into sensitivity sets for the underlying medical injury. A
transfer component combines density-ratio weighting with a one-sided mismatch diagnostic
for new counties and years. A risk component calibrates severity-weighted and fatal
omission. The framework records each source of slack so that coverage, sensitivity,
transfer, and risk statements remain distinct.

The paper makes four contributions. First, it assembles a model-agnostic certification
layer for KABCO. An ordinal CDF score in the style of \citet{lu2022} gives contiguous
intervals, and Mondrian calibration within crash strata that are frozen on training data
gives finite-sample coverage in each stratum. Theorem~\ref{thm:oracle} specializes the known
efficiency argument to this setting and characterizes when a single pooled threshold
under-covers a stratum. Second, it transfers coverage from police-reported to medically
referenced injury under a declared ordinal band. The guarantee loses the declared
beyond-band mass $\delta$ and costs an additive $b^-+b^+$ categories, and this bound is sharp
(Theorems~\ref{thm:noise} and~\ref{thm:sharp}). Without the band, the best available bound
degrades with the total reporting error rate, which linkage studies put near one half.
Third, under a declared reporting channel, it derives a lower bound on the width of any
valid predictor (Theorem~\ref{thm:channel}), and it shows that the distribution-free width
floor cannot be certified from below (Theorem~\ref{thm:noncert}). Together these results
bound how much width a declared reporting process alone forces, conditional on that
declaration. Fourth, it adds bounds on severity-weighted and joint fatal omission and an
accounting rule (Theorem~\ref{thm:compose}) that keeps every issued statement tied to its
assumptions. The layer is released as the open-source \texttt{choircert} package.

The framework is evaluated on 5.2 million CRIS driver records from 2017--2025, one per
crash, with a primary certification sample of 4.04 million crashes from 2017--2023. Seven base
models are wrapped by the same layer, namely the ordered logit, multinomial logit,
latent-class and random-parameters ordered logits, gradient boosting, a deep ordinal network,
and the TabPFN tabular foundation model. The evaluation compares pooled and stratum-level
calibration across four prespecified driver safety strata (RQ1), measures variation across
repeated calibration splits, and compares the layer with other conformal methods. It
examines the declared reporting relation through semi-synthetic stress tests and a
reporting-channel sensitivity study (RQ2). It applies frozen certificates to held-out counties
and later years, audits severe and fatal omission for record review, and reports the
computational cost of the layer (RQ3).

The remainder of the paper is organized as follows. Section~\ref{sec:lit} positions the work in the crash-severity and uncertainty-quantification literatures. Section~\ref{sec:framework} gives the framework, guarantees, and proofs. Section~\ref{sec:data} describes the Texas CRIS records and the experimental design. Section~\ref{sec:results} reports the application, and Section~\ref{sec:discussion} discusses what the results mean for transportation safety practice. Section~\ref{sec:conclusion} summarizes the contribution, limitations, and future directions.

\section{Related work}
\label{sec:lit}

Three literatures underpin this paper. Injury modeling supplies the problem and the base
models. Uncertainty quantification supplies coverage concepts for safety-critical AI.
Conformal inference supplies the established calibration tools. This paper builds a
crash-injury framework from those tools and identifies the additional statements that
require police-reported KABCO structure.

\subsection{Severity modeling and the measurement of severity}
\label{sec:lit-severity}

Police-reported KABCO injury analysis has been an ordered-response problem since \citet{mccullagh1980},
and the methodological arc since then has been a sustained effort to make the conditional
distribution more flexible. \citet{savolainen2011} survey the alternatives and their
trade-offs. The heterogeneity program associated with Mannering and colleagues is the
central development, and \citet{manneringshankarbhat2016} argue that unobserved heterogeneity
is not a nuisance to be averaged away but the dominant structural feature of crash data,
and \citet{manneringbhat2014} set the methodological frontier that random-parameters and
latent-class specifications were built to address. \citet{mannering2018} adds the temporal
dimension, showing that the relationships themselves are unstable across time. More recently,
machine learning and deep models have expanded the available estimators, and \citet{seyfi2025}
assess their performance relative to established statistical models.

This paper uses the heterogeneity structure established in that literature as an input to
certification. The latent classes from a latent-class ordered logit form the partition used
in Section~\ref{sec:het}, and Section~\ref{sec:seven-model} places a random-parameters ordered logit
inside the same certification layer as the other base models. These estimators target the
conditional severity distribution, while the certification layer supplies a finite-sample
statement for the resulting sampled-driver prediction set.

The second thread in this stream is measurement, and it is the one that motivates
Section~\ref{sec:noise}, because police-reported KABCO injury and medically assessed
injury are different quantities. \citet{burdett2015} and the underlying linkage study \citep{burdett-thesis}
match Wisconsin crash reports to medical records and find that agreement between the
police rating and the medically referenced category is close to half, with the
disagreement concentrated in the under-reporting direction and heavily structured by
category. \citet{burdett2022} show that the discrepancy varies systematically by reporting
agency. This is why this paper declares a band as a sensitivity input rather than
estimating a confusion matrix; the paper does not claim that either a kernel or a band
transfers across jurisdictions. \citet{taylor2024} reach a similar conclusion from linked North
Carolina trauma-registry data, reporting sensitivity near 50\% for police identification
of serious injury. \citet{farmer2003} and \citet{compton2005} provide the national
comparisons, and both report much closer agreement for fatal injury than for the nonfatal
categories. That asymmetry motivates, but does not establish, the exact-K premise that
Remark~\ref{rem:catdep} declares. This literature establishes that the label an injury
model is trained and evaluated on is noisy in a specific, ordered, category-dependent way.
It does not provide a way to state what a prediction means for the underlying injury.

\subsection{Uncertainty quantification for safety-critical AI}
\label{sec:lit-uq}

Recent safety-critical AI research gives greater attention to predictive uncertainty.
\citet{qian2024} survey uncertainty quantification for traffic forecasting across
Bayesian, ensemble, bootstrap, and calibration approaches. Conformal methods have been
used to wrap a graph neural network for travel-time forecasting \citep{patil2024}, to give
coverage-guaranteed intervals for traffic demand \citep{yang2026}, and to apply Mondrian
conformal prediction to mode choice \citep{bohlouli2025}. In crash analysis,
\citet{islam2024} pursue calibrated confidence for real-time crash and severity prediction.

Closest to this paper is \citet{wei2025}, who apply generic and class-conditional
conformal prediction to pre-crash injury risk and obtain a guaranteed confidence level.
That work establishes a direct precedent for coverage guarantees in injury prediction.
Split conformal prediction provides its established guarantee only for the label used in
calibration, under its exchangeability conditions. It does not itself provide a statement
about an unobserved medical injury, an unlabeled new jurisdiction, or a crash-record audit
that combines coverage and omission risk. These boundaries describe the scope of generic
conformal tools. The contribution here is a KABCO-specific framework that states what
additional declared structure is required for each extension.

\subsection{Conformal foundations}
\label{sec:lit-conformal}

The conformal framework originates with \citet{vovk2005}, and split conformal prediction
in the form used here is standard \citep{lei2018}. Its limits are equally well
established, since \citet{vovk2012} and \citet{foygelbarber2021} prove that exact
covariate-conditional coverage is unattainable without trivial sets, which is why every
conditional statement in this paper is made with respect to finitely many pre-declared
groups. Work on conditional guarantees continues along that line, through
\citet{ding2023} for many-class settings, \citet{gibbs2025} for covariate-shift-indexed
classes, \citet{kiyani2024} for learned partitions, and \citet{dunn2022} for hierarchical
structure.

Ordinal conformal prediction is an active subfield with directly relevant constructions. \citet{lu2022} and
\citet{chakraborty2024} construct CDF-based ordinal scores of the kind
Definition~\ref{def:score} uses, \citet{haas2026rps} use the ranked probability score to
produce median-centered contiguous sets, and \citet{xu2023} formulate conformal risk control for
ordinal losses. Set size in the ordinal setting has its own literature, in which
\citet{zhang2025minlength} construct minimum-length ordinal sets that are optimal per
instance given the base model's scores, which is a reminder that width has a
set-construction component and is not simply a property of the estimator underneath. The
sets built here are certified rather than length-optimized, and where this paper reports
width it reports the width of the construction it certifies.

The undercoverage of minority subpopulations under marginal calibration has also been
documented outside transportation. \citet{tursunbadalov2026quietfailure} find, in molecular
property prediction, that marginal conformal prediction meets its global target while
minority-class coverage falls far below it, and that class-conditional calibration restores
it at a modest cost in set size. \citet{rafe2026socioconformal} find the same subgroup gaps in
survey-based social measurement, but there group-specific calibration worsened the
efficiency and fairness trade-off. That study attributes the result to thin calibration
cells, with 32 to 355 observations per cell. This is Theorem~\ref{thm:oracle}(ii) read at
the wrong end of its rate, since the error of a class-conditional threshold shrinks as
$O_p(n_c^{-1/2})$. The rollup rule of \eqref{eq:nmin} is this paper's safeguard against the problem. It
declares a minimum cell size and merges a sparse cell into a coarser parent, at the cost of
giving up a separate guarantee for that cell. No declared cell in the Texas data falls
below the minimum, so the rule is not triggered here. In the declared safety-stratum
audits reported here, the thinnest calibration cell holds 2{,}787 records
(Table~\ref{tbl:compose}), and the thinnest cell in the primary four-cell audit holds 8{,}386. Theorem~\ref{thm:het} applies
the established Mondrian argument to a partition fixed in advance, and
Theorem~\ref{thm:oracle} specializes the efficiency characterization to the associated
threshold family and rate.

Coverage transfer from ambiguous or weakly supervised labels to the clean label is treated
by \citet{cauchois2022} and \citet{stutz2023}, and it is the closest precedent for
Section~\ref{sec:noise}. Label noise has also been attacked from several other directions,
in that \citet{einbinder2024} study robustness and correction under noisy labels,
\citet{sesia2023} model random contamination, \citet{noiseadaptive2025} pursue
noise-adaptive classification, \citet{penso2025} estimate a clean threshold under uniform
noise, \citet{xi2025} handle known uniform noise online, and \citet{cohen2025} address
noisy regression. Distribution shift is addressed by weighted conformal prediction
\citep{tibshirani2019}, the non-exchangeable extensions of \citet{barber2023},
group-weighted constructions \citep{bhattacharyya2024}, and recent work on likelihood-ratio
regularization \citep{joshi2025}, optimal transport with unlabeled targets
\citep{correia2025}, and clipped weights \citep{wanggoel2026}. Risk beyond coverage is
handled by conformal risk control \citep{angelopoulos2024}, whose closest safety-critical
application is the wildfire-evacuation mapping of \citet{dayan2026}. Sequential monitoring
of exchangeability by test martingales is due to \citet{vovk-martingales}.

Each of these supplies an ingredient that this paper uses and labels as such. The
label-noise literature also defines the remaining methodological gap. Those methods
correct or calibrate under a stochastic noise model, a noise rate, or a uniformity premise.
The reporting structure can vary by agency and era \citep{burdett2022}, so this paper does
not estimate a Texas confusion kernel. Instead, it makes a declared ordinal support band a
sensitivity input. The cited linkage evidence points to mostly adjacent-category disagreement and much closer
agreement for fatal injury. The declared band is therefore a transparent assumption,
not a learned reporting law.

\subsection{The gap}
\label{sec:lit-gap}

The three streams reviewed above supply complementary pieces, but not a single framework
for a deployed police-reported injury model. Severity modeling supplies flexible conditional
distributions and heterogeneity structure, but not a finite-sample certificate.
Safety-critical applications of generic conformal prediction inherit generic guarantees.
The conformal literature supplies validity, groupwise conditioning, shift correction, and
risk control as separate results. CHOIR retains those established tools and connects them
to a prespecified KABCO estimand, training-frozen final cells, and an auditable reporting
structure. Its accident-specific results address contiguous ordinal sets, coverage transfer
under a declared reporting band, the associated width cost and sharpness, one-sided
mismatch accounting, and severity-weighted omission risk.
The Texas CRIS analysis evaluates the framework across seven base models and four safety
strata, and uses county, year, and risk-budget analyses as descriptive stress tests.

\newcommand{\frameworkopening}{%
\noindent The framework treats police-reported KABCO injury as the observed ordinal
outcome for each sampled driver. Table~\ref{tbl:notation} gathers the notation used below, and
Figure~\ref{fig:pipeline} outlines the certification workflow.\par
\begin{table}[pos=htbp]
\caption{Notation.}
\label{tbl:notation}
\centering\footnotesize
\begin{threeparttable}
\begin{tabular*}{\tblwidth}{@{\extracolsep{\fill}}llc@{}}
\toprule
Symbol & Meaning & Defined \\
\midrule
\multicolumn{3}{@{}l}{\textsc{Data and outcome}}\\
$\Y=\{1,\dots,5\}$ & police-reported KABCO injury, O$\prec$C$\prec$B$\prec$A$\prec$K & \S\ref{sec:setup} \\
$Y$, $\yt$ & underlying injury, police-reported injury & \S\ref{sec:setup} \\
$X\in\X$ & covariates & \S\ref{sec:setup} \\
$\Itr$, $\Ical$ & training, calibration index sets & \S\ref{sec:setup} \\
$n$ & calibration size; test point is $n{+}1$ & \S\ref{sec:setup} \\
\midrule
\multicolumn{3}{@{}l}{\textsc{Fitted tier (arbitrary, never trusted)}}\\
$\hat F(k\mid x)$ & base conditional CDF & \eqref{eq:Fhat} \\
$\hat c(x)$ & latent-class assignment, $C$ classes & \S\ref{sec:het} \\
$h(x)$, $R(x)$ & raw-cell map and training-frozen final-cell map & \eqref{eq:finalcell} \\
$\hat w(x)$ & density ratio for transfer & \eqref{eq:wq} \\
\midrule
\multicolumn{3}{@{}l}{\textsc{Certification machinery}}\\
$s(x,y)$ & ordinal cumulative score & \eqref{eq:score} \\
$C_\lambda(x)$ & threshold set, a contiguous interval & \eqref{eq:sets} \\
$\hat q$, $\hat q_c$, $\hat q_a$ & marginal, class, final-cell quantile & \eqref{eq:qhat}, \eqref{eq:qc}, Thm.~\ref{thm:groups} \\
$q_c$, $q_{\mathrm{mix}}$ & population class and pooled quantiles & Thm.~\ref{thm:oracle} \\
$\alpha$ & miscoverage level & \eqref{eq:qhat} \\
$\Tmap(\yt)$ & declared compatibility band & \eqref{eq:band} \\
$\delta$ & declared beyond-band mass (an input) & \eqref{eq:band} \\
$\varepsilon_{\mathrm{tot}}$ & total reporting error $\Prob(\yt\ne Y)$ & Prop.~\ref{prop:tightness} \\
$\Cop(x)$ & band-expanded set & \eqref{eq:expansion} \\
$\kappa(\cdot)$, $\beta$ & severity cost, risk budget & \eqref{eq:crc} \\
$\hat\lambda$ & conformal risk-control threshold & \eqref{eq:lambdahat} \\
$\dtv$ & population transfer slack in Theorem~\ref{thm:transfer} & \eqref{eq:slack} \\
$\widehat\Delta^{\mathrm{LCB}}_{\mathrm{emp}}$ & classifier mismatch statistic & Thm.~\ref{thm:transfer}(ii) \\
$M_t$, $\varepsilon$, $b$ & review score, betting parameter, review threshold & \eqref{eq:martingale}, Alg.~\ref{alg:monitor} \\
\bottomrule
\end{tabular*}
\begin{tablenotes}[flushleft]\footnotesize
\item[] \textit{Note:} The labels established, adapted, specialized, and derived in each result heading record whether a statement is taken from the literature or developed here.
\end{tablenotes}
\end{threeparttable}
\end{table}
\begin{figure}[pos=htbp]
\centering
\includegraphics[width=\linewidth,height=.58\textheight,keepaspectratio]{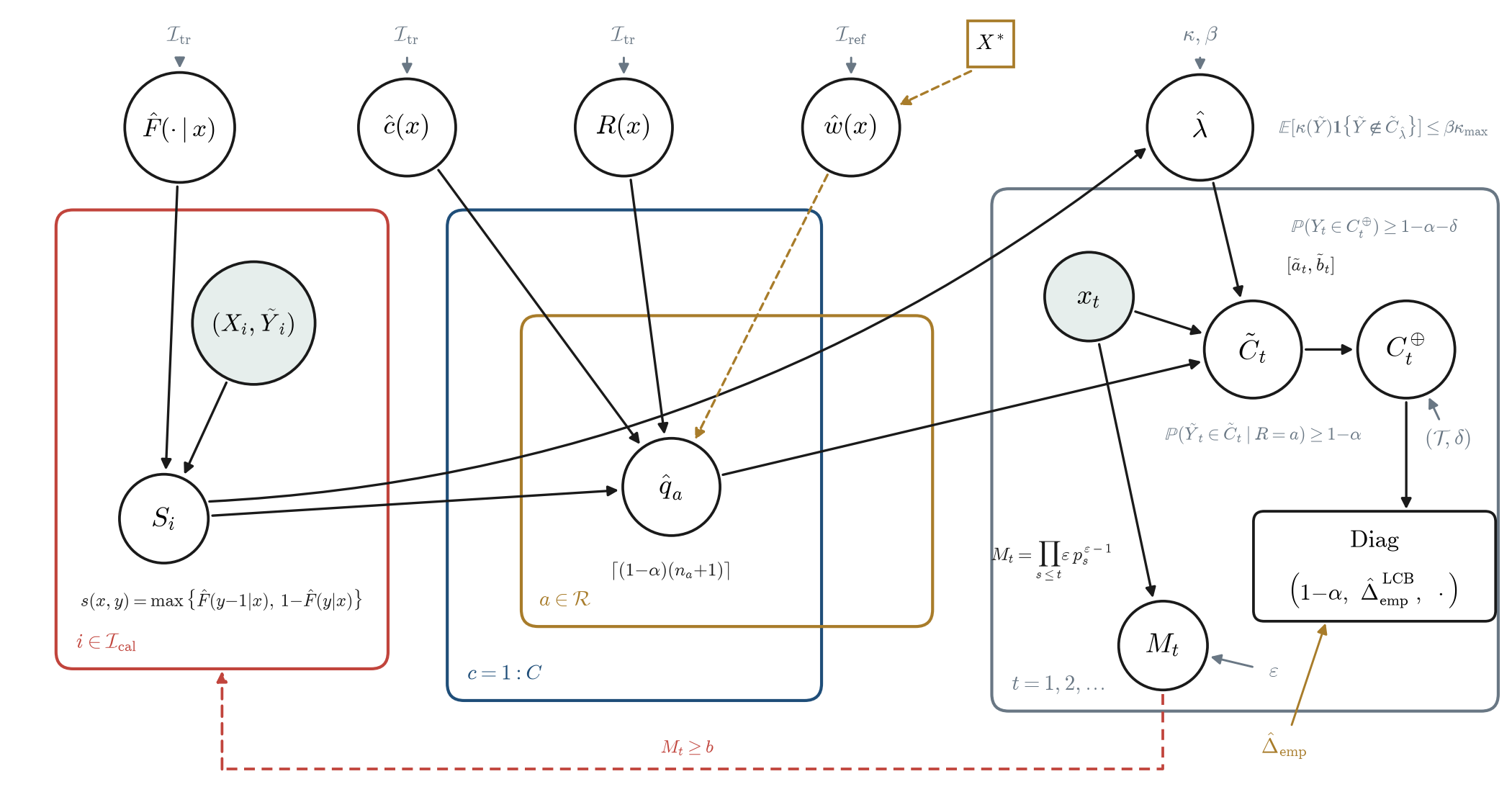}
\caption{CHOIR as a graph. Circles are random or fitted objects, the rectangle an exported
artifact, shaded nodes observed data; plates denote replication over calibration records
($i\in\Ical$), latent classes ($c$), final calibration cells ($a$) and the deployment stream
($t$). The base CDF $\hat F$, the class gate $\hat c$, and the final-cell map $R$ are fit on
the training set; the density ratio $\hat w$ is fit on a calibration-law reference sample
$\mathcal I_{\mathrm{ref}}$ and unlabeled target covariates $X^\ast$. Validity comes from
calibration. The final-cell quantile $\hat q_a$ creates the reported-label certificate. The
dashed gold path is the separate new-stratum branch, in which the weighted quantile of
Theorem~\ref{thm:transfer} replaces $\hat q_a$. The declared band $(\Tmap,\delta)$ expands the
set for latent-injury sensitivity analysis, and $(\kappa,\beta)$ controls observed-cell
omission risk. The Diag box is the exported diagnostic tuple, not a certificate. The dashed
red loop is a review trigger with threshold $b$ (Algorithm~\ref{alg:monitor}) that records an
alarm and starts recalibration; it carries no anytime error guarantee.}
\label{fig:pipeline}
\end{figure}}

\section{Framework and guarantees}
\label{sec:framework}

\ifdefined\frameworkopening
  \frameworkopening
\fi

Each result carries one of four labels in its heading. Established results are taken from the literature with their original attribution. Adapted results apply a known argument to a new object, such as the frozen crash strata. Specialized results apply known arguments to the ordinal crash-injury structure and state the consequences for KABCO. Derived results are new statements that connect the declared assumptions within CHOIR. Table~\ref{tbl:attribution} states the statistical basis and crash-specific role of each component. Proofs of Theorems~\ref{thm:oracle}, \ref{thm:sharp}, and~\ref{thm:transfer} are given in Appendix~\ref{app:proofs}; all other proofs appear in line.

\begin{table}[pos=htbp]
\centering
\caption{Statistical attribution and accident-specific role of the CHOIR components.}
\label{tbl:attribution}
\begin{tabular}{p{.20\linewidth}p{.31\linewidth}p{.39\linewidth}}
\toprule
Component & Statistical basis & Accident-specific role \\
\midrule
Ordinal intervals & CDF-based conformal scores & Contiguous prediction sets on the KABCO scale \\
Final-cell validity & Split and Mondrian conformal prediction & Training-frozen crash strata and geographic rollup \\
Reporting-map transfer & Coverage transfer under label uncertainty & Declared sensitivity set for medically referenced injury \\
Reporting-channel floor & Oracle expected-size optimality & Width lower bound under a declared crash reporting channel \\
Population shift & Weighted conformal prediction and total variation & Conditional transfer bound with a separate mismatch diagnostic \\
Omission risk & Conformal risk control & Injury-weighted and fatality-specific bounds under declared premises \\
Certification accounting & Probability union bounds & Assumption-attributable audit record for each issued result \\
\bottomrule
\end{tabular}
\end{table}

\subsection{Setup, notation, and standing assumptions}
\label{sec:setup}

Severity categories $\Y=\{1,\dots,K\}$ with $K=5$, ordered
$1(\mathrm{O}) \prec 2(\mathrm{C}) \prec 3(\mathrm{B}) \prec 4(\mathrm{A}) \prec 5(\mathrm{K})$
(KABCO, running from no injury through possible, non-incapacitating and incapacitating injury to fatal). Covariates $X\in\X$. $Y$ denotes medically referenced injury severity; $\yt$ denotes police-reported injury severity. The CRIS models are fit to $\yt$ because $Y$ is not observed in CRIS. Data are split into a training set $\Itr$, on which the base model and all prediction-cell maps are fit, and a calibration set $\Ical=\{1,\dots,n\}$; the test point is indexed $n+1$. A density-ratio model for a shifted target uses a separate calibration-law covariate reference sample, as stated in Theorem~\ref{thm:transfer}. All objects used in a conformal rank argument are fixed before the scored calibration fold is used.

\begin{assumption}[A1. Exchangeability]\label{ass:a1}
$(X_1,\yt_1),\dots,(X_n,\yt_n),(X_{n+1},\yt_{n+1})$ are exchangeable.
\end{assumption}

\begin{remark}[Clustered crash data]\label{rem:cluster}
Persons in one crash are dependent, so rows from the same crash violate A1 if they straddle calibration and test. Under the protocol, all splits are taken at the crash identifier, and the primary analysis calibrates and tests on one randomly sampled driver row per crash, so that the sampled rows are exchangeable under random crash-level splitting and A1 holds for them. Using every occupant row would require the grouped conformal constructions of \citet{dunn2022}; this paper does not use all rows.
\end{remark}

\paragraph{Base model.} Any estimator fit on $\Itr$ exporting an estimated conditional CDF
\begin{equation}\label{eq:Fhat}
\hat F(k\mid x)=\hat P(\yt\le k\mid X=x),\qquad \hat F(0\mid x)\equiv 0,\quad \hat F(K\mid x)\equiv 1,
\end{equation}
non-decreasing in $k$. Ordered logit and probit, random-parameters logit, and latent-class logit export $\hat F$ natively; any probabilistic classifier exports it by cumulating $\hat p(\cdot\mid x)$. Here $\hat P$ is a fitted score-producing function and does not assert correct specification. It is not an estimate of $P(Y\le k\mid X=x)$ from CRIS. Every reported-label guarantee below is distribution-free in $\hat F$; model quality affects set size, not the rank argument.

\subsection{Ordinal scores, interval sets, marginal validity}
\label{sec:score}

\begin{definition}[Ordinal cumulative score]\label{def:score}
\begin{equation}\label{eq:score}
s(x,y)=\max\bigl\{\hat F(y-1\mid x),\; 1-\hat F(y\mid x)\bigr\}.
\end{equation}
\end{definition}

$\hat F(y-1\mid x)$ is the predicted mass strictly below $y$; $1-\hat F(y\mid x)$ the mass strictly above, so that $s$ measures how deep into a tail the candidate label sits (cf.\ CDF-based ordinal scores, \citealp{lu2022,chakraborty2024}).

\begin{definition}[Threshold sets]\label{def:sets}
\begin{equation}\label{eq:sets}
C_\lambda(x)=\{k\in\Y \mid s(x,k)\le\lambda\}.
\end{equation}
\end{definition}

\begin{lemma}[Contiguity and non-triviality; \tagspecialized{}]\label{lem:contig}
For every $x$ and $\lambda\ge 0$, the following properties hold.
(i) $C_\lambda(x)$ is a contiguous (possibly empty) interval,
\begin{equation}\label{eq:interval}
C_\lambda(x)=\{a_\lambda(x),\dots,b_\lambda(x)\}\cap\Y;
\end{equation}
(ii) if $\lambda\ge 1/2$ then $C_\lambda(x)\ne\emptyset$ and contains the predictive median
\begin{equation}\label{eq:median}
k^\ast(x)=\min\{k\mid\hat F(k\mid x)\ge 1/2\};
\end{equation}
(iii) $\lambda\mapsto C_\lambda(x)$ is nested non-decreasing.
\end{lemma}

\begin{proof}
(i) $C_\lambda(x)=L\cap U$ with $L=\{k:\hat F(k-1\mid x)\le\lambda\}$, $U=\{k:\hat F(k\mid x)\ge 1-\lambda\}$. Since $\hat F(\cdot\mid x)$ is non-decreasing, $L$ is a down-set and $U$ an up-set; $1\in L$ (as $\hat F(0)=0$) and $K\in U$ (as $\hat F(K)=1$). The intersection of a down-set $\{1,\dots,b_\lambda\}$ and an up-set $\{a_\lambda,\dots,K\}$ is the interval $\{a_\lambda,\dots,b_\lambda\}$, empty iff $a_\lambda>b_\lambda$.
(ii) At $k^\ast$, $\hat F(k^\ast-1\mid x)<1/2\le\lambda$ and $1-\hat F(k^\ast\mid x)\le 1/2\le\lambda$, so $s(x,k^\ast)\le 1/2\le\lambda$.
(iii) Immediate from Definition~\ref{def:sets}.
\end{proof}

\begin{definition}[Deployed non-empty set]\label{conv:nonempty}
The raw set is $C_\lambda$ in Definition~\ref{def:sets}. The deployed set is
\begin{equation}\label{eq:nonempty}
C^{\mathrm{ne}}_\lambda(x)=
\begin{cases}
C_\lambda(x),&C_\lambda(x)\ne\emptyset,\\
\{k^\dagger(x)\},&C_\lambda(x)=\emptyset,
\end{cases}
\qquad
k^\dagger(x)=\min\arg\min_{k\in\Y}s(x,k).
\end{equation}
The minimum-index rule makes the fallback deterministic. Pointwise,
$C_\lambda(x)\subseteq C^{\mathrm{ne}}_\lambda(x)$. Thus, lower coverage bounds transfer to the deployed set. Coverage upper bounds, coverage equivalences, and raw-family efficiency statements do not transfer unless stated separately.
\end{definition}

\begin{proposition}[Marginal validity of split conformal; \tagknown{}; \citealp{vovk2005,lei2018}]\label{prop:marginal}
Let $S_i=s(X_i,\yt_i)$ for $i\in\Ical$, with order statistics $S_{(1)}\le\dots\le S_{(n)}$, and set
\begin{equation}\label{eq:qhat}
\hat q \;=\; S_{(\lceil(1-\alpha)(n+1)\rceil)}
\end{equation}
(with $\hat q=+\infty$, i.e.\ $C=\Y$, if $\lceil(1-\alpha)(n+1)\rceil>n$). Under A1,
\begin{equation}\label{eq:marginal}
\Prob\bigl(\yt_{n+1}\in C_{\hat q}(X_{n+1})\bigr)\ \ge\ 1-\alpha,
\end{equation}
and the same lower bound holds with $C^{\mathrm{ne}}_{\hat q}$ in place of $C_{\hat q}$. If the $n+1$ scores are a.s.\ distinct, the raw-set coverage is also at most $1-\alpha+\tfrac{1}{n+1}$. This upper bound is not asserted for the deployed fallback set.
\end{proposition}

\begin{proof}
Exchangeability makes the rank of $S_{n+1}$ among $\{S_1,\dots,S_{n+1}\}$ (sub-)uniform on $\{1,\dots,n+1\}$. The raw-set event $\yt_{n+1}\in C_{\hat q}(X_{n+1})$ equals $\{S_{n+1}\le\hat q\}$, which contains the event that this rank is at most $\lceil(1-\alpha)(n+1)\rceil$, of probability $\ge 1-\alpha$. The raw upper bound is the standard no-ties argument. The deployed lower bound follows from $C_{\hat q}\subseteq C^{\mathrm{ne}}_{\hat q}$. This containment does not preserve the upper bound.
\end{proof}

\begin{remark}[Ties]\label{rem:ties}
$s$ takes values in the finite set of $\hat F$-values, so ties are possible; the lower bound is unaffected. Exactness can be restored by tie-breaking jitter, provided a single uniform draw is added to all candidate labels of the same record, which keeps each set a threshold set and therefore contiguous. How much ties matter depends on the probability that sits on the tied score values, not on the calibration size.
\end{remark}

\begin{remark}[Why intervals matter]\label{rem:intervals}
Contiguity earns its place for two substantive reasons. First, a contiguous set has interpretable lower and upper endpoints and cannot omit a middle category while retaining both ends. It is not automatically an upper-tail statement such as ``B or worse''; any escalation rule must be defined from the endpoints. Second, contiguity keeps the noise expansion of Theorem~\ref{thm:noise} additive, because an interval expands to an interval at a cost of $b^-+b^+$ categories, whereas a general set of size $|\Ct|$ admits only the multiplicative bound $|\Ct|(b^-+b^++1)$. The guarantee itself does not require contiguity (Remark~\ref{rem:contigrole}); the efficiency of the guarantee does.
\end{remark}

\begin{proposition}[Impossibility of exact conditional coverage; \tagknown{}; \citealp{vovk2012,foygelbarber2021}]\label{prop:impossible}
Suppose $P_X$ is non-atomic. Then any procedure achieving $\Prob(Y\in C(X)\mid X=x)\ge 1-\alpha$ for a.e.\ $x$ under all distributions must, for almost every $x$, include each label with probability at least $1-\alpha$, so its sets cannot adapt usefully to $x$. Coverage conditional on finitely many pre-declared groups is therefore a natural attainable target, and this is the justification for Sections~\ref{sec:het} and~\ref{sec:shift}. The transportation-relevant groups are heterogeneity classes, jurisdictions, and time periods.

The non-atomicity of $P_X$ is the operative condition and is not a technicality, because the obstruction is driven by the covariate law, not by $|\Y|<\infty$. Were $X$ finitely supported the statement would fail, and Theorem~\ref{thm:het} with $\hat c$ the identity on that support would be the counterexample. Whether the population law is non-atomic is a modeling condition. Field names such as age, speed limit, and traffic volume do not establish it because recorded values can lie on discrete grids. The proposition applies to the application only under an explicit non-atomicity assumption; near-uniqueness of observed records is empirical support, not a proof of that assumption.
\end{proposition}

\subsection{Heterogeneity-conditional validity}
\label{sec:het}

Let $\hat c:\X\to\{1,\dots,C\}$ be \emph{any} measurable class-assignment function fit on $\Itr$; in CHOIR it is the MAP class of a latent-class model (deep gate or classical latent-class ordered logit). Let $n_c=\#\{i\le n:\hat c(X_i)=c\}$ and let $\hat q_c$ be the class-wise analogue of \eqref{eq:qhat},
\begin{equation}\label{eq:qc}
\hat q_c \;=\; S^{(c)}_{(\lceil(1-\alpha)(n_c+1)\rceil)},\qquad
S^{(c)}_{(1)}\le\dots\le S^{(c)}_{(n_c)}\ \text{the ordered scores}\ \{S_i\mid\hat c(X_i)=c\},
\end{equation}
and define the raw and deployed class-conditional predictors
\begin{equation}\label{eq:chet}
C^{\mathrm{het,raw}}(x)=C_{\hat q_{\hat c(x)}}(x),
\qquad
C^{\mathrm{het}}(x)=C^{\mathrm{ne}}_{\hat q_{\hat c(x)}}(x).
\end{equation}

\begin{theorem}[Class-conditional coverage; \tagadapted{}; Mondrian argument, \citealp{vovk2005}]\label{thm:het}
Under A1, for every class $c\in\{1,\dots,C\}$,
\begin{equation}\label{eq:hetcov}
\Prob\bigl(\yt_{n+1}\in C^{\mathrm{het,raw}}(X_{n+1})\,\bigm|\,\hat c(X_{n+1})=c\bigr)\ \ge\ 1-\alpha,
\end{equation}
and the same lower bound holds for $C^{\mathrm{het}}$. The raw set is an interval or empty; the deployed set is always a non-empty contiguous interval.
\end{theorem}

\begin{proof}
Fix $c$; condition on $E_c=\{\hat c(X_{n+1})=c\}$ and on the index set $I_c\subseteq\{1,\dots,n\}$ of calibration points in class $c$. Because $\hat c$ is a fixed measurable function, membership in class $c$ is a deterministic function of the data point; exchangeability of the full collection implies exchangeability of $\{(X_i,\yt_i):i\in I_c\}\cup\{(X_{n+1},\yt_{n+1})\}$ conditionally on $E_c$ and $|I_c|=n_c$, because any permutation $\pi$ of the class-$c$ points extends to a permutation of all $n+1$ points fixing the complement, and the conditioning event is permutation-symmetric within the class. Apply Proposition~\ref{prop:marginal}'s rank argument to these $n_c+1$ exchangeable scores with quantile index $\lceil(1-\alpha)(n_c+1)\rceil$; average over $n_c$. The deployed result follows by pointwise containment. Contiguity is Lemma~\ref{lem:contig} at $\lambda=\hat q_c$, with the deterministic singleton fallback when the raw set is empty.
\end{proof}

\begin{remark}[Arbitrariness is harmless]\label{rem:arbitrary}
Theorem~\ref{thm:het} requires nothing about $\hat c$ being ``correct'', since even a nonsensical partition yields per-cell validity. Partition quality affects \emph{efficiency} (set size), quantified next. The latent-class machinery of the random-parameters tradition cannot break validity; it can only earn efficiency. The guarantee is per cell and marginal, in that each cell covers at $1-\alpha$ over the calibration draw, not simultaneously across cells at a family-wise level, so a deployment auditing many cells at once should read the coverages as marginal and budget accordingly.
\end{remark}

\begin{theorem}[Efficiency within a fixed classwise threshold family; \tagspecialized{}]\label{thm:oracle}
Define the class-conditional score CDFs, their $(1-\alpha)$-quantiles, and the class proportions
\begin{equation}\label{eq:Gc}
G_c(t)=\Prob\bigl(s(X,\yt)\le t\,\bigm|\,\hat c(X)=c\bigr),\qquad
q_c=\inf\{t\mid G_c(t)\ge 1-\alpha\},\qquad
p_c=\Prob(\hat c(X)=c)>0.
\end{equation}
Let $q_{\mathrm{mix}}$ be the $(1-\alpha)$-quantile of the mixture $\sum_c p_cG_c$, the population threshold of marginal split conformal. We use assumptions R1 and R2. \textbf{(R1)} Each $G_c$ is continuous and strictly
increasing on an interval containing $q_c$ and $q_{\mathrm{mix}}$, and the mixture
$\sum_c p_cG_c$ is continuous and strictly increasing near $q_{\mathrm{mix}}$.
\textbf{(R2)} For every $k\in\Y$ and $c$, we have
$\Prob(s(X,k)=q_c\mid\hat c(X)=c)=0$.

Consider the family $\mathcal F_{\mathrm{raw}}$ of raw predictors $C_{\boldsymbol\lambda}(x)=C_{\lambda_{\hat c(x)}}(x)$ indexed by class-wise thresholds $\boldsymbol\lambda=(\lambda_1,\dots,\lambda_C)$. The following statements hold.

\noindent\textbf{(i) Oracle characterization.} Within $\mathcal F_{\mathrm{raw}}$, class-$c$ coverage $\ge 1-\alpha$ holds iff $\lambda_c\ge q_c$; the componentwise-smallest valid threshold vector is $\boldsymbol\lambda^\star=(q_1,\dots,q_C)$, and it is an expected-size minimizer with value
\begin{equation}\label{eq:Sstar}
\mathcal S^\star=\sum_c p_c\,\E\bigl[\,|C_{q_c}(X)|\,\bigm|\,\hat c(X)=c\,\bigr].
\end{equation}
\noindent\textbf{(ii) Attainment.} Under i.i.d.\ sampling within each class, as $\min_c n_c\to\infty$, $\hat q_c\xrightarrow{p}q_c$ for every $c$, with $|\hat q_c-q_c|=O_p(n_c^{-1/2})$ under (R1) with positive density at $q_c$ \citep[Cor.~21.5]{vandervaart1998}. The conditional-on-calibration expected size of $C^{\mathrm{het,raw}}$ converges in probability to $\mathcal S^\star$.

\noindent\textbf{(iii) Diagnosis of marginal conformal.} Marginal split conformal has class-$c$ coverage converging to $G_c(q_{\mathrm{mix}})$, which is $<1-\alpha$ for every class with $q_c>q_{\mathrm{mix}}$, and $>1-\alpha$ for every class with $q_c<q_{\mathrm{mix}}$. In this limit, it is class-conditionally valid for all $c$ if and only if $q_1=\dots=q_C$; in finite samples the statement holds up to the sampling error of the empirical quantiles.
\end{theorem}

\begin{remark}[What (iii) does and does not identify]\label{rem:diagnosisscope}
Part (iii) follows because the pooled threshold $q_{\mathrm{mix}}$ is a single quantile of the mixture of class score distributions, so every class whose own quantile lies above it is under-covered and every class whose quantile lies below it is over-covered. It predicts a \emph{direction} and nothing more, supplying no magnitude, and in particular it does not identify \emph{which} classes are the hard ones, since whether $q_c>q_{\mathrm{mix}}$ for a given subpopulation is an empirical property of the data and not a consequence of this theorem. That the hard classes in Texas crash records turn out to be unrestrained drivers, motorcycle drivers and rural high-speed records is an empirical finding, reported in Section~\ref{sec:results} and not claimed here. The theorem diagnoses the fixed score and threshold family. It does not show that a different fitted score, feature set, or model class cannot improve the class-conditional score distributions.
\end{remark}

\paragraph{Scope of Theorem~\ref{thm:oracle}.}
Theorem~\ref{thm:oracle} does \emph{not} claim that class-conditional (Mondrian) sets are smaller on average than marginal sets. Marginal calibration gives hard classes sets that are too small and easy classes sets that are larger than needed; class-conditional calibration reverses both, so the average size can rise or fall. The proved minimum is within $\mathcal F_{\mathrm{raw}}$ for the fixed score. It does not include the non-empty fallback, whose added singleton can change coverage and size below the raw oracle threshold. For safety applications, undercoverage of severe strata is precisely the failure mode that matters, and the empirical section reports average sizes in full. A second point of scope concerns (R1)--(R2), since a discrete base model exports finitely many $\hat F$-values, so the score distributions are atomic and both conditions can fail exactly; the tie-breaking jitter of Remark~\ref{rem:ties} makes the scores continuous, but the strict monotonicity in (R1) remains an assumption.

\subsection{When certification is informative, and the floor no model reaches below}
\label{sec:informative}

Theorem~\ref{thm:oracle} compares members of the threshold family $\mathcal F_{\mathrm{raw}}$ at a fixed $\hat F$. It is therefore silent on the question a practitioner asks on being handed a certificate reading ``somewhere between no injury and fatal'', which is whether a better base model would have done better. This subsection gives a law-dependent floor on expected set size in which no base model appears.

Throughout, $Z$ denotes the label being certified. Everything below is stated for $Z=\yt$, the reported label on which the sets of Sections~\ref{sec:score}--\ref{sec:het} are calibrated, and holds verbatim for $Z=Y$. Let
\begin{equation}\label{eq:truep}
p(z\mid x)=\Prob(Z=z\mid X=x)
\end{equation}
be the \emph{true} conditional mass function, which is nowhere assumed known and nowhere estimated. Fix a stratum $A\in\sigma(X)$ with $\Prob(X\in A)>0$.

\begin{definition}[Level sets and the level-set functional]\label{def:levelset}
For $t\ge 0$,
\begin{equation}\label{eq:Lt}
L_t(x)=\{z\in\Y\mid p(z\mid x)>t\},\qquad N(x,t)=|L_t(x)|,
\end{equation}
and write $\gamma_A(t)=\Prob\bigl(p(Z\mid X)>t\bigm| X\in A\bigr)$ for the coverage of the level-set rule at $t$, non-increasing with $\gamma_A(0)=1$. Let $U=(U_1,\ldots,U_K)$ have independent $\mathrm{Unif}(0,1)$ coordinates and be independent of $(X,Z)$. Define the \emph{randomized floor} as
\begin{equation}\label{eq:Ndef}
\mathcal N_A(\alpha)\;=\;\inf\Bigl\{\E\bigl[|C(X,U)|\bigm| X\in A\bigr]\ :\ C\ \sigma(X,U)\text{-measurable},\ \Prob\bigl(Z\in C(X,U)\bigm| X\in A\bigr)\ge 1-\alpha\Bigr\},
\end{equation}
which exists for every $\alpha\in(0,1)$ because $C\equiv\Y$ is feasible. A deterministic predictor is the special case that ignores $U$. Theorem~\ref{thm:levelset} identifies the randomized optimum with a level set plus boundary randomization. Writing $t_A(\alpha)=\sup\{t\ge 0:\gamma_A(t)\ge 1-\alpha\}$ gives
\begin{equation}\label{eq:tA}
\mathcal N_A(\alpha)=\E\bigl[N(X,t_A(\alpha))\bigm| X\in A\bigr]
\qquad\text{whenever}\qquad
\Prob\bigl(p(Z\mid X)=t_A(\alpha)\bigm| X\in A\bigr)=0 .
\end{equation}
$\alpha\mapsto\mathcal N_A(\alpha)$ is non-increasing.
\end{definition}

\begin{remark}[Why the floor is defined as a value and not as a level set]\label{rem:boundary}
Equation \eqref{eq:tA} can fail at an atom. If $Z=f(X)$, then $\gamma_A(t)=1$ for every $t<1$, so $t_A(\alpha)=1$ and $N(x,1)=0$. The randomized floor in \eqref{eq:Ndef} is instead $1-\alpha$, attained by returning $\{f(X)\}$ with probability $1-\alpha$ and the empty set otherwise. A separately imposed non-empty-output constraint defines a different feasible class; the deployed fallback in Definition~\ref{conv:nonempty} must not be inserted into the raw randomized optimization.
\end{remark}

$\mathcal N_A(\alpha)$ is a functional of the joint law of $(X,Z)$ restricted to $A$ and of nothing else. It does not depend on $\hat F$, on the calibration split, or on any modeling choice.

\begin{theorem}[Oracle expected-size benchmark; \tagspecialized{}; level-set optimality, \citealp{sadinle2019}]\label{thm:levelset}
Let $C:\X\times[0,1]^K\to 2^{\Y}$ be any possibly randomized set-valued predictor satisfying
\begin{equation}\label{eq:Acov}
\Prob\bigl(Z\in C(X,U)\bigm| X\in A\bigr)\ \ge\ 1-\alpha .
\end{equation}
Then, by \eqref{eq:Ndef},
\begin{equation}\label{eq:floor}
\E\bigl[\,|C(X,U)|\bigm| X\in A\,\bigr]\ \ge\ \mathcal N_A(\alpha),
\end{equation}
and the floor is attained by including all labels with $p(z\mid x)>t_A(\alpha)$ and randomizing only among labels on the boundary $p(z\mid x)=t_A(\alpha)$. When the boundary has probability zero, the optimizer is the deterministic set $L_{t_A(\alpha)}$ and \eqref{eq:tA} holds.
\end{theorem}

\begin{proof}
For a randomized predictor define the conditional inclusion probability $u(x,z)=\Prob_U\{z\in C(x,U)\}$. Conversely, any measurable $u:\X\times\Y\to[0,1]$ is implemented by including $z$ when $U_z\le u(x,z)$. Thus, \eqref{eq:Ndef} is equivalent to
\begin{equation}\label{eq:lp}
\min_u\ \E\Bigl[\sum_{z\in\Y}u(X,z)\Bigm| X\in A\Bigr]
\quad\text{s.t.}\quad
\E\Bigl[\sum_{z\in\Y}p(z\mid X)\,u(X,z)\Bigm| X\in A\Bigr]\ \ge\ 1-\alpha .
\end{equation}
For $\tau\ge 0$ the Lagrangian is
\begin{equation}\label{eq:lagr}
\E\Bigl[\sum_{z}u(X,z)\bigl(1-\tau\,p(z\mid X)\bigr)\Bigm| X\in A\Bigr]+\tau(1-\alpha),
\end{equation}
which decouples across $(x,z)$, since the coefficient of $u(x,z)$ is negative exactly when $p(z\mid x)>1/\tau$. A pointwise minimizer therefore sets $u=1$ above the threshold, $u=0$ below it, and uses a common fractional allocation on the boundary as needed to make the constraint bind. This is the stated boundary-randomized level set. Away from a boundary atom the solution is integral, hence deterministic, and its value is $\E[N(X,t_A(\alpha))\mid X\in A]$.
\end{proof}

\begin{remark}[Variational form]\label{rem:variational}
Equation \eqref{eq:floor} is often written as a budget allocation, in which one spends miscoverage $u(x)\ge 0$ at $x$ subject to $\E[u(X)\mid X\in A]\le\alpha$ and pays the smallest set capturing conditional mass $1-u(x)$. Theorem~\ref{thm:levelset} is that infimum solved. Its content is that the optimal budget is not free to vary arbitrarily with $x$, since it is the one induced by a single threshold on $p(\cdot\mid x)$, which is the Neyman--Pearson structure and is what makes $\mathcal N_A(\alpha)$ computable in principle from the law alone.
\end{remark}

\begin{corollary}[The floor binds interval predictors]\label{cor:levelsetint}
Every deterministic contiguous-interval predictor is a special case of the class in Theorem~\ref{thm:levelset}. Equation~\eqref{eq:floor} applies only when \eqref{eq:Acov} holds on the same event $\{X\in A\}$. Thus it applies to $C^{\mathrm{het}}$ on its certified class or final calibration cell, and to $\Cop$ only on a final cell where the latent-label bound is stated at the corresponding error level. It does not create validity on a raw leaf inside a rolled-up cell.
\end{corollary}

\begin{corollary}[Expectation and full-set consequence]\label{cor:vacuity}
Fix $\varepsilon\in(0,1)$. If $\mathcal N_A(\alpha)\ge K-\varepsilon$, then every predictor valid at level $1-\alpha$ on $A$ in the sense of \eqref{eq:Acov} has mean set size at least $K-\varepsilon$. Because $K-|C(X,U)|$ is a non-negative integer,
\begin{equation}\label{eq:vacuity-prob}
\Prob\{C(X,U)=\Y\mid X\in A\}\ge 1-\varepsilon.
\end{equation}
This consequence applies to an expanded predictor only when that expanded predictor has the required latent-label coverage on the same event $A$. It does not infer where a base interval is wide from an expected-width bound. The condition is restrictive. The predictor that keeps each label independently with probability $1-\alpha$ is feasible, so $\mathcal N_A(\alpha)\le(1-\alpha)K$ always, and the condition needs $\varepsilon\ge\alpha K$. It can never hold when $\alpha\ge 1/K$, and at $\alpha=0.10$ and $K=5$ it gives at most $\Prob\{C=\Y\mid X\in A\}\ge 0.5$. The corollary is therefore a limiting statement and is not used to interpret the empirical widths.
\end{corollary}

\begin{remark}[What this licenses, and the direction of the logic]\label{rem:levelsetscope}
Corollary~\ref{cor:vacuity} is a conditional statement whose condition is an unknown functional of the joint law. It separates base-model inadequacy from a large value of $\mathcal N_A(\alpha)$, which is a property of the Texas crash-record distribution. Theorem~\ref{thm:levelset} also implies that the expected size of any \emph{valid} predictor is an \emph{upper} bound on $\mathcal N_A(\alpha)$. For a conformal predictor, whose coverage holds on average over the calibration draw, the statement applies to its size averaged over that draw. Agreement across a library of base models therefore cannot establish that the floor is high. A plug-in estimate of $\mathcal N_A(\alpha)$ supplies evidence about the data regime, as reported in Section~\ref{sec:results}, but it is not a distribution-free certificate. Nonparametric estimation of $p(\cdot\mid x)$ requires repeated covariate patterns. The recorded covariates are close to unique per record, so empirical label frequencies at exact covariate patterns are unavailable. Coarsening the covariates raises the floor and gives a bound in the opposite direction. The estimand is defined relative to $\sigma(X)$ for the recorded fields $X$, so an unrecorded field cannot lower this floor.
\end{remark}

\begin{remark}[Relation to Theorem~\ref{thm:oracle}]\label{rem:oraclevslevelset}
The two results quantify over different objects and neither implies the other. Theorem~\ref{thm:oracle} ranges over the raw threshold family $\mathcal F_{\mathrm{raw}}$ at a fixed $\hat F$ and identifies the smallest class-conditionally valid member, a statement about calibration given a model. Theorem~\ref{thm:levelset} ranges over all possibly randomized predictors based on $X$ and identifies a floor none reaches below, a statement about the data law and feature information. When $A$ is one of the classes of $\hat c$, write $\mathcal S^\star_A=\E[\,|C_{q_A}(X)|\mid\hat c(X)=A]$ for its term in \eqref{eq:Sstar}. Read together the two results then bracket the achievable, $\mathcal N_A(\alpha)\le\mathcal S^\star_A$, with the gap attributable to the fixed score family and the floor attributable to $P$.
\end{remark}

\begin{corollary}[Informativeness frontier]\label{cor:frontier}
For a declared informativeness target $m\in\{1,\dots,K-1\}$ define
\begin{equation}\label{eq:frontier}
\alpha_A^\star(m)=\inf\{\alpha\in(0,1):\ \mathcal N_A(\alpha)\le m\}.
\end{equation}
For $\alpha<\alpha_A^\star(m)$ no valid predictor on $A$ attains mean size at most $m$. Since $\alpha\mapsto\mathcal N_A(\alpha)$ is non-increasing, $\alpha_A^\star$ is well defined and non-increasing in $m$. The map $m\mapsto\alpha_A^\star(m)$ supports selection of a review level compatible with the available information, and it is what the empirical frontier of Section~\ref{sec:results} estimates.
\end{corollary}

\subsection{The floor cannot be certified from below}
\label{sec:noncert}

Theorem~\ref{thm:levelset} localizes informativeness in a single functional of the law, and Remark~\ref{rem:levelsetscope} observed that valid predictors bound $\mathcal N_A(\alpha)$ only from above. The result below gives a pointwise obstruction at one diffuse law for any uniformly valid lower-certificate procedure. It does not rule out informative lower output at every other law.

Let $\mathcal{P}$ denote the set of joint laws of $(X,Z)$ on $\X\times\Y$ whose covariate marginal is non-atomic and gives the fixed stratum $A$ positive probability, and let $D_n=((X_1,Z_1),\dots,(X_n,Z_n))$ be i.i.d.\ from $P\in\mathcal{P}$.

\begin{definition}[Lower certificate]\label{def:lowercert}
A \emph{$(1-\zeta)$ lower certificate} for the floor is a measurable map $\hat{L}_n:(\X\times\Y)^n\to[0,K]$ with
\begin{equation}\label{eq:lowercert}
\Prob_{P}\bigl(\hat{L}_n(D_n)\le\mathcal N_A(\alpha;P)\bigr)\ \ge\ 1-\zeta
\qquad\text{for every }P\in\mathcal{P} .
\end{equation}
It is distribution-free, so no $P$ in the class may be excluded, which is precisely the standard the rest of this paper holds itself to.
\end{definition}

\begin{theorem}[No uniform lower certificate for the oracle size benchmark; \tagspecialized{}; indistinguishability argument, \citealp{barber2020}]\label{thm:noncert}
Fix $K\ge 2$, $\alpha\in(0,1)$, $n\in\mathbb{N}$, $\zeta\in(0,1)$ and $\varepsilon>0$. Let $\hat{L}_n$ be any $(1-\zeta)$ lower certificate in the sense of Definition~\ref{def:lowercert}, and let $P_0\in\mathcal{P}$ be the law with $Z\perp X$ and $Z$ uniform on $\Y$, for which $\mathcal N_A(\alpha;P_0)=(1-\alpha)K$. Then
\begin{equation}\label{eq:noncert}
\Prob_{P_0}\bigl(\hat{L}_n(D_n)\le 1-\alpha\bigr)\ \ge\ 1-\zeta-\varepsilon .
\end{equation}
On the stated law whose floor is $(1-\alpha)K$, every uniformly valid lower certificate therefore returns a bound no larger than $1-\alpha$ with probability at least $1-\zeta-\varepsilon$. The certified value is then at most one $K$th of the true floor. The theorem does not claim that the same procedure must be uninformative at every $P\in\mathcal P$.
\end{theorem}

\begin{proof}
Because $P_{0,X}$ is non-atomic, for any $m$ there is a partition $\X=\bigsqcup_{j=1}^{m}B_j$ with $P_{0,X}(B_j)=1/m$ for every $j$. For $v=(v_1,\dots,v_m)\in\Y^m$ let $f_v(x)=v_j$ for $x\in B_j$, a measurable function, and let $Q_v\in\mathcal{P}$ be the law with the same covariate marginal $P_{0,X}$ and $Z=f_v(X)$ almost surely. Non-atomicity also gives a measurable set $H\subseteq A$ with $P_{0,X}(H\mid X\in A)=1-\alpha$. Under $Q_v$, the predictor $C(x)=\{f_v(x)\}$ for $x\in H$ and $C(x)=\varnothing$ otherwise has conditional coverage and expected size $1-\alpha$ on $A$. Conversely, coverage is no larger than expected set size, so $\mathcal N_A(\alpha;Q_v)=1-\alpha$. Applying \eqref{eq:lowercert} to each $Q_v$ and averaging over $V$ uniform on $\Y^m$ gives
\begin{equation}\label{eq:mixcert}
\E_V\Bigl[\Prob_{Q_V}\bigl(\hat{L}_n(D_n)\le 1-\alpha\bigr)\Bigr]\ \ge\ 1-\zeta .
\end{equation}
Let $\bar Q_m$ be the $V$-mixture of $Q_V^{\otimes n}$, which is the law of $D_n$ when $V$ is drawn once and $n$ points are then drawn from $Q_V$. Let $E$ be the event that the $n$ sampled covariates fall in $n$ distinct cells. On $E$ the labels $Z_i=V_{j(i)}$ involve $n$ distinct coordinates of $V$, which are i.i.d.\ uniform on $\Y$ and independent of the covariates; the covariates themselves have marginal $P_{0,X}$ under both laws. Hence $\bar Q_m$ and $P_0^{\otimes n}$ agree on $E$, and by a birthday bound $\Prob(E^c)\le n^2/(2m)$ under either, so
\begin{equation}\label{eq:tvbound}
\dtv\bigl(\bar Q_m,\ P_0^{\otimes n}\bigr)\ \le\ \frac{n^2}{2m}.
\end{equation}
Choosing $m\ge n^2/(2\varepsilon)$ and combining \eqref{eq:mixcert} with \eqref{eq:tvbound} gives \eqref{eq:noncert}. Finally $\mathcal N_A(\alpha;P_0)=(1-\alpha)K$, because under $P_0$ the conditional law is uniform for every $x$, so covering reported-label mass $1-\alpha$ costs expected size $(1-\alpha)K$, attained by including each label with probability $1-\alpha$ and unimprovable by any rule reaching that mass.
\end{proof}

\begin{remark}[Attribution]\label{rem:noncertattrib}
The mechanism underlying Theorem~\ref{thm:noncert} is not ours, and its provenance is worth setting out exactly. \citet{barber2020} proves that any distribution-free confidence interval for the conditional label probability in binary regression obeys a length lower bound that is independent of the sample size and holds for every distribution without point masses, so that no distribution-free procedure can adapt to structure in the law. That is the same obstruction, under the same non-atomicity condition, established by the same style of indistinguishability argument, and Theorem~\ref{thm:noncert} transports that argument to a different target. What is ours is the target and its consequence, since the estimand here is the level-set functional $\mathcal N_A(\alpha)$, a property of the whole conditional law on a stratum rather than of a single conditional probability, and the consequence is that the floor of Theorem~\ref{thm:levelset} is not certifiable from below, which is what closes the account opened in Remark~\ref{rem:levelsetscope}. The technique follows \citet{barber2020} and \citet{foygelbarber2021}; the statement for the level-set floor is the part this paper needs.
\end{remark}

\begin{remark}[What is and is not being claimed]\label{rem:noncertscope}
The theorem does not say the floor is unknowable in principle. It says that a lower-certificate procedure that is valid uniformly over $\mathcal P$ must be nearly trivial at $P_0$. The construction uses non-atomicity to make a diffuse conditional law close in finite samples to a mixture of deterministic labelings. Assumptions that tie $p(\cdot\mid x)$ across nearby $x$, such as smoothness or a parametric form, evade the construction and permit estimation. Informative output at a restricted law is therefore possible, but it is not a consequence of the distribution-free theorem.
\end{remark}

\begin{remark}[The empirical face of the obstruction]\label{rem:noncertempirical}
The Texas records are close to unique on the recorded feature vector. In the full 2017--2025 extract, the $54{,}872$ motorcyclist records carry $54{,}868$ distinct patterns across the $38$ recorded fields, with four patterns occurring twice and none more often. This pattern supports the practical relevance of sparse repeats. It does not establish that the population covariate law is non-atomic. Coarsening the covariates restores repeats but changes the estimand to a smaller $\sigma$-field.
\end{remark}

\subsection{Coverage on true severity under reporting noise}
\label{sec:noise}

Calibration uses reported labels $\yt$. Coverage for $Y$ is not identified from CRIS alone. We therefore state it only as a sensitivity result under a declared compatibility relation and declared error allowance.

\begin{assumption}[N($\Tmap,\delta$). Compatibility]\label{ass:noise}
There is a set-valued map $\Tmap:\Y\to 2^{\Y}$ (``plausible truths for a report'') satisfying $\tilde y\in\Tmap(\tilde y)$ for every $\tilde y\in\Y$, and $\delta\in[0,1)$ with
\begin{equation}\label{eq:band}
\Prob\bigl(Y_{n+1}\in\Tmap(\yt_{n+1})\bigr)\ \ge\ 1-\delta.
\end{equation}
Reflexivity is part of the declared relation. It gives $\Ct\subseteq\Cop$, which Corollary~\ref{cor:fatal} uses, but it is not estimated from the Texas data.
\emph{Banded special case} N($b^-,b^+,\delta$):
\begin{equation}\label{eq:bandmap}
\Tmap(\tilde y)=\{y\in\Y:\ \tilde y-b^+\le y\le\tilde y+b^-\}.
\end{equation}
Here $b^+$ guards \emph{over}-reporting (medically referenced severity below the report; the set reaches downward) and $b^-$ guards \emph{under}-reporting (medically referenced severity above the report; the set reaches upward). Under-reporting is the safety-critical direction. The candidate KABCO analysis sets $b^-=b^+=1$ and treats $\delta$ as a declared sensitivity value. External KABCO--MAIS studies inform the candidate maps and sensitivity range, but they do not estimate a Texas reporting law. Remark~\ref{rem:catdep} gives category-dependent alternatives.
\end{assumption}

The assumption concerns only the test point's joint law $(Y_{n+1},\yt_{n+1})$; calibration labels are never de-noised because calibration only ever touches $\yt$.

\begin{definition}[Compatibility expansion]\label{def:expansion}
For an interval $\Ct(x)=[\tilde a(x),\tilde b(x)]$ produced by any procedure of Sections~\ref{sec:score}--\ref{sec:het}, let
\begin{equation}\label{eq:expansion}
\Cop(x)=\bigcup_{\tilde y\in\Ct(x)}\Tmap(\tilde y).
\end{equation}
In the banded case
\begin{equation}\label{eq:bandexp}
\Cop(x)=\bigl[\tilde a(x)-b^+,\ \tilde b(x)+b^-\bigr]\cap\Y,
\end{equation}
again an interval.
\end{definition}

\begin{theorem}[Coverage transfer under a declared compatibility map; \tagspecialized{}; \citealp{cauchois2022,stutz2023}]\label{thm:noise}
Let $\Ct$ be any set-valued predictor with $\Prob(\yt_{n+1}\in\Ct(X_{n+1}))\ge 1-\alpha$ (Proposition~\ref{prop:marginal}), or $\ge 1-\alpha$ conditionally on a class or stratum (Theorem~\ref{thm:het}) with N($\Tmap,\delta$) holding conditionally. Then under N($\Tmap,\delta$),
\begin{equation}\label{eq:noise}
\Prob\bigl(Y_{n+1}\in\Cop(X_{n+1})\bigr)\ \ge\ 1-\alpha-\delta.
\end{equation}
\end{theorem}

\begin{proof}
On $A=\{\yt_{n+1}\in\Ct(X_{n+1})\}\cap\{Y_{n+1}\in\Tmap(\yt_{n+1})\}$ we have $Y_{n+1}\in\Tmap(\tilde y)$ for some $\tilde y\in\Ct(X_{n+1})$, hence $Y_{n+1}\in\Cop(X_{n+1})$. And $\Prob(A)\ge\Prob(\yt_{n+1}\in\Ct)-\Prob(Y_{n+1}\notin\Tmap(\yt_{n+1}))\ge 1-\alpha-\delta$. The conditional version is identical with all probabilities conditional.
\end{proof}

The proof is short because the work is done by the \emph{banded} compatibility map, which is crash-reporting structure. The non-vacuity comes from the declared support restriction $\Tmap(\tilde y)\subsetneq\Y$; ordinal structure is what makes a narrow, contiguous restriction of this kind plausible for KABCO. Transferring conformal coverage from a noisy or weakly supervised label to the clean one is not itself new \citep{cauchois2022,stutz2023,einbinder2024}; what the ordinal band contributes is that the transfer costs an \emph{additive} $b^-+b^+$ categories rather than degrading multiplicatively or by the full error rate, a distinction made precise against the total-variation baseline in Proposition~\ref{prop:tightness} and Remark~\ref{rem:einbinder}.

\begin{remark}[What contiguity does and does not buy]\label{rem:contigrole}
Theorem~\ref{thm:noise} holds for \emph{any} set-valued $\Ct$; its proof never uses contiguity, and $\Cop=\bigcup_{\tilde y\in\Ct}\Tmap(\tilde y)$ is defined for arbitrary $\Ct$. Ordinality makes $\Tmap(\tilde y)$ a \emph{proper} subset of $\Y$. Contiguity supplies separate gains in cost and interpretation. The set $\Cop$ remains an interval, the expansion cost is additive, and its two endpoints support a prespecified decision rule. Contiguity alone does not make the interval an upper tail.
\end{remark}

\begin{theorem}[Sharpness of the compatibility-map bound; \tagderived{}]\label{thm:sharp}
Fix the banded map $\Tmap(\tilde y)=[\tilde y-b^+,\tilde y+b^-]$.
\textbf{(i) Tightness of the constant.} For every $\alpha\in(0,1)$, $\delta\in[0,\alpha\wedge(1-\alpha))$, $b^\pm$ with $K\ge b^-+2$, and every $\eta>0$, there exists a joint law of $(X,Y,\yt)$ satisfying A1 and N($b^-,b^+,\delta$) under which
\begin{equation}\label{eq:sharplimit}
\lim_{n\to\infty}\Prob\bigl(Y_{n+1}\in\Cop\bigr)\ \le\ 1-\alpha-\delta+\eta;
\end{equation}
hence no constant larger than $1-\alpha-\delta$ can replace the bound of Theorem~\ref{thm:noise}.

\textbf{(ii) Necessity of both expansion widths.} Fix $\alpha\in(0,1/2)$. For the upper-trim statement assume $1\le b^-\le K-1$. For the lower-trim statement assume $1\le b^+\le K-1$. These conditions make the trimmed rules well defined. For the rule that expands upward only by $b^--1$ and outputs $[\tilde a-b^+,\ \tilde b+b^--1]$, there is a law satisfying N($b^-,b^+,0$) with asymptotic true-label coverage $\le\alpha<1-\alpha$. Symmetrically, for the rule that expands downward only by $b^+-1$, there is a law satisfying N($b^-,b^+,0$) with the same result. Thus, both expansion widths are necessary for the guarantee.
\end{theorem}

\begin{proposition}[Structure buys tightness; \tagadapted{}; (i) cf.\ \citealp{einbinder2024}, Cor.~3]\label{prop:tightness}
Suppose only $\Prob(\yt\ne Y)\le\varepsilon_{\mathrm{tot}}$ is known, without a band. The following statements hold.
(i) the best transfer without expansion is
\begin{equation}\label{eq:notransfer}
\Prob(Y\in\Ct)\ \ge\ 1-\alpha-\varepsilon_{\mathrm{tot}},
\end{equation}
and no label-space expansion whose output $\bigcup_{\tilde y\in\Ct}D(\tilde y)$ remains a proper subset of $\Y$ improves the constant, since without support restrictions an adversarial coupling puts the off-event mass outside that union;
(ii) under the banded assumption, Theorem~\ref{thm:noise} attains $1-\alpha-\delta$ at the cost of at most $b^-+b^+$ added categories. In the KABCO regime elicited from linkage studies ($\varepsilon_{\mathrm{tot}}$ of order $0.3$--$0.5$, dominated by adjacent A/B/C confusion; beyond-band mass $\delta$ of order $0.01$--$0.05$), the guarantee improves from $1-\alpha-\varepsilon_{\mathrm{tot}}$ (weak, and at the upper end near-uninformative at $\alpha=0.1$) to $1-\alpha-\delta$ (near-nominal).
\end{proposition}

\begin{remark}[Attribution of Proposition~\ref{prop:tightness}]\label{rem:einbinder}
Part (i) is known. \citet{einbinder2024} obtain the same
degradation from a purely structural premise, assuming only a total-variation bound
between the clean and noisy label laws and correcting the nominal level accordingly; up
to the constant, their Corollary~3 \emph{is} \eqref{eq:notransfer}. We state it here
because it is the baseline Theorem~\ref{thm:noise} must beat, and the comparison in (ii)
is the point, because their premise and ours are both structural, but theirs is indexed by the
total error rate $\varepsilon_{\mathrm{tot}}$ and ours by the beyond-band mass $\delta$.
In the KABCO regime those differ by an order of magnitude, which is what the band buys. Their total-variation premise is at most $\Prob(\yt\ne Y)$ and can be smaller, so comparing it with $\varepsilon_{\mathrm{tot}}$ uses an upper bound on their premise, and the size of the gap should be read in that light.
The adversarial-coupling clause in (i) is a short addition of ours.
\end{remark}

\begin{proof}
(i) is the union bound of Theorem~\ref{thm:noise} with $\Tmap(\tilde y)=\{\tilde y\}$ on a $1-\varepsilon_{\mathrm{tot}}$ event, plus the stated adversarial coupling. (ii) is Theorem~\ref{thm:noise} plus arithmetic; the KABCO magnitudes are cited (Section~\ref{sec:noise-example}), not proved.
\end{proof}

\subsubsection{Numeric grounding of the band}
\label{sec:noise-example}

KABCO--MAIS linkage studies provide external evidence about police-report disagreement with medically assessed severity. In the Wisconsin CODES linkage (2008--2012; $120{,}534$ linked crash victims), overall agreement between the police KABCO rating and the MAIS-based reference category is $51\%$, so $\varepsilon_{\mathrm{tot}}\approx 0.49$ \emph{in that linked sample}; victims rated B, C, or O whose reference severity is serious (MAIS~3+) constitute $2.9\%$ of all cases, and the reported per-category beyond-adjacent under-reporting proportions are $0.4\%$ (O), $1.7\%$ (C), and $2.0\%$ (B) \citep{burdett2015,burdett-thesis}. In the cited Wisconsin sample, all K-coded victims were mapped to MAIS~6; that sample proportion does not prove exact K recording in Texas or future data \citep{farmer2003}. Over-reporting from A reaches two categories down in the cited data ($38.2\%$ of A-coded victims at MAIS~1), which motivates a category-dependent candidate map. Linked samples over-represent medically attended victims, and MAIS is a distinct scale mapped to KABCO by a declared correspondence. Therefore every map and $\delta$ used here is a sensitivity input, not an estimated or validated Texas reporting channel.

\begin{remark}[Is the band minimal, or is the saturation self-inflicted?]\label{rem:bandminimal}
A band of total width $b^-+b^+=2$ saturates a five-point scale from any base interval of width $3$ that touches neither end of the scale (intervals that include O or K grow by less), so a reader is entitled to ask whether the uninformativeness reported in Section~\ref{sec:results} is a property of the data or of our own declaration. The two directions answer differently, and both answers are needed.

Upward, the band is a requirement rather than a convenience. Setting $b^-=0$ places every adjacent under-reporting event outside the band, and the linkage evidence puts essentially all of the $49\%$ disagreement there, so $\delta$ would rise to roughly that value and the floor $1-\alpha-\delta$ to roughly $0.41$. Against this, $b^-=1$ leaves only the beyond-adjacent mass, which weighted by the reported KABCO composition of the test split is $0.822(0.004)+0.099(0.017)+0.062(0.020)\approx 0.0062$. One category upward buys about two orders of magnitude in $\delta$, and no certificate worth issuing survives without it.

Downward, the cited evidence leaves the choice open and does not force $b^+=1$. The Wisconsin linkage quantifies over-reporting only for category A, at $38.2\%$ reaching two categories down, which contributes $0.014(0.382)\approx 0.0053$ of beyond-band mass under a constant unit band and brings the total to $\approx 0.0115$, comfortably inside the declared $\delta=0.02$. It reports no adjacent over-reporting rate for O, C, or B, so a deployment could defensibly declare $b^+=0$ and absorb the over-reporting into a slightly larger $\delta$.

That trim is the sharpest one the evidence permits, and it does not change the finding. With $b^-=1$ and $b^+=0$ the expansion adds one category rather than two, and the class-conditional base widths on the motorcyclist stratum, which lie between $3.55$ and $4.22$ across all seven base models of Section~\ref{sec:results}, compose to mean widths of at most $4.55$ to $5$, the upper values because truncation at the scale ends can only reduce them. Every one remains at or above width $4$ on a scale of five. The rural high-speed stratum depends on the base model. Under DLCON it sits at $1.95$ and composes to $2.95$, which is not vacuous under either band. Under histogram gradient boosting it already sits at $4.65$ before expansion (Table~\ref{tab:maincells}), and the composed sets in Table~\ref{tbl:compose} are almost all full scale. The band and the base sets therefore act together. On the motorcyclist and unrestrained strata the base sets are already wide for every model, and any unit expansion saturates them; on the rural high-speed stratum saturation depends on the base model. Section~\ref{sec:e8e9} reports the composition audit for both a wide and a narrow base model.
\end{remark}

\begin{remark}[Known-$\Pi$ alternatives]\label{rem:knownpi}
If a full misclassification kernel $\Pi(\tilde y\mid y)$ were known and stable, quantile-correction approaches \citep{einbinder2024,sesia2023,noiseadaptive2025} could yield smaller sets. This paper does not establish such stability, including stability of the support band, for Texas jurisdictions and years. It instead reports sensitivity results over declared candidate maps and values of $\delta$. The known-$\Pi$ route is complementary.
\end{remark}

\begin{remark}[Category-dependent bands and a declared K boundary]\label{rem:catdep}
Replace the constant band by a category-dependent declared map, for example $\Tmap(5)=\{5\}$ and $5\notin\Tmap(\tilde y)$ for $\tilde y\le 4$ in an analysis that assumes exact fatality recording after record amendment, and $\Tmap(4)\supseteq\{2,3,4\}$ when two-level over-reporting from A is admitted. Theorem~\ref{thm:noise} holds because its proof does not use a constant band. The smaller expansion at the K boundary is conditional on this declared map; it is not identified by CRIS.
\end{remark}

\subsection{An assumption-conditional floor under a declared reporting channel}
\label{sec:channelfloor}

Theorem~\ref{thm:noncert} gives a distribution-free obstruction. This subsection asks what follows if a reporting channel and a within-stratum nondifferentiality condition are declared. Under those inputs, a genie-aided converse gives a computable lower bound. The result is conditional on the declared admissible channel set; the Texas data do not identify the channel. The argument is related to list-size lower bounds in coding \citep{guruswami2010} and to partial-identification analysis \citep{molinari2008}. Oracle prediction sets under interval censoring \citep{liu2025interval} address a different problem.

Write $M(t\mid y)=\Prob(\yt=t\mid Y=y)$ for the channel taking medically referenced severity $Y$ to the reported label $\yt$. On a stratum $A\in\sigma(X)$ with $\Prob(X\in A)>0$, let $p_A(y)=\Prob(Y=y\mid X\in A)$ be the latent composition and $\tilde p_A(t)=\sum_y M(t\mid y)\,p_A(y)$ the observed reported composition. The theorem below assumes nondifferential reporting within the stratum, with $M(t\mid y,x)=M_A(t\mid y)$ for $x\in A$. This condition is weaker than global nondifferentiality, but it remains a declared assumption that the Texas records do not directly validate.

\begin{definition}[Channel floor]\label{def:channelfloor}
For inclusion weights $c(t\mid y)\in[0,1]$,
\begin{equation}\label{eq:channelfloor}
\mathcal N^{\ast}_A(\alpha)=\min_{c}\ \sum_y p_A(y)\sum_t c(t\mid y)
\quad\text{s.t.}\quad \sum_y p_A(y)\sum_t c(t\mid y)\,M_A(t\mid y)\ \ge\ 1-\alpha .
\end{equation}
This is a fractional knapsack, in which one orders the cells $(y,t)$ by $M_A(t\mid y)$ and includes them greedily until the coverage constraint binds, with one fractional boundary cell. It is the smallest expected set size for covering the reported label at level $1-\alpha$ by a predictor that knows the true label and may randomize; no contiguity is imposed. A value $\mathcal N^{\ast}_A(\alpha)<1$ uses rules that can return the empty set. If non-empty output is required, every feasible rule has expected size at least $\max\{1,\mathcal N^{\ast}_A(\alpha)\}$, but equality is not asserted. The deterministic fallback in Definition~\ref{conv:nonempty} changes the feasible predictor and must be analyzed as a deployed superset, not inserted into the fractional program.
\end{definition}

\begin{theorem}[Channel floor on expected set size; \tagadapted{}; level-set optimality \citep{sadinle2019} composed with a genie reduction under within-stratum nondifferential reporting]\label{thm:channel}
Under within-stratum nondifferential reporting, any $\sigma(X)$-measurable, possibly randomized set-valued predictor $C$ with reported-label coverage $\Prob(\yt\in C(X)\mid X\in A)\ge 1-\alpha$ satisfies
\begin{equation}\label{eq:channelbound}
\E\bigl[\,|C(X)|\ \bigm|\ X\in A\,\bigr]\ \ge\ \mathcal N^{\ast}_A(\alpha).
\end{equation}
$\mathcal N^{\ast}_A(\alpha)$ depends only on the declared stratum channel $M_A$ and latent composition $p_A$; no base model, feature set, or calibration scheme enters. Its numerical value is assumption-conditional when these inputs are not identified.
\end{theorem}

\begin{proof}
Write $N(\mathcal F)$ for the value of \eqref{eq:channelfloor}'s minimization over $\mathcal F$-measurable inclusion probabilities with the same marginal coverage constraint. Any valid possibly randomized $C$ gives $c(t\mid x)=\Prob_U\{t\in C(x,U)\}$, so $\E[|C(X,U)|\mid A]\ge N(\sigma(X))$. Since $\sigma(X)\subseteq\sigma(X,Y)$, the larger feasible set gives $N(\sigma(X))\ge N(\sigma(X,Y))$. Under within-stratum nondifferentiality, the reported law given $(X,Y)=(x,y)$ is $M_A(\cdot\mid y)$. For any $\sigma(X,Y)$-measurable $c(t\mid x,y)$, form $\bar c(t\mid y)=\E[c(t\mid X,Y)\mid Y=y,\,X\in A]$. Linearity preserves the objective and constraint, so $N(\sigma(X,Y))=N(\sigma(Y))=\mathcal N^{\ast}_A(\alpha)$. This proves \eqref{eq:channelbound}. A genie that observes $Y$ and uses auxiliary uniforms to include label $t$ with probability $c^\star(t\mid Y)$ attains $\mathcal N^{\ast}_A(\alpha)$; an ordinary $X$-based predictor need not attain this lower bound.
\end{proof}

\begin{corollary}[Channel-imposed vacuity]\label{cor:channelvacuity}
$\mathcal N^{\ast}_A(\alpha)>1$ only if $\sum_y p_A(y)\max_t M_A(t\mid y)<1-\alpha$, that is, only if no single reported category covers the stratum's reported label at $1-\alpha$ even knowing the true label. The converse can fail, because a rule may spend width two on a peaked column and width zero on another, so that width-one coverage falls short yet the LP value is below one (the empty-set budget of Definition~\ref{def:channelfloor}). When the channel does force $\mathcal N^{\ast}_A(\alpha)>1$, the observed width on $A$ decomposes as
\begin{equation}\label{eq:channeldecomp}
\underbrace{\mathcal N^{\ast}_A(\alpha)}_{\text{channel floor}}\ \le\ \underbrace{\text{irreducible }\sigma(X)\text{-floor}}_{\le\,\text{observed width}} .
\end{equation}
The middle term is not point-identified. For each fixed admitted channel and composition, $\mathcal N^{\ast}_A(\alpha)$ is a valid lower bound under the theorem's assumptions. Across a declared sensitivity set, the smallest such value is an assumption-conditional lower endpoint. Wide sets can be attributed to the reporting channel only within that sensitivity analysis; the Texas data alone do not identify this mechanism.
\end{corollary}

\begin{remark}[Differential reporting]\label{rem:channeldegrade}
If reporting is differential within $A$ with $\sup_{x\in A,\,y}\dtv\bigl(M(\cdot\mid y,x),M_A(\cdot\mid y)\bigr)\le\nu$, then any predictor feasible under the true process has surrogate-channel coverage at least $1-\alpha-\nu$, so $\E[|C(X)|\mid A]\ge\mathcal N^{\ast}_A(\alpha+\nu)$. The floor moves continuously in the differential. One consequence is diagnostic, because a valid predictor achieving width below $\mathcal N^{\ast}_A(\alpha)$ using scene fields coded by the same officer who assigns the label is evidence of rater-mediated dependence, not a contradiction.
\end{remark}

\begin{remark}[Declared channel scenarios and specification check; \tagderived{}]\label{rem:channelscope}
Computing $\mathcal N^{\ast}_A$ requires $M_A$ and $p_A$, neither of which is observed directly. The linkage constants of Section~\ref{sec:noise-example} are marginal rates from external selected samples rather than a Texas channel. Guided by the direct-misclassification approach of \citet{molinari2008}, the empirical analysis evaluates a finite grid of 900 declared channel scenarios. Each scenario specifies per-category agreement, disagreement direction, cross-category leakage, and a reported-label selection index $\varrho\in\{1,2,5\}$. The case $\varrho=1$ represents selection invariance. For each fixed scenario, latent composition is allowed to vary over the simplex subject to reproducing every reported KABCO share within its rounding interval of $\pm 0.00005$. The minimum floor is obtained exactly by solving the 25 possible fractional-knapsack boundary cases as linear programs. The reported value is the smallest exact floor over all feasible scenarios in the declared finite set. It is not an endpoint for a continuous class of channels or selection functions, and it is not an identified Texas quantity. An achieved valid width below the floor implied by a retained scenario refutes that scenario, up to sampling error in the achieved width and the averaging of conformal validity over calibration draws. This calculation is a specification check within the declared sensitivity analysis and does not identify the reporting process.
\end{remark}

\subsection{Spatiotemporal deployment shift}
\label{sec:shift}

A raw deployment cell can combine the fitted class, time bin, and spatial unit. Let $h:\X\to\mathcal H$ be this raw-cell rule. Before calibration starts, a deterministic hierarchy uses \emph{training-split covariate counts} to construct a map $r:\mathcal H\to\mathcal R$. Define the final calibration cell
\begin{equation}\label{eq:finalcell}
R(x)=r\bigl(h(x)\bigr).
\end{equation}
The inverse images $r^{-1}(a)$ form one disjoint tree cut. If a sparse child is rolled to a parent, all siblings assigned to that parent use the same final cell; a leaf and its ancestor cannot both be final cells. The map $R$ is frozen before calibration labels, scores, and support counts are used.

\begin{theorem}[Observed final cells with group-conditional validity and rollup; \tagadapted{}]\label{thm:groups}
For each final cell $a\in\mathcal R$, use only calibration scores with $R(X_i)=a$ and form the conformal quantile with the usual finite-sample index. For every final cell with positive test probability, the raw predictor $C^{\mathrm{raw}}_a$ and its deployed non-empty version $C^{\mathrm{ne}}_a$ satisfy
\begin{equation}\label{eq:stratcov}
\Prob\bigl(\yt_{n+1}\in C^{\mathrm{raw}}_a(X_{n+1})\mid R(X_{n+1})=a\bigr)\ \ge\ 1-\alpha,
\end{equation}
with the same lower bound for $C^{\mathrm{ne}}_a$. Hence, for a test point drawn independently of the calibration sample from any pure reweighting of the final-cell laws,
\begin{equation}\label{eq:mixture}
P_{\mathrm{test}}=\sum_a\omega_a P_a
\quad\Longrightarrow\quad
P_{\mathrm{test}}(\yt\in C)\ \ge\ 1-\alpha.
\end{equation}
The theorem conditions on $R=a$, the final cell actually used for calibration. It gives no guarantee conditional on a proper raw child of $h$ inside that final cell.
\end{theorem}

\begin{proof}
The training-frozen $R$ is a fixed measurable function. Apply the conditional rank argument of Theorem~\ref{thm:het} to the partition $R$. Pointwise containment transfers the lower bound from the raw set to the deployed set. The mixture statement is the tower property. Nothing in this proof conditions on an original leaf after rollup.
\end{proof}

\emph{Transport reading.} A change only in mixture weights across unchanged final-cell laws preserves marginal coverage. A change within a final-cell law or a genuinely new stratum is outside this theorem.

\paragraph{Choice of $n_{\min}$.} We set
\begin{equation}\label{eq:nmin}
n_{\min}=\lceil 2/\alpha\rceil\cdot 50
\end{equation}
(e.g.\ $n_{\min}=1000$ at $\alpha=0.1$). This is a training-support threshold used only to freeze a reproducible final partition. If a final cell also has at least $n_{\min}$ calibration observations, its quantile granularity is at most $1/(n_{\min}+1)$ and the corresponding binomial audit standard error at coverage $1-\alpha$ is at most $\sqrt{\alpha(1-\alpha)/n_{\min}}$. These numerical audit properties must be checked from the actual calibration counts. They are not validity requirements.

\begin{theorem}[Weighted conformal coverage with a total-variation remainder; \tagadapted{}; \citealp{tibshirani2019}]\label{thm:transfer}
Let the target stratum $g^\ast$ represent a new county or a future year and satisfy the covariate-shift condition. Its $X$-marginal is $P^\ast_X$, and its conditional distribution is $P_{\mathrm{cal}}(\yt\mid X)$. Let $w(x)=\mathrm dP^\ast_X/\mathrm dP_{\mathrm{cal},X}(x)$ exist. Fit $\hat w\ge0$ by contrasting unlabeled target covariates with covariates drawn from the \emph{calibration reference law} $P_{\mathrm{cal},X}$ rather than the training law. For the finite-sample statement, the ratio-fitting samples are independent of the scored calibration fold. A cross-fit construction with independent scored folds provides the same separation. Conditional on the fitted $\hat w$, run weighted split conformal.
\begin{equation}\label{eq:wq}
\hat q_{\hat w}=\inf\Bigl\{t\mid \sum_{i=1}^n p^{\hat w}_i\,\ind\{S_i\le t\}\ \ge\ 1-\alpha\Bigr\},
\qquad
p^{\hat w}_i=\frac{\hat w(X_i)}{\sum_{j=1}^n\hat w(X_j)+\hat w(X_{n+1})},
\end{equation}
where the remaining mass $p^{\hat w}_{n+1}=\hat w(X_{n+1})/\bigl(\sum_j\hat w(X_j)+\hat w(X_{n+1})\bigr)$ is placed at $+\infty$ \citep{tibshirani2019}; in particular $\hat q_{\hat w}=+\infty$ (i.e.\ $C=\Y$) whenever the calibration mass cannot reach $1-\alpha$.
Define the estimated tilt $Q_{\hat w}$ on $\X$ by
\begin{equation}\label{eq:tilt}
\mathrm dQ_{\hat w}\ \propto\ \hat w\,\mathrm dP_{\mathrm{cal},X}.
\end{equation}
The following statements hold.

\textbf{(i) Slack identity.}
\begin{equation}\label{eq:slack}
P_{g^\ast}\bigl(\yt_{n+1}\in C_{\hat q_{\hat w}}(X_{n+1})\bigr)\ \ge\ 1-\alpha-\dtv\bigl(P^\ast_X,\,Q_{\hat w}\bigr),
\end{equation}
and the slack is a covariate-space population quantity involving no labels and no unknown true density ratio.

\textbf{(ii) Assumption-free discrepancy diagnostic.} For any classifier $\phi$ and any reference distribution $Q$ on $\X$, the balanced accuracy for distinguishing independent samples from $P^\ast_X$ and $Q$ satisfies
\begin{equation}\label{eq:ba}
2\,\mathrm{ba}(\phi)-1\ \le\ \dtv(P^\ast_X,Q).
\end{equation}
An exact independent sampler from $Q_{\hat w}$ would make a sample-split Clopper--Pearson calculation a lower confidence bound on the slack in \eqref{eq:slack}. The implemented diagnostic instead resamples with replacement from the finite weighted calibration reference
\begin{equation}\label{eq:empirical-tilt}
\widehat Q_{\hat w,n}=\sum_{i=1}^{n}\frac{\hat w(X_i)}{\sum_{j=1}^{n}\hat w(X_j)}\,\delta_{X_i}.
\end{equation}
Conditional on the calibration covariates, the same calculation is an exact lower confidence bound on $\dtv(P^\ast_X,\widehat Q_{\hat w,n})$. Because $\widehat Q_{\hat w,n}$ is discrete and $P^\ast_X$ has no atoms, that distance equals one, so the bound is valid but says nothing about its target. What the calculation measures in practice is the balanced accuracy of a fixed, low-capacity discriminator, which is a two-sample statistic for covariate mismatch that such a discriminator can detect. We report it only in that role, as a classifier mismatch statistic. It is not a confidence bound on $\dtv(P^\ast_X,Q_{\hat w})$ and is not inserted into \eqref{eq:slack}. A large value shows detectable mismatch against the weighted reference; a value near zero shows only that the discriminator finds none, and does not establish target coverage.

\textbf{(iii) Upper bound under realizability.} Suppose $w=w_{\theta_0}$ for an interior and locally identifiable parameter $\theta_0$. Assume the associated probabilistic-classifier likelihood is twice continuously differentiable near $\theta_0$, has nonsingular information, and gives $\hat\theta-\theta_0=O_p((n_{\mathrm{fit}}\wedge n_{\mathrm{ref}})^{-1/2})$, where $n_{\mathrm{fit}}$ is the size of the target ratio-fitting sample and $n_{\mathrm{ref}}$ the size of the calibration-law reference sample. Also assume $\E_{P_{\mathrm{cal}}}w_\theta$ is bounded away from zero near $\theta_0$ and that an integrable function $g$ satisfies $|w_\theta(x)-w_{\theta_0}(x)|\le g(x)\|\theta-\theta_0\|$ in that neighborhood. Then $\dtv(P^\ast_X,Q_{\hat w})=O_p\bigl((n_{\mathrm{fit}}\wedge n_{\mathrm{ref}})^{-1/2}\bigr)$.
\end{theorem}

\paragraph{Scope of Theorem~\ref{thm:transfer}.}
(a) The theorem assumes covariate shift, including stability of the reported-label conditional law. A labeled target batch can audit this premise after labels arrive. The frozen-reference score in Algorithm~\ref{alg:monitor} is only exploratory and does not test the premise with exact anytime error control. (b) The population slack in \eqref{eq:slack} is distinct from the classifier mismatch statistic of part (ii). An assumption-free finite-sample upper bound on total variation from unlabeled samples alone is unavailable, and part (iii) gives only a rate, so no numerical slack is computed. The exported tuple reports the nominal level and the mismatch statistic as separate quantities. Weight clipping at $w_{\max}$ takes the fitted ratio outside a realizable family, so part (iii) applies only to the unclipped fit. (c) The covariate-shift condition requires $P^\ast_X$ to be absolutely continuous with respect to $P_{\mathrm{cal},X}$. Fields that identify the target itself would violate this; neither county identity nor calendar year is a model covariate (Section~\ref{sec:features}), so the ratio compares covariate laws that can overlap.

\begin{algorithm}[t]
\caption{New-stratum weighted calibration and the exported diagnostic tuple from Theorem~\ref{thm:transfer}}
\label{alg:transfer}
\begin{algorithmic}[1]
\Require scored calibration fold $\{(S_i,X_i)\}_{i\in\Ical}$; an independent calibration-law covariate reference sample $\mathcal D_{\mathrm{ref}}$; unlabeled target covariates from $g^\ast$; level $\alpha$; one-sided confidence level $1-\alpha_{\mathrm{d}}$ for the empirical discrepancy bound; weight cap $w_{\max}$
\State split the unlabeled target sample into a ratio-fitting half $\mathcal D_{\mathrm{fit}}$ and a diagnostic half $\mathcal D_{\mathrm{diag}}$
\State fit a probabilistic classifier to ($\mathcal D_{\mathrm{ref}}=0$) vs ($\mathcal D_{\mathrm{fit}}=1$); set
$\hat w(x)\gets\min\bigl\{\tfrac{\hat p(1\mid x)}{\hat p(0\mid x)}\cdot\tfrac{n_0}{n_1},\ w_{\max}\bigr\}$ \Comment{any $\hat w\ge 0$ is admissible; the slack in \eqref{eq:slack} absorbs estimation error and clipping}
\For{each test point $x$ from $g^\ast$}
  \State $\hat q_{\hat w}(x)\gets$ weighted quantile \eqref{eq:wq} with test weight $\hat w(x)$; predict $C_{\hat q_{\hat w}(x)}(x)$ \Comment{Thm.~\ref{thm:transfer}(i)}
\EndFor
\State draw with replacement from $\widehat Q_{\hat w,n}$ in \eqref{eq:empirical-tilt}
\State train a discriminator $\phi$ on one half of $\mathcal D_{\mathrm{diag}}$ vs one half of the resample; compute the balanced accuracy $\mathrm{ba}(\phi)$ on the two held-out halves \Comment{sample split, so $\phi$ never sees its evaluation data}
\State $\widehat\Delta_{\mathrm{emp}}^{\mathrm{LCB}}\gets\max\bigl\{0,\ 2\,\mathrm{ba}_{\mathrm{LCB}}(\phi)-1\bigr\}$, with $\mathrm{ba}_{\mathrm{LCB}}$ a one-sided level-$(1-\alpha_{\mathrm{d}})$ lower confidence bound formed by a Bonferroni split of $\alpha_{\mathrm{d}}$ across exact (Clopper--Pearson) bounds on the two class-wise binomial accuracies
\State \textbf{emit} the diagnostic tuple $\bigl(1-\alpha,\ \widehat\Delta_{\mathrm{emp}}^{\mathrm{LCB}}\bigr)$
\end{algorithmic}
\end{algorithm}

\paragraph{Exploratory drift monitoring.}
The covariate-shift condition of Theorem~\ref{thm:transfer} is not label-free falsifiable. For each deployed record $t$ with score $s_t$, one can compute the smoothed rank against the frozen calibration scores,
\begin{equation}\label{eq:pvalue}
p_t=\frac{\#\{i\in\Ical\mid S_i>s_t\}+U_t\bigl(\#\{i\in\Ical\mid S_i=s_t\}+1\bigr)}{n+1},\qquad U_t\sim\mathrm{Unif}(0,1)\ \text{i.i.d.},
\end{equation}
and the corresponding power-product score
\begin{equation}\label{eq:martingale}
M_t=\prod_{s\le t}\varepsilon\,p_s^{\varepsilon-1},\qquad \varepsilon\in(0,1).
\end{equation}
Under a separately specified online conformal protocol that produces sequentially valid randomized ranks, $M_t$ is a non-negative martingale and Ville's inequality gives an exact anytime bound (\tagknown{}; \citealp{vovk-martingales}). Algorithm~\ref{alg:monitor}, however, reuses one frozen calibration set. Its $p_t$ share that random reference and are dependent \citep{bates2023}; the displayed product is not established as a martingale or e-process. Therefore the frozen rule has no claimed Ville bound. Its threshold is an operational trigger for review and recalibration, not a calibrated false-alarm budget, and it does not change an issued certificate. In the empirical audit the product is restarted at the start of each calendar month, with $\varepsilon=0.5$ and $b=100$ ($\log_{10}b=2$).

\begin{algorithm}[t]
\caption{Exploratory frozen-reference drift score (no anytime error guarantee)}
\label{alg:monitor}
\begin{algorithmic}[1]
\Require frozen calibration scores $\{S_i\}_{i\in\Ical}$; betting parameter $\varepsilon\in(0,1)$; prespecified review threshold $b>1$
\State $M\gets 1$
\For{each deployed record $t=1,2,\dots$}
  \State if $t$ opens a new audit period (a calendar month here), set $M\gets 1$
  \State compute $s_t$ and the smoothed p-value $p_t$ by \eqref{eq:pvalue}
  \State $M\gets M\cdot\varepsilon\,p_t^{\varepsilon-1}$ \Comment{running product \eqref{eq:martingale}}
  \If{$M\ge b$}
    \State \textbf{review trigger}: freeze new certificates and assess recalibration \Comment{exploratory; no Ville guarantee}
  \EndIf
\EndFor
\end{algorithmic}
\end{algorithm}

\subsection{Severity-cost risk control}
\label{sec:risk}

Let $\kappa:\Y\to[0,\infty)$ be non-decreasing with $\kappa_{\max}=\kappa(K)>0$ (USDOT/FHWA comprehensive crash-cost scale; the domain admits the indicator cost of Corollary~\ref{cor:fatal}).

\begin{theorem}[Cost-weighted false-omission control; \tagadapted{}; conformal risk control, \citealp{angelopoulos2024}]\label{thm:crc}
Define the loss
\begin{equation}\label{eq:loss}
\ell_\lambda(x,\tilde y)=\frac{\kappa(\tilde y)}{\kappa_{\max}}\,\ind\{\tilde y\notin C_\lambda(x)\}\ \in\ [0,1].
\end{equation}
Then $\lambda\mapsto\ell_\lambda$ is non-increasing and right-continuous with $\ell_1\equiv 0$ (as $s\le 1$ pointwise implies $C_1=\Y$). With
\begin{equation}\label{eq:lambdahat}
\hat\lambda=\inf\Bigl\{\lambda:\ \tfrac{n}{n+1}\hat R_n(\lambda)+\tfrac{1}{n+1}\ \le\ \beta\Bigr\},\qquad
\hat R_n(\lambda)=\tfrac1n\sum_{i=1}^n\ell_\lambda(X_i,\yt_i),
\end{equation}
under A1, the following bound holds.
\begin{equation}\label{eq:crc}
\E\bigl[\kappa(\yt_{n+1})\ind\{\yt_{n+1}\notin C_{\hat\lambda}(X_{n+1})\}\bigr]\ \le\ \beta\,\kappa_{\max}.
\end{equation}
The same bound holds for the deployed set $C^{\mathrm{ne}}_{\hat\lambda}$ because it contains the raw set pointwise.
\end{theorem}

\begin{proof}
For monotonicity, $C_\lambda$ is nested non-decreasing (Lemma~\ref{lem:contig}(iii)), so $\ind\{\tilde y\notin C_\lambda\}$ is non-increasing in $\lambda$. For right-continuity, $\lambda\mapsto\ind\{s(x,\tilde y)>\lambda\}$ is right-continuous. Boundedness with $B=1$ after normalization; and $\ell_{\lambda_{\max}}=\ell_1\equiv 0\le\beta$. These are exactly the hypotheses of the conformal risk control theorem \citep{angelopoulos2024}, whose threshold is $\inf\{\lambda:\frac{n}{n+1}\hat R_n(\lambda)+\frac{B}{n+1}\le\beta\}$ with $\hat\lambda=\lambda_{\max}$ if the set is empty, and the conclusion is theirs, rescaled by $\kappa_{\max}$.
\end{proof}

\emph{Operational reading.} Under A1, the expected severity-weighted cost of reported-label omissions in the stated review population is at most $\beta\kappa_{\max}$. The bound constrains the sets only when $\beta$ is below the expected normalized cost $\E[\kappa(\yt)]/\kappa_{\max}$; for larger $\beta$ even the smallest sets satisfy it and $\hat\lambda=0$. In the Texas test fold this cost is about 0.005 for the severity-cost loss and 0.0031 for the fatal indicator. Sets remain intervals, so a completed-record review rule based on their endpoints can be prespecified and audited.

\begin{theorem}[Latent-severity omission bound under a declared band; \tagderived{}]\label{thm:noisycrc}
Assume N($b^-,b^+,\delta$) and let
\begin{equation}\label{eq:kappaplus}
\kappa^+(\tilde y)=\kappa\bigl(\min(\tilde y+b^-,K)\bigr),
\end{equation}
the largest cost compatible with report $\tilde y$ within the band. Run Theorem~\ref{thm:crc} with $\kappa^+$ in place of $\kappa$ (same machinery; loss bounded by $1$ after normalizing by $\kappa_{\max}$) and output expanded sets $C^{\oplus}_{\hat\lambda}$. Then
\begin{equation}\label{eq:noisycrc}
\E\bigl[\kappa(Y_{n+1})\ind\{Y_{n+1}\notin C^{\oplus}_{\hat\lambda}(X_{n+1})\}\bigr]\ \le\ \beta\,\kappa_{\max}+\delta\,\kappa_{\max}.
\end{equation}
\end{theorem}

The added term $\delta\kappa_{\max}$ makes this bound informative only when $\delta$ is well below the expected normalized cost. With the declared $\delta=0.02$ and the Texas cost scale, whose expected normalized cost is about 0.005, the bound is vacuous at every budget, so the empirical analysis relies on the fatal-omission bound of Corollary~\ref{cor:fatal}, which needs no $\delta$ allowance.

\begin{proof}
Split on the band event $B_{n+1}=\{Y_{n+1}\in\Tmap(\yt_{n+1})\}$, where $\Prob(B_{n+1}^c)\le\delta$.
\[
\E[\kappa(Y)\ind\{Y\notin C^{\oplus}\}]\ \le\ \E\bigl[\kappa(Y)\ind\{Y\notin C^{\oplus}\}\ind_{B}\bigr]+\kappa_{\max}\delta.
\]
On $B$, two properties apply. First, $Y\notin C^{\oplus}\Rightarrow\yt\notin\Ct$ by the contrapositive of Definition~\ref{def:expansion}. Second, $Y\le\yt+b^-$, so $\kappa(Y)\le\kappa^+(\yt)$ by monotonicity. Hence the first term is at most $\E[\kappa^+(\yt)\ind\{\yt\notin\Ct_{\hat\lambda}\}]\le\beta\kappa_{\max}$ by Theorem~\ref{thm:crc} applied to the $\kappa^+$-loss, which is bounded by $\kappa(K)=\kappa_{\max}$.
\end{proof}

\begin{corollary}[Fatal-omission control under a declared no-under-reporting premise; \tagderived{}]\label{cor:fatal}
Take $\kappa(y)=\ind\{y=5\}$ and assume $Y=5\Rightarrow\yt=5$ for the record version and evaluation time in use. Then, with no $\delta$ inflation at the K boundary,
\begin{equation}\label{eq:fatal}
\Prob\bigl(Y_{n+1}=5,\ Y_{n+1}\notin C^{\oplus}_{\hat\lambda}(X_{n+1})\bigr)\ \le\ \beta.
\end{equation}
\end{corollary}

\begin{proof}
On $\{Y=5\}$, the premise gives $\yt=5$. Reflexivity gives $\Ct\subseteq C^{\oplus}$, so $Y\notin C^{\oplus}$ implies $\yt\notin\Ct$. The joint risk is therefore at most $\E[\ind\{\yt=5\}\ind\{\yt\notin\Ct_{\hat\lambda}\}]\le\beta$ by Theorem~\ref{thm:crc}. Only the stated direction $Y=5\Rightarrow\yt=5$ is used.
\end{proof}

\begin{remark}[Which direction of exactness the fatal bound needs]\label{rem:fataldirection}
The proof uses only the under-reporting direction $Y=5\Rightarrow\yt=5$. The cited result that all K-coded victims in one linked sample mapped to MAIS~6 concerns the reverse direction and is only a sample proportion. It neither supplies the premise used by the proof nor establishes exact recording in Texas.

The needed direction is a declared sensitivity premise. It can fail when a report precedes a later death or record amendment. Analyses at filing time must use Theorem~\ref{thm:noisycrc} with a declared K-boundary error allowance rather than this corollary.
\end{remark}

Equation~\eqref{eq:fatal} is a \emph{joint} probability over the full test population. Its empirical denominator is the full test-fold size. If $P(Y=5)>0$, the conditional omission rate among fatalities is
\begin{equation}\label{eq:fatal-conditional}
\Prob\{Y\notin C^{\oplus}_{\hat\lambda}(X)\mid Y=5\}
=\frac{\Prob\{Y=5,\ Y\notin C^{\oplus}_{\hat\lambda}(X)\}}{\Prob(Y=5)}
\le\frac{\beta}{\Prob(Y=5)},
\end{equation}
not $\beta$. The joint and conditional metrics must be reported with distinct denominators.

\subsection{Certification accounting}
\label{sec:compose}

\begin{theorem}[Certification accounting rule; \tagderived{}]\label{thm:compose}
Let $R$ be the training-frozen final-cell map in \eqref{eq:finalcell}. Suppose N($\Tmap,\delta_a$) holds conditional on $R=a$ for each observed final cell. Apply the compatibility expansion after final-cell calibration. The following statements hold.

\textbf{(i) Observed final cell.} For every final cell $a$ with positive test probability,
\begin{equation}\label{eq:compose-obs}
\Prob\bigl(Y\in C^{\oplus}\mid R=a\bigr)\ \ge\ 1-\alpha-\delta_a.
\end{equation}
This is a final-cell statement and does not imply the same bound within an original leaf rolled into $a$.

\textbf{(ii) New stratum.} If the covariate-shift and independent ratio-fitting conditions of Theorem~\ref{thm:transfer} hold within class $c$, and N($\Tmap,\delta_{c,g^\ast}$) holds under the same target law, then
\begin{equation}\label{eq:compose-new}
P_{g^\ast}\bigl(Y\in C^{\oplus}\mid\hat c=c\bigr)\ \ge\ 1-\alpha-\delta_{c,g^\ast}-\dtv\bigl(P^\ast_{X\mid c},\,Q_{\hat w,c}\bigr),
\end{equation}
where $Q_{\hat w,c}$ is the $\hat w$-tilt of the class-$c$ calibration covariate law. The empirical analysis uses the single-class case, in which the class is the whole population and part (ii) reduces to Theorem~\ref{thm:transfer} combined with the declared band.

\textbf{(iii) Risk scope.} In an exchangeable observed final cell $a$, running the risk-control calibration within that cell yields the corresponding bound $\beta\kappa_{\max}+\delta_a\kappa_{\max}$. This result applies to the observed final cell. Theorem~\ref{thm:transfer} establishes coverage transfer for the weighted quantile procedure under covariate shift.

The displayed coverage slacks are additive. The term $\delta$ comes from the declared compatibility relation, and $\dtv$ comes from the covariate-shift transfer.
\end{theorem}

\begin{proof}
(i) Theorem~\ref{thm:groups} gives reported-label coverage conditional on $R=a$. Theorem~\ref{thm:noise}'s union-bound argument, run on the same event $\{R=a\}$, subtracts $\delta_a$.
(ii) Within class $c$, Theorem~\ref{thm:transfer}(i) applies with all distributions conditioned on $\hat c=c$ (the class is a covariate-measurable event, so the covariate-shift condition and the tilt are well defined class-wise); then subtract $\delta$ as in (i).
(iii) Conditional exchangeability within an observed final cell permits Theorems~\ref{thm:crc} and~\ref{thm:noisycrc} to run cell-wise. Those theorems are not invoked for a shifted target law.
\end{proof}

\begin{remark}[The one delicate point]\label{rem:delicate}
Observed-cell calibration uses the single frozen map $R$. New-stratum weighting uses the calibration law as its reference and is a separate branch. Risk adjustment is available only on the exchangeable observed-cell branch. Algorithm~\ref{alg:choir} records these scopes.
\end{remark}

\begin{algorithm}[t]
\caption{CHOIR audit workflow with separate observed-cell and new-stratum branches}
\label{alg:choir}
\begin{algorithmic}[1]
\Require calibration scores $S_i=s(X_i,\yt_i)$; training-frozen final-cell map $R$; compatibility map $\Tmap$ with declared $\delta_a$; level $\alpha$; optional cost $\kappa$ and budget $\beta$ for observed cells; calibration-reference and target covariates for a new stratum
\State assign each $i\in\Ical$ its final cell $R(X_i)$ and record both the raw-cell and final-cell identifiers
\For{each observed final cell $a$}
  \State $\hat q_a\gets$ the $\lceil(1-\alpha)(n_a+1)\rceil$-th smallest score among $\{i:R(X_i)=a\}$ \Comment{Thm.~\ref{thm:groups}}
  \State optionally compute $\hat\lambda_a$ from \eqref{eq:lambdahat} within $a$ \Comment{risk control is observed-cell only}
\EndFor
\If{a target stratum is new}
  \State run Algorithm~\ref{alg:transfer}; do not attach an observed-cell risk bound
\EndIf
\State at prediction in an observed cell, set $\Ct(x)=C^{\mathrm{ne}}_{\hat q_{R(x)}}(x)$, or $C^{\mathrm{ne}}_{\max(\hat q_{R(x)},\hat\lambda_{R(x)})}(x)$ when risk control is requested, which carries both bounds; expand by \eqref{eq:expansion}
\State \textbf{emit} raw cell, final cell, calibration count, threshold, fallback indicator, map version, declared $\delta_a$, and every diagnostic with its provenance
\end{algorithmic}
\end{algorithm}

\begin{remark}[Slack budget as an artifact, and what it is not]\label{rem:budget}
For a new target, the diagnostic table can include
\begin{equation}\label{eq:certificate}
\widehat{\mathrm{Diag}}(c,g^\ast)\;=\;\Bigl(1-\alpha,\ \delta_{c,g^\ast},\ \widehat\Delta^{\,\mathrm{LCB}}_{\mathrm{emp},c,g^\ast}\Bigr),
\end{equation}
with each term carried with its provenance; $\delta_{c,g^\ast}$ is the declared beyond-band mass for the target and $\widehat\Delta^{\,\mathrm{LCB}}_{\mathrm{emp}}$ the classifier mismatch statistic. The two kinds of term are not interchangeable.

For a fixed declared $\delta_{c,g^\ast}$, Theorem~\ref{thm:compose} subtracts that input from the reported-label bound. The transfer theorem subtracts the \emph{true} $\dtv$, whereas the exported classifier quantity is only a lower confidence bound on $\dtv$. Subtracting a lower bound on a penalty gives an upper bound on the unknown coverage floor. Consequently $\widehat{\mathrm{Diag}}$ is not a coverage guarantee and must not be reported as one.

A large $\widehat\Delta^{\,\mathrm{LCB}}_{\mathrm{emp}}$ shows mismatch that the discriminator can detect. A small value does not certify target coverage. The tuple preserves the declared sensitivity inputs separately and does not subtract an empirical discrepancy from a coverage level.
\end{remark}

\subsection{Default base model}
\label{sec:dlcon}

DLCON (deep latent-class ordinal network) is the default base model, built as a mixture of ordinal experts,
\begin{equation}\label{eq:dlcon-cdf}
\hat P(\yt\le k\mid x)\;=\;\sum_{c=1}^{C}\pi_c(x_W;\phi)\,\sigma\bigl(\tau_{c,k}-f_c(x;\theta_c)\bigr),\qquad k=1,\dots,K-1,
\end{equation}
with $\sigma$ the logistic function, per-class MLP experts $f_c(\cdot;\theta_c)$, and a gating network on a designated covariate sub-vector $x_W$ of $x$ (slow, planning-level covariates when the latent classes are to be interpreted; $x_W=x$ in our implementation),
\begin{equation}\label{eq:gate}
\pi_c(x_W;\phi)\;=\;\frac{\exp g_c(x_W;\phi)}{\sum_{c'=1}^{C}\exp g_{c'}(x_W;\phi)},
\end{equation}
where $g(\cdot;\phi)$ is an MLP. Thresholds are strictly ordered by construction via cumulative softplus increments,
\begin{equation}\label{eq:taus}
\tau_{c,1}\in\mathbb{R}\ \text{free},\qquad
\tau_{c,k}\;=\;\tau_{c,1}+\sum_{j=1}^{k-1}\bigl(\log(1+e^{u_{c,j}})+\epsilon\bigr),\quad k=2,\dots,K-1,
\end{equation}
with free parameters $u_{c,j}$ and a fixed jitter $\epsilon=10^{-3}$ guaranteeing $\tau_{c,1}<\dots<\tau_{c,K-1}$. Training minimizes the ordinal negative log-likelihood plus the ranked probability score on the training labels,
\begin{align}
\mathcal{L}_{\mathrm{NLL}} &= -\frac{1}{n_{\mathrm{tr}}}\sum_{i\in\Itr}\log \hat p(\yt_i\mid x_i),
\qquad \hat p(k\mid x)=\hat F(k\mid x)-\hat F(k-1\mid x), \label{eq:nllterm}\\
\mathcal{L}_{\mathrm{RPS}} &= \frac{1}{n_{\mathrm{tr}}}\sum_{i\in\Itr}\sum_{k=1}^{K-1}\bigl(\hat F(k\mid x_i)-\ind\{\yt_i\le k\}\bigr)^2, \label{eq:rpsterm}\\
\mathcal{L} &= \mathcal{L}_{\mathrm{NLL}} + \tfrac12\,\mathcal{L}_{\mathrm{RPS}}, \label{eq:dlconloss}
\end{align}
optionally augmented by an entropy load-balancing penalty on the batch-mean gate distribution, which discourages gate collapse onto a single expert (an efficiency and interpretability aid; no guarantee depends on it). The MAP class
\begin{equation}\label{eq:mapclass}
\hat c(x)\;=\;\arg\max_{c}\ \pi_c(x_W;\phi)
\end{equation}
feeds Theorem~\ref{thm:het}.

The mixture components are used only to define a training-frozen partition. No identification claim is required. Multinomial-logit gates and mixture experts can have parameter symmetries, so component labels are not interpreted as unique population types. The validity results in Sections~\ref{sec:score}--\ref{sec:compose} hold for any fitted $\hat F$ and $\hat c$, including misspecified fits. The architecture is therefore optional, while the calibration guarantees depend on the frozen partition and the stated exchangeability conditions.

\section{Data and experimental design}
\label{sec:data}

\subsection{Source, unit of analysis, and attrition}
\label{sec:source}

The empirical work uses the Texas Crash Records Information System (CRIS), the statutory
repository of police-reported crashes maintained by the Texas Department of
Transportation, for calendar years 2017 through 2025. CRIS is distributed as linked
crash, unit, and person tables. The analysis joins these to two external sources. These are a
vehicle table decoded from the reported vehicle identification number through the
National Highway Traffic Safety Administration vPIC service, and a geographic file
mapping each crash identifier to a census block group and, where one applies, a
core-based statistical area.

The outcome is the police-reported KABCO injury category for one sampled driver, ordered
as O, C, B, A, and K. The estimand is that driver's recorded category in a completed
crash record, defined at the sampled-driver level rather than at the crash level.
Eligible rows are drivers of motor-vehicle units with a known KABCO category. The sampling
step retains one eligible driver per crash by ranking the person identifiers with the
MD5 hash of seed 20260704 concatenated with the identifier and retaining the smallest
hash. This reproducible draw uses only the person identifiers and the declared seed. Missing
and unknown KABCO categories remain outside the eligible sample. The certified population is
therefore drivers of motor-vehicle units with a recorded KABCO category, one per crash, so
estimates weight crashes rather than drivers. Passengers, pedestrians, cyclists, and the
968{,}939 driver rows without a valid KABCO code are outside it; the excluded drivers may
differ systematically from those included, for example when a driver leaves the scene.
Table~\ref{tbl:attrition} reports the attrition. Of 11{,}250{,}255 raw
crash-unit-person rows, 10{,}678{,}038 belong to motor-vehicle units, 10{,}047{,}354 are
driver rows, and 9{,}078{,}415 carry a valid severity code. Geocoding to a block group
succeeds for 8{,}585{,}232 of these. Geocoding is recorded as an indicator. All records
contribute to analyses that use no geographic condition, and geographic analyses use the
indicator to define their support.

\begin{table}[pos=htbp]
\caption{Attrition from raw CRIS records to the analysis sample.}
\label{tbl:attrition}
\centering\footnotesize
\begin{threeparttable}
\begin{tabular*}{\tblwidth}{@{\extracolsep{\fill}}lrr@{}}
\toprule
Stage & Rows & Dropped \\
\midrule
CRIS crash-unit-person rows, 2017--2025 & 11{,}250{,}255 & \\
Motor-vehicle units                     & 10{,}678{,}038 & 572{,}217 \\
Driver rows                             & 10{,}047{,}354 & 630{,}684 \\
Valid KABCO severity                    & 9{,}078{,}415  & 968{,}939 \\
Geocoded to block group\tnote{a}        & 8{,}585{,}232  & \\
\midrule
Full-extract descriptive rows, one driver per crash\tnote{b} & \textbf{5{,}213{,}921} & \\
\bottomrule
\end{tabular*}
\begin{tablenotes}[flushleft]\footnotesize
\item[a] An indicator, not a filter; records that fail to geocode remain in the sample.
\item[b] Sampled under Remark~\ref{rem:cluster}; see Section~\ref{sec:splitting}.
\end{tablenotes}
\end{threeparttable}
\end{table}

Table~\ref{tbl:attrition} describes the 2017--2025 full extract of 5{,}213{,}921 records.
The reported-label certificates use the 2017--2023 primary sample described below, and
the 2024--2025 records serve the temporal and monitoring analyses. Full-extract counts are
therefore descriptive and are not the denominators of the certificates.

\subsection{Sample designs and training-frozen final cells}
\label{sec:splitting}

Two people involved in the same crash do not carry independent severities, because they
share a speed, a road, a weather condition, and an impact. Assumption~\ref{ass:a1} therefore
requires the crash to remain the sampling cluster. A split that places one occupant in
calibration and another in test would inflate the effective calibration sample and would
calibrate the quantile against within-crash dependence. The protocol in
Remark~\ref{rem:cluster} keeps every crash within one fold.

The primary sample, called S1 below, covers 2017--2023. It contains 4{,}042{,}510 sampled
drivers from 4{,}042{,}510 unique crashes. Under seed 20260704, the random split assigns
2{,}423{,}834 records to training, 808{,}594 to calibration, and 810{,}082 to test
(60/20/20). Each crash appears in one fold only.

Raw support counts are computed on training records only. The fixed priority assigns a
motorcycle driver first; among the remainder, a driver with restraint code \texttt{None};
among the remaining records, a rural driver on a road with a posted limit of at least
55 mph; and otherwise the baseline cell. The training counts are 25{,}053, 30{,}224,
275{,}290, and 2{,}093{,}267 in that order. Each exceeds the prespecified threshold of
1{,}000, so no rollup is needed and the four raw cells are the four final cells of the
primary audit. Their calibration counts are 8{,}386 (motorcycle), 10{,}256 (unrestrained),
92{,}187 (rural high-speed), and 697{,}765 (baseline). The map is fixed before
calibration labels or scores are used; thresholds and reported-label claims condition on
the resulting final cell, not on an overlapping raw leaf.

The composition audit in Section~\ref{sec:e8e9} uses a finer declared partition. It
crosses the four cells with the CRIS rural or urban indicator, which gives eight product
leaves. The rural high-speed class is rural by definition, so its urban leaf is empty.
The remaining seven leaves all exceed the 1{,}000-record threshold, so the rollup rule
again leaves them unchanged and they are the seven final cells of
Table~\ref{tbl:compose}.

The temporal design, called S2, trains on 2017--2021, calibrates on 2022--2023, and
evaluates 2024--2025. Its County-Mondrian variant calibrates separately within each county
that has at least 1{,}000 training records and pools the remaining counties into one
statewide cell. The county design, called S3, uses the 2017--2023 records and the gradient
boosting base model. It orders the
254 Texas counties by their share of rural crashes and holds out 54 counties by systematic
sampling along that ordering, so that the held-out set spans the rural-urban range. The
model is trained and calibrated on the other 200 counties, with a random 75/25 split
between training and calibration, and each held-out county is evaluated with that single
frozen statewide threshold. The density-ratio diagnostic is computed only for the 29
held-out counties with at least 2{,}000 test records; its reference sample is drawn from the
training split, which follows the same law as calibration in this design. These designs provide
stress tests of completed records for temporal transfer, county variation, and monitoring.
Because a single frozen threshold is applied to data from other years or counties, none of
them satisfies Assumption~\ref{ass:a1}, and their results are descriptive.

Two further partitions are used in Sections~\ref{sec:e2e3} and~\ref{sec:repair}. The KMeans
partition has eight clusters fit on the encoded covariates of a random 500{,}000-record
training subsample. The speed-hour partition crosses four posted-speed bands (below 35, 35
to 54, 55 to 69, and 70 mph or more) with a night band (22:00 to 05:59) and a day band;
a speed-hour cell with fewer than 1{,}000 calibration records is merged into its speed
band. The speed-hour bands were set after the descriptive analysis of
Section~\ref{sec:composition}, and the merge rule uses calibration counts, so this partition
is not prespecified in the sense of the four-cell audit and its results are illustrative.

The base models are specified as follows. The ordered logit is a proportional-odds model
with a small ridge penalty ($10^{-4}$), fit on a random 300{,}000-record subsample of the
training split. The multinomial logit is fit on a 500{,}000-record subsample. The
latent-class ordered logit has three classes with a multinomial-logit gate on the
covariates and is fit by five EM iterations on a 300{,}000-record subsample. The
random-parameters ordered logit has a normally distributed random intercept and a random
slope on posted speed limit, and is fit by simulated maximum likelihood with 32 scrambled
Halton draws on a 100{,}000-record subsample. Gradient boosting is the scikit-learn
histogram implementation with balanced class weights, a learning rate of 0.15, 31 leaves,
at most 100 iterations with early stopping, and at most 1{,}000{,}000 training records.
Balanced class weights raise the predicted probabilities of the rare severe categories, a
common choice when severe crashes are the focus, but they also make its probability estimates
poorly calibrated (Table~\ref{tbl:models}).
DLCON (Section~\ref{sec:dlcon}) is trained on the full training split and uses four latent classes, a gating network with one hidden
layer of 64 units, and expert networks with hidden layers of 128 and 64 units. It is
trained with Adam at a learning rate of $10^{-3}$ and a batch size of 8{,}192 for at most
eight epochs, with early stopping after three epochs without improvement on a 3\%
validation split and a gate load-balancing weight of 0.05. TabPFN (\texttt{tabpfn} package version 9.0.0) is used with its
default classifier settings and a training context of 10{,}000 records drawn from the
training split by stratified sampling, in which each of the categories C, B, A, and K receives its
population share or 200 records, whichever is larger, and the no-injury category receives
the remainder (8{,}002, 988, 610, 200, and 200 records for O, C, B, A, and K). The floor
over-represents the rare severe categories relative to their population shares, and no
prior correction is applied to the predicted probabilities. As an in-context learner,
TabPFN's fit step only prepares the context, and prediction over the full calibration and
test folds is its dominant cost. DLCON and TabPFN were run on a GPU. The
subsamples and the fixed context explain the short fit times in Table~\ref{tbl:models}. All seven models are
calibrated and tested on the full calibration and test folds.

\subsection{Features, exclusions, and missingness}
\label{sec:features}

Features are declared in two completed-record configurations for the empirical audits.
Configuration P represents the finalized fields available in the primary analysis, and
Configuration T adds further finalized fields for an extended comparison.
Configuration P contains selected finalized-record fields, including crash context such as posted speed
limit, weather, light and surface condition, road type and alignment, traffic control,
functional system, annual daily traffic, day of week and hour; person attributes such as
driver age, gender, restraint use and license class; vehicle attributes decoded from the
vehicle identification number, including vehicle age, body class, gross weight class,
drive type, door count, engine displacement and power, and the presence of electronic
stability control, collision warning, lane departure warning, blind spot monitoring and
curtain or side airbags; and geography. Configuration T adds finalized-record fields,
including manner of collision, harmful event, object struck, vehicle damage severity, airbag
deployment, and ejection. Configuration T is an extended completed-record analysis. The
reported-label coverage results use Configuration P, and Configuration T supplies a
descriptive feature-ablation comparison.

Restraint use is a completed-record covariate with shared reporting provenance. The field
describes a pre-crash state and is recorded after the crash by the same officer, on the same
form, and at the same time as the KABCO code. Validity holds for every declared partition and
every base model built from the completed record, but the provenance matters for
interpretation. The unrestrained cell is defined by an officer judgment made together with
the outcome, so its coverage deficit may partly reflect reporting practice, and the
within-stratum nondifferential reporting assumed by Theorem~\ref{thm:channel} is doubtful
for that cell. The Configuration T fields in Section~\ref{sec:e8e9} share the same provenance.

Certain fields are excluded from every configuration because they encode the outcome.
These include injury and death counts, suspected serious injury counts, and the crash-level
severity code. Investigator narratives and blood alcohol results are also excluded. Blood
alcohol results occur for a small, non-random subset and are recorded after the event. All
personally identifying material, including the vehicle identification number itself,
license and registration identifiers, dates of birth, and free-text narratives, is removed
at ingest and appears in no released artifact.

The decoded vehicle fields record equipment availability rather than equipment engagement.
A record indicating that a vehicle line offers automatic emergency braking therefore enters
as a vehicle-technology availability covariate. Decode
coverage is also systematically sparser for older vehicles, so
an undecoded indicator enters the design matrix rather than the record being silently
dropped.

Missingness is handled explicitly and no record is discarded for incompleteness.
Categorical fields receive an explicit missing level, and levels below 0.5\% frequency
are pooled. Continuous fields are median imputed and standardized, and a missingness
indicator is added for each continuous field with missing values, so a model can condition
on the fact of absence. The encoder is fit on the training split of each design. The
``stored ind.'' column of Table~\ref{tbl:data} counts only the two indicators stored in the
data (VIN undecoded and geocoded). Two continuous fields are substantially incomplete, annual daily traffic at
48.65\% and engine power at 50.99\%, and both are reported in Table~\ref{tbl:data} with
their indicators noted. Incompleteness of traffic volume is also why the descriptive map
below uses hour of day rather than volume as its second axis. Separately, two fields
carry sentinel values rather than nulls, in that posted speed limit is recorded as $-1$ on
244{,}965 records of the full extract (206{,}716 in the primary sample), and vehicle age is
negative on 18{,}335 records whose model year post-dates the crash year. Neither is a null,
so neither registers in a missingness count, and both enter the design matrix as stored. The
speed-limit sentinel is a coding limitation, because the linear models treat it as a very low
posted limit, and these records can never enter the rural high-speed cell. A negative
vehicle age of one year is a genuine next-model-year vehicle rather than an error. The descriptive statistics in
Table~\ref{tbl:data} exclude them and say so.

\subsection{Composition of the sample}
\label{sec:composition}

The reported-label analysis uses the training-frozen final cells. The latent-injury result
uses the declared compatibility maps of Section~\ref{sec:noise}, with the category-dependent
map of Remark~\ref{rem:catdep} as the primary map, and $\delta_a=0.02$ in each final cell
(Section~\ref{sec:e8e9} checks this value cell by cell). The map and $\delta_a$ are sensitivity inputs informed by external linkage
evidence. Two semi-synthetic audits treat observed CRIS KABCO as a proxy for the underlying
injury and inject synthetic reported labels. In the sensitivity grid of
Section~\ref{sec:e4}, a fixed histogram gradient-boosting score and a common recordwise
kernel generate synthetic reported labels in both the calibration and test folds under
seed 20260708, so that test-fold report coverage and the outside-map share can be
measured. In the composition audit of Section~\ref{sec:e8e9}, synthetic reported labels are
generated in the calibration fold only. In both audits, expanded coverage is evaluated
against the observed test-fold KABCO. These designs audit the implementation under a known
injected process, while medical measurement and local channel estimation remain separate
linkage analyses.

Table~\ref{tbl:data} reports the composition of the full 2017--2025 extract in three panels, which
give records by year and reported outcome, the covariate profile across the severity ladder,
and the feature blocks with their completeness. Panel A shows a stable annual volume near
580{,}000 records, with the lowest total of the series in 2020. Panel B is the panel that
motivates the framework. Reading across the ladder from no injury to fatal, mean posted
speed limit rises from 44.6 to 55.5 mph, the rural share rises from 0.255 to 0.571, the
motorcycle-driver share rises from 0.002 to 0.196, and median annual daily traffic falls
from 34{,}378 to 12{,}336. The severe end of the KABCO scale is a structurally distinct
population rather than a random subsample of the mild end, and it is also a small one,
holding 15{,}774 fatal records, or 0.30\% of the sample.

\begin{table}[pos=htbp]
\caption{Composition of the full 2017--2025 extract.}\label{tbl:data}
\centering\footnotesize
\begin{threeparttable}
\begin{tabular*}{\tblwidth}{@{\extracolsep{\fill}}lrrrrrr@{}}
\toprule
\multicolumn{7}{@{}l}{\textsc{Panel A. Driver records by year and reported outcome}}\\
\midrule
Year & \cellcolor{seqteal1}O & \cellcolor{seqteal1}C & \cellcolor{seqteal2}B & \cellcolor{seqteal3}A & \cellcolor{seqteal4}\textcolor{white}{K} & Total\\
\midrule
2017 & 474,174 & 60,193 & 34,815 & 7,965 & 1,680 & 578,827\\
2018 & 483,410 & 61,957 & 32,384 & 6,891 & 1,650 & 586,292\\
2019 & 498,230 & 64,451 & 32,108 & 7,123 & 1,599 & 603,511\\
2020\tnote{a} & 413,477 & 52,313 & 28,680 & 6,939 & 1,733 & 503,142\\
2021 & 483,262 & 55,335 & 36,570 & 9,056 & 1,975 & 586,198\\
2022 & 485,983 & 53,150 & 40,110 & 8,759 & 1,871 & 589,873\\
2023 & 489,323 & 52,352 & 42,525 & 8,644 & 1,823 & 594,667\\
2024 & 488,537 & 50,787 & 44,236 & 8,319 & 1,806 & 593,685\\
2025 & 474,709 & 48,087 & 45,282 & 8,011 & 1,637 & 577,726\\
\cmidrule{1-7}
\textbf{Total} & \textbf{4,291,105} & \textbf{498,625} & \textbf{336,710} & \textbf{71,707} & \textbf{15,774} & \textbf{5,213,921}\\
\midrule
\multicolumn{7}{@{}l}{\textsc{Panel B. Covariate profile across the severity ladder}}\\
\midrule
 & \cellcolor{seqteal1}O & \cellcolor{seqteal1}C & \cellcolor{seqteal2}B & \cellcolor{seqteal3}A & \cellcolor{seqteal4}\textcolor{white}{K} & \\
\midrule
$n$ & 4,291,105 & 498,625 & 336,710 & 71,707 & 15,774 & \\
mean speed limit (mph) & 44.6 & 46.0 & 47.9 & 51.3 & 55.5 & \\
rural share & \databar{.2552}.255 & \databar{.2104}.210 & \databar{.3045}.304 & \databar{.4572}.457 & \databar{.5708}.571 & \\
motorcycle-driver share & \databar{.0018}.002 & \databar{.0202}.020 & \databar{.0595}.060 & \databar{.1902}.190 & \databar{.1962}.196 & \\
median ADT (known)\tnote{b} & 34,378 & 35,376 & 27,321 & 16,842 & 12,336 & \\
mean driver age & 38.7 & 39.9 & 38.9 & 38.9 & 41.9 & \\
female share & \databar{.3958}.396 & \databar{.5148}.515 & \databar{.4440}.444 & \databar{.3047}.305 & \databar{.2068}.207 & \\
median vehicle age (yr) & 8 & 8 & 9 & 10 & 11 & \\
VIN undecoded & \databar{.0304}.030 & \databar{.0247}.025 & \databar{.0307}.031 & \databar{.0427}.043 & \databar{.0304}.030 & \\
geocoded & \databar{.9263}.926 & \databar{.9697}.970 & \databar{.9716}.972 & \databar{.9756}.976 & \databar{.9888}.989 & \\
\midrule
\multicolumn{7}{@{}l}{\textsc{Panel C. Feature blocks and completeness (Configuration P)}}\\
\midrule
Block & cont. & categ. & stored ind. & levels & \multicolumn{2}{c}{continuous features (\% missing)}\\
\midrule
crash context & 3 & 10 & 0 & 80 & \multicolumn{2}{p{0.40\textwidth}}{\raggedright\scriptsize speed limit~0.00, ADT~48.65, hour~0.00}\\
person & 1 & 5 & 0 & 33 & \multicolumn{2}{p{0.40\textwidth}}{\raggedright\scriptsize driver age~0.85}\\
vehicle (VIN) & 4 & 11 & 1 & 148 & \multicolumn{2}{p{0.40\textwidth}}{\raggedright\scriptsize vehicle age~0.48, doors~26.07, displacement~4.63, engine power~50.99}\\
geography & 0 & 2 & 1 & 11 & \multicolumn{2}{p{0.40\textwidth}}{\raggedright\scriptsize --}\\
\cmidrule{1-7}
\textbf{Total} & \textbf{8} & \textbf{28} & \textbf{2} & \textbf{272} & \multicolumn{2}{p{0.40\textwidth}}{\raggedright\scriptsize one-hot encoded at a 0.005 minimum frequency; rarer levels pooled}\\
\bottomrule
\end{tabular*}
\begin{tablenotes}[flushleft]\footnotesize
\item[a] The 2020 total is the lowest in the series.
\item[b] ADT is 48.65\% missing and engine power 50.99\% missing; both enter the design matrix with median imputation and a missingness indicator, so no record is dropped for incompleteness.
\item[c] Panel B statistics for posted speed limit exclude the 244,965 records (4.70\%) carrying the CRIS $-1$ sentinel, and those for vehicle age exclude 18,335 records with a model year later than the crash year. Both values are sentinels rather than nulls, so Panel C's missingness counts do not register them; the fitted models see the columns as stored. Motorcycle share is the CRIS-reported body style, which is populated for every record, rather than the VIN decode.
\end{tablenotes}
\end{threeparttable}
\end{table}

The same structure is visible in the raw data before any model is fit.
Figure~\ref{fig:landscape} maps the share of severe or fatal outcomes over the plane of
hour of day and posted speed limit. Severe outcomes concentrate in a high-speed
late-night region, where cells at or above 55 mph between midnight and 04:00 pool to 5.8\%
against 0.9\% on the daytime low-speed plateau, a sixfold contrast, with a peak cell of
7.4\%. Figure~\ref{fig:landscape} describes reporting frequencies in the records, and it
matters here for one narrow and specific reason. A single calibration threshold
computed over this plane is dominated by the plateau, because that is where almost all of
the records are, and the ridge is where the rate of severe outcomes is highest.

\begin{figure}[pos=htbp]
\centering
\includegraphics[width=\linewidth,height=.58\textheight,keepaspectratio]{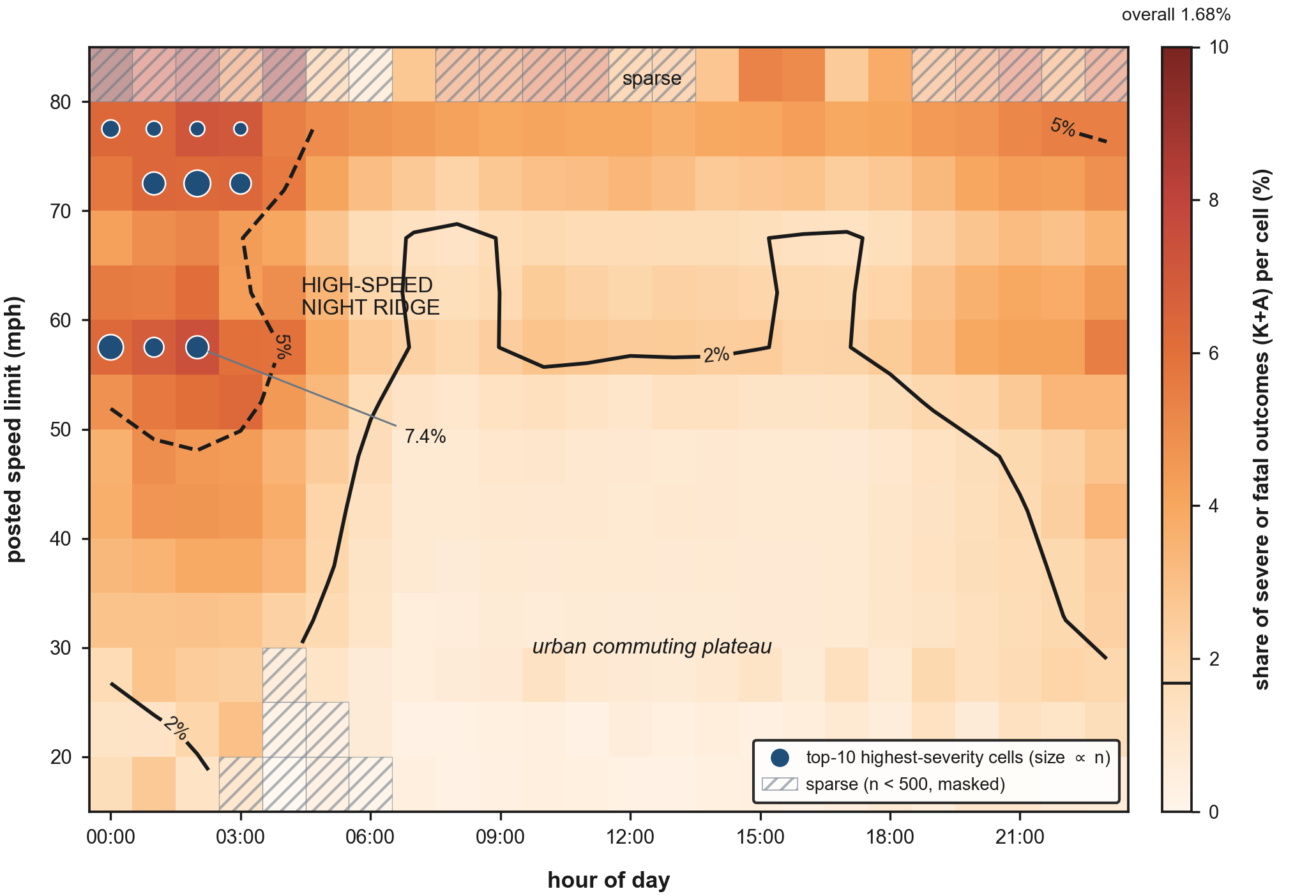}
\caption{The severity landscape of Texas driver records, 2017--2025. Color gives the
share of severe or fatal (K+A) outcomes in each hour by speed-limit cell, for the
4{,}877{,}108 records with posted limits between 15 and 85 mph; contours mark the 2\% and
5\% levels, and the colorbar tick marks the 1.68\% K+A share of the full sample. Cells at
or above 55 mph between 00:00 and 03:59 pool to 5.8\%, against 0.9\% for cells at or
below 45 mph between 06:00 and 21:59. Cells with fewer than 500 records are masked.}
\label{fig:landscape}
\end{figure}

Figure~\ref{fig:corner} shows the same separation jointly rather than one covariate at a time. The
no-injury and severe-or-fatal groups occupy overlapping but visibly displaced regions of
covariate space, because the severe group sits ten miles per hour higher in posted limit and
three years older in vehicle age, and carries a late-night tail nearly three times
heavier. Driver age barely moves, which is worth stating plainly, because it shows the
displacement is a property of where and in what people crash rather than of who they are.
The two distributions overlap substantially, and that overlap is the difficulty. They are
neither separable, which would make the problem easy, nor identical, which would make one
threshold sufficient.

\begin{figure}[pos=htbp]
\centering
\includegraphics[width=\linewidth,height=.58\textheight,keepaspectratio]{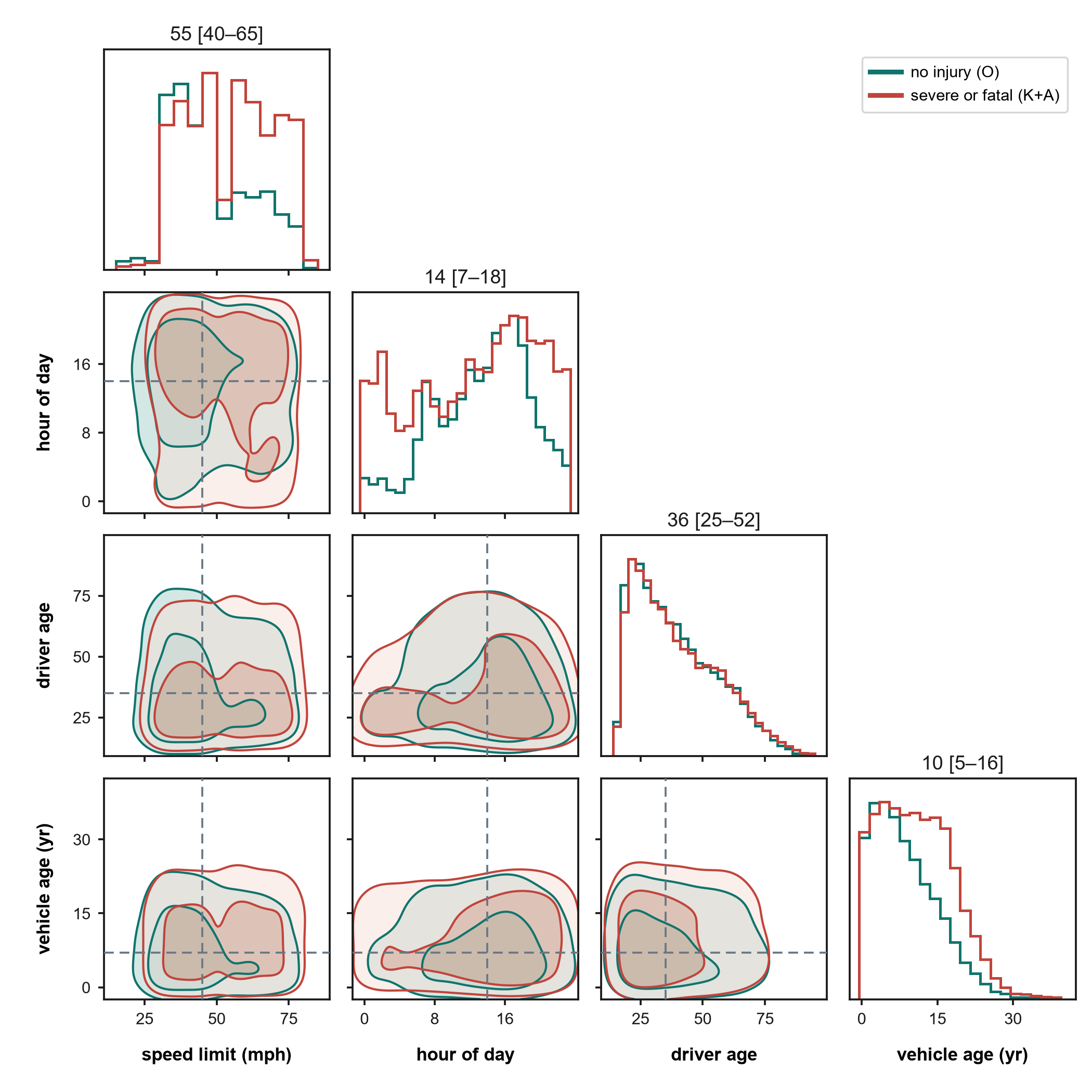}
\caption{Joint covariate geometry of no-injury (teal) and severe or fatal (coral) driver
records. Filled contours hold 50\% and 90\% of each group, dashed lines mark the
no-injury medians, and panel titles give the severe group's median and interquartile
range. The severe group sits at higher posted speed limits (median 55 against 45 mph) and
in older vehicles (10 against 7 years), and carries a heavier late-night tail (20.8\% of
severe records fall between 00:00 and 05:59, against 7.8\% of no-injury records), while
driver age is nearly unchanged. The figure uses complete cases, so its medians differ slightly from Table~\ref{tbl:data}. All 84{,}419 complete-case severe records are shown
against 150{,}000 no-injury records sampled at seed 20260704.}
\label{fig:corner}
\end{figure}

\section{Results}
\label{sec:results}

The primary results evaluate a reported-label certificate on the primary sample, a
declared compatibility-expansion sensitivity analysis, and descriptive audits of completed
records. All analyses use finalized records. Test-fold coverage is reported as empirical
evidence alongside the stated finite-sample guarantee. Coverage uncertainty is reported
with two-sided Clopper--Pearson intervals and a Bonferroni correction over the stated
family, so that the intervals hold simultaneously while each guarantee remains a per-cell
statement (Remark~\ref{rem:arbitrary}). These intervals are conditional on the fitted
calibration thresholds, whereas the guarantee averages over calibration draws. To measure
that variability, the calibration and test folds of the primary sample are pooled and
re-split at random into two halves 200 times, with each base model held fixed; each repeat
recalibrates on one half and evaluates on the other.
Only the random crash-level split of the primary sample satisfies Assumption~\ref{ass:a1};
the temporal, county, and monitoring analyses are descriptive stress tests.

\subsection{Reported-label audit in the four final cells}
\label{sec:e1}

At nominal coverage 0.90, calibration within the four final cells gives 28
model-by-stratum estimates for the seven base models, from 0.8983 to 0.9074, and every
Bonferroni-adjusted 95\% interval contains 0.90 (Table~\ref{tab:maincells}, Panel A). The
cells are mutually exclusive and follow the training-frozen priority rule in
Section~\ref{sec:splitting}. The reported certificate conditions on the final cells. Across
the 200 repeated splits, mean coverage lies between 0.8996 and 0.9005 in all 28 cells. The
standard deviation across repeats is 0.0005 in the baseline cell and 0.004 to 0.005 in the
motorcyclist and unrestrained cells, 95\% of repeats fall between 0.891 and 0.910, and no
repeat falls below 0.88 in any cell. Certified widths vary by at most 0.02 categories across
repeats. The single-split results in Table~\ref{tab:maincells} are therefore representative
of the calibration-draw distribution.

\begin{table}[pos=htbp]
\caption{Coverage and width by final cell for all seven base models ($\alpha=0.10$).}
\label{tab:maincells}
\centering\scriptsize
\setlength{\tabcolsep}{3.5pt}
\begin{threeparttable}
\begin{tabular*}{\tblwidth}{@{\extracolsep{\fill}}lccccc@{}}
\toprule
Base model & Baseline & Motorcycle & Rural high-speed & Unrestrained & All \\
\midrule
\multicolumn{6}{@{}l}{\textsc{Panel A. Final-cell (certified) coverage, with simultaneous 95\% interval}}\\
\midrule
Ordered logit & 0.8999 [0.899, 0.901] & 0.9004 [0.890, 0.910] & 0.8988 [0.896, 0.902] & 0.9002 [0.891, 0.909] & 0.8998 \\
Multinomial logit & 0.8998 [0.899, 0.901] & 0.8992 [0.889, 0.909] & 0.8991 [0.896, 0.902] & 0.9030 [0.894, 0.912] & 0.8998 \\
Latent-class OL & 0.8996 [0.898, 0.901] & 0.9015 [0.891, 0.911] & 0.8992 [0.896, 0.902] & 0.9037 [0.894, 0.913] & 0.8996 \\
Random-parameters OL & 0.8998 [0.899, 0.901] & 0.8988 [0.888, 0.909] & 0.8990 [0.896, 0.902] & 0.8995 [0.890, 0.908] & 0.8997 \\
Gradient boosting & 0.8997 [0.899, 0.901] & 0.8996 [0.889, 0.910] & 0.8983 [0.895, 0.901] & 0.9022 [0.893, 0.911] & 0.8996 \\
DLCON & 0.8999 [0.899, 0.901] & 0.8984 [0.888, 0.908] & 0.9000 [0.897, 0.903] & 0.9026 [0.893, 0.911] & 0.8999 \\
TabPFN & 0.8997 [0.899, 0.901] & 0.9014 [0.891, 0.911] & 0.8993 [0.896, 0.902] & 0.9074 [0.898, 0.916] & 0.8998 \\
\midrule
\multicolumn{6}{@{}l}{\textsc{Panel B. Final-cell (certified) mean set width, KABCO categories}}\\
\midrule
Ordered logit & 1.42 & 3.61 & 2.03 & 3.77 & 1.54 \\
Multinomial logit & 1.41 & 3.58 & 1.99 & 3.63 & 1.53 \\
Latent-class OL & 1.41 & 3.62 & 2.00 & 3.74 & 1.53 \\
Random-parameters OL & 1.42 & 3.63 & 2.02 & 3.76 & 1.54 \\
Gradient boosting & 3.16 & 4.22 & 4.65 & 4.58 & 3.36 \\
DLCON & 1.40 & 3.55 & 1.95 & 3.58 & 1.52 \\
TabPFN & 1.42 & 4.18 & 2.01 & 4.25 & 1.55 \\
\midrule
\multicolumn{6}{@{}l}{\textsc{Panel C. Coverage under pooled (marginal) calibration}}\\
\midrule
Ordered logit & 0.9051 & 0.8410 & 0.8741 & 0.8223 & 0.8998 \\
Multinomial logit & 0.9036 & 0.8295 & 0.8775 & 0.8957 & 0.8997 \\
Latent-class OL & 0.9044 & 0.8691 & 0.8744 & 0.8427 & 0.8998 \\
Random-parameters OL & 0.9031 & 0.8750 & 0.8833 & 0.8482 & 0.8999 \\
Gradient boosting & 0.9467 & 0.5828 & 0.6236 & 0.3739 & 0.8989 \\
DLCON & 0.9007 & 0.8675 & 0.8991 & 0.8895 & 0.9000 \\
TabPFN & 0.9005 & 0.8374 & 0.9039 & 0.8679 & 0.8998 \\
\midrule
Test records $n$ & 699,412 & 8,530 & 91,805 & 10,335 & 810,082 \\
\bottomrule
\end{tabular*}
\begin{tablenotes}[flushleft]\footnotesize
\item[] \textit{Note:} The nominal coverage is 0.90. The 28 two-sided Clopper--Pearson intervals in Panel A use a Bonferroni correction for at least 95\% simultaneous coverage, and every interval contains 0.90. The intervals are conditional on the fitted calibration thresholds; the variation across 200 repeated calibration splits is reported in Section~\ref{sec:e1}. Cells are mutually exclusive and follow the priority motorcycle, unrestrained among the remainder, rural high-speed among the remainder, and baseline. Panel C uses one pooled threshold for all records; its binomial standard errors are about 0.004 in the motorcycle and unrestrained cells. OL is ordered logit; gradient boosting uses balanced class weights.
\end{tablenotes}
\end{threeparttable}
\end{table}

Panel C of Table~\ref{tab:maincells} shows the same models under one pooled threshold.
Every model reaches about 0.90 overall, yet every model covers at least one of the
motorcyclist and unrestrained cells at 0.868 or lower, more than eight binomial standard
errors below 0.90. The shortfall varies by model and cell. The ordered logit covers the
two cells at 0.841 and 0.822, DLCON at 0.868 and 0.890, and TabPFN at 0.837 and 0.868. The
class-balanced gradient boosting model is the extreme case, at 0.583 for motorcyclists,
0.624 for the rural high-speed cell, and 0.374 for unrestrained drivers. It is used as the
anchor in several later analyses because it is a common machine-learning choice for
tabular crash data and makes the mechanism easy to see, but Table~\ref{tbl:models} shows
that its class balancing also makes its probability estimates much worse than those of the
other six models. Results that depend on the anchor are therefore also reported for a
narrow model where the difference matters (Sections~\ref{sec:e7} and~\ref{sec:e8e9}).

Panel B gives the price of certification. The certified sets for motorcyclists span 3.55
to 4.22 of the five KABCO categories across models, and for unrestrained drivers 3.58 to
4.58. In the baseline cell they span 1.40 to 1.42 categories for the six narrow models. The
guarantee is therefore informative for most drivers and much less informative for the
groups where it matters most.

The contribution under test in this section is the certification layer rather than any
base model. A smaller set size is a property of the base model. Coverage that holds for
every base model is the property the layer is designed to deliver.

\subsection{Seven-model completed-record audit}
\label{sec:seven-model}

Table~\ref{tbl:models} compares the seven base models, namely the ordered logit that has
anchored severity analysis since \citet{mccullagh1980}, a multinomial logit, a
latent-class ordered logit, a random-parameters ordered logit, gradient boosting, the
DLCON network of Section~\ref{sec:dlcon}, and the TabPFN tabular foundation model. DLCON
is the default base model of Section~\ref{sec:dlcon}; it is included as one of the seven
and is not a claimed contribution of the paper.

\begin{table}[pos=htbp]
\caption{One certification layer over seven base models ($\alpha=0.10$).}\label{tbl:models}
\centering\footnotesize
\setlength{\tabcolsep}{2.6pt}
\begin{threeparttable}
\begin{tabular*}{\tblwidth}{@{\extracolsep{\fill}}lcccccccc@{}}
\toprule
 & \multicolumn{2}{c}{Pooled calibration} & Certified & & & & & \\
\cmidrule(lr){2-3}
Base model & Coverage & Width & width & RPS & Log loss & Fit & Predict & Calibrate \\
\midrule
Ordered logit & 0.8998 & 1.54 & 1.54 & 0.216 & 0.592 & 6\,s & 2\,s & 0.14\,s \\
Multinomial logit & 0.8997 & 1.53 & 1.53 & 0.215 & 0.585 & 24\,min & $<$1\,s & 0.13\,s \\
Latent-class OL & 0.8998 & 1.53 & 1.53 & 0.215 & 0.590 & 3\,min & 3\,s & 0.13\,s \\
Random-parameters OL & 0.8999 & 1.54 & 1.54 & 0.216 & 0.591 & 19\,s & 7\,s & 0.13\,s \\
Gradient boosting & 0.8989 & 3.41 & 3.36 & 0.698 & 1.269 & 43\,s & 14\,s & 0.10\,s \\
DLCON & 0.9000 & 1.52 & 1.52 & 0.213 & 0.584 & 22\,s & $<$1\,s & 0.14\,s \\
TabPFN & 0.8998 & 1.56 & 1.55 & 0.219 & 0.596 & 9\,s & 25\,min & 0.13\,s \\
\bottomrule
\end{tabular*}
\begin{tablenotes}[flushleft]\footnotesize
\item[] \textit{Note:} Held-out test records $n=810{,}082$; nominal coverage 0.90. Pooled calibration uses one threshold for all records; the certified width uses calibration within the four final cells (Table~\ref{tab:maincells}). RPS is the ranked probability score and log loss the mean negative log-likelihood of the reported category on the test fold (lower is better). Every set is contiguous. At $\alpha\in\{0.05,0.10,0.20\}$, 20 of 21 model-by-$\alpha$ cells met a descriptive screen of nominal less three binomial standard errors; the exception is gradient boosting at $\alpha=0.10$, short by 0.00008. Fit and predict belong to the base model and reflect the training subsamples and context sizes in Section~\ref{sec:splitting}; five models ran on a desktop CPU and DLCON and TabPFN on a GPU, so timings compare orders of magnitude only. OL is ordered logit; gradient boosting uses balanced class weights.
\end{tablenotes}
\end{threeparttable}
\end{table}

Coverage is indistinguishable across the seven models under both pooled and certified
calibration, as Proposition~\ref{prop:marginal} and Theorem~\ref{thm:groups} assert.
Efficiency is where the models separate. After certification within the four final cells,
six models give mean widths between 1.52 and 1.55 categories, and the ordered logit
(1.54) is within 0.03 categories of the narrowest model (DLCON, 1.52) and slightly narrower
than the tabular foundation model (1.55). The difference is larger where it matters. In the
motorcyclist and unrestrained cells the certified ordered-logit sets span 3.61 and 3.77
categories, against 4.18 and 4.25 for TabPFN (Table~\ref{tab:maincells}), whose
stratified context raises its predicted probabilities for severe categories. The six
models' predictive quality is also close, with ranked probability scores between 0.213 and
0.219. The balanced gradient-boosting model is the widest (3.36),
because class balancing shifts its predicted distribution toward rare categories, and its
ranked probability score (0.698) and log loss (1.269) are much worse. The comparison has
two limits. The models were trained on different subsample sizes and TabPFN on a fixed
context (Section~\ref{sec:splitting}), so the table compares the models as fitted here
rather than each model at its best. And the mean width is dominated by the baseline cell,
so the per-cell widths in Table~\ref{tab:maincells} are the more informative comparison.

The widths in these tables are those of the construction certified here, which fixes a
threshold on the score of Definition~\ref{def:score}. Width in the ordinal setting has a
set-construction component that is separable from the estimator, and
\citet{zhang2025minlength} give sets of minimal length per instance given the base scores.
Any construction exporting contiguous sets can be certified by the same layer.

Table~\ref{tab:baselines} compares the layer with other conformal methods on the same
cached probabilities. Pooled LAC and pooled CHOIR both reach 0.90 overall and both
under-cover the safety cells. Label-conditional LAC, the class-conditional construction
used in prior injury-risk work \citep{wei2025}, conditions on the KABCO category rather
than on the road-user group; it leaves the unrestrained cell at 0.630 to 0.685 for the
narrow models and roughly doubles the mean set size. APS over-covers (0.98 or more overall) with sets of 2.8 to 4.0 categories; with the
narrow models it covers every cell, but at 0.96 to 1.00 and with sets nearly twice as large
as the certified ones. Calibration within the declared safety cells is the only method that
holds every cell near 0.90, and it does so with contiguous sets.

\begin{table}[pos=htbp]
\caption{Alternative conformal methods on the primary sample ($\alpha=0.10$).}\label{tab:baselines}
\centering\footnotesize
\setlength{\tabcolsep}{3pt}
\begin{threeparttable}
\begin{tabular*}{\tblwidth}{@{\extracolsep{\fill}}llcccccc@{}}
\toprule
 & & & & & \multicolumn{3}{c}{Coverage in safety cells} \\
\cmidrule(lr){6-8}
Base model & Method & Coverage & Mean size & Contiguous & Motorcycle & Rural high-speed & Unrestrained \\
\midrule
Ordered logit & LAC, pooled & 0.900 & 1.54 & 100.0\% & 0.913 & 0.884 & 0.873 \\
 & LAC, label-conditional & 0.900 & 3.50 & 100.0\% & 0.856 & 0.784 & 0.630 \\
 & APS & 0.984 & 2.81 & 100.0\% & 0.995 & 0.966 & 0.957 \\
 & CHOIR, pooled & 0.900 & 1.54 & 100.0\% & 0.841 & 0.874 & 0.822 \\
 & CHOIR, four cells & 0.900 & 1.54 & 100.0\% & 0.900 & 0.899 & 0.900 \\
\addlinespace
DLCON & LAC, pooled & 0.900 & 1.50 & 99.7\% & 0.932 & 0.890 & 0.921 \\
 & LAC, label-conditional & 0.899 & 3.16 & 97.7\% & 0.857 & 0.764 & 0.685 \\
 & APS & 0.989 & 2.93 & 100.0\% & 1.000 & 0.989 & 0.995 \\
 & CHOIR, pooled & 0.900 & 1.52 & 100.0\% & 0.868 & 0.899 & 0.890 \\
 & CHOIR, four cells & 0.900 & 1.52 & 100.0\% & 0.898 & 0.900 & 0.903 \\
\addlinespace
Gradient boosting & LAC, pooled & 0.899 & 3.05 & 93.7\% & 0.506 & 0.607 & 0.311 \\
 & LAC, label-conditional & 0.899 & 3.38 & 95.4\% & 0.509 & 0.607 & 0.320 \\
 & APS & 0.981 & 4.04 & 98.8\% & 0.719 & 0.930 & 0.537 \\
 & CHOIR, pooled & 0.899 & 3.41 & 100.0\% & 0.583 & 0.624 & 0.374 \\
 & CHOIR, four cells & 0.900 & 3.36 & 100.0\% & 0.900 & 0.898 & 0.902 \\
\bottomrule
\end{tabular*}
\begin{tablenotes}[flushleft]\footnotesize
\item[] \textit{Note:} All methods use the same cached base-model probabilities, calibration fold and test fold ($n=810{,}082$). LAC uses the score $1-\hat p(\yt\mid x)$ with one threshold (pooled) or one threshold per KABCO category (label-conditional, the class-conditional construction of prior injury-risk work). APS is the non-randomized adaptive prediction set. CHOIR uses the ordinal cumulative score with one threshold (pooled) or within the four final cells. Contiguous is the share of sets that form one run of adjacent KABCO categories.
\end{tablenotes}
\end{threeparttable}
\end{table}

\subsection{Heterogeneity and marginal coverage deficits}
\label{sec:e2e3}

Marginal coverage is an average over the population and can hide poor coverage in small
groups. Under the gradient boosting anchor the pooled threshold covers baseline drivers at
0.947 and the safety cells at 0.374 to 0.624. Read per reported KABCO category, coverage is
0.889 for O, 0.952 for C, 0.958 for B, 0.869 for A, and 0.822 for K. The pooled threshold
therefore over-covers the minor-injury categories and under-covers both the no-injury
majority and the severe categories, so the pattern is not a simple ladder from mild to
severe.

The unrestrained stratum is defined narrowly. CRIS records restraint use in several codes,
and only one of them, ``None,'' states that a restraint was available and not used. The
neighboring ``Not Applicable'' code is carried by every motorcyclist (8{,}530 of its
10{,}518 test-split records), and motorcyclists are assigned to their own cell first.
Defining the stratum as ``None'' alone leaves 10{,}335 records. A looser definition that
folds in the inapplicable and residual codes would raise the anchor's pooled coverage from
0.374 to 0.417; the narrow definition gives the larger deficit.

The shortfall follows the mixture-threshold mechanism in Theorem~\ref{thm:oracle}(iii)
for each fitted score. A single threshold calibrated on a population that is 82\% uninjured
fits that population, and the severe strata receive sets that are too small. A marginally
calibrated model is therefore least reliable for the motorcyclist and unrestrained-driver
records where a completed-record audit needs stable uncertainty.

Final-cell calibration works for any prespecified partition. For the gradient boosting
anchor, the four declared cells reach 0.8983 to 0.9022, and the eight cells of the KMeans
partition of Section~\ref{sec:splitting} reach 0.8980 to 0.9016, as Theorem~\ref{thm:het}
implies for any partition fixed in advance.

Validity does not depend on where a partition comes from, but efficiency does. For the same
anchor and at matched validity, the latent-class ordered logit's own gate gives the
narrowest sets of the partitions tested, at a mean width of 3.236 categories against 3.356
for the declared safety strata, 3.402 for the KMeans partition, and 3.406 for the DLCON
gate. Its three classes also have a clear reading. Their widths run from 2.44 to 4.04
categories while each holds coverage within a few thousandths of nominal, so the model
separates groups that differ in how predictable their severity is. The latent-class gate is
estimated jointly with an ordinal outcome model, which may explain why it separates the
score distribution better than a clustering fit for geometric compactness. These width
differences are single-split point estimates.

Efficiency in the aggregate is not the same as repair where it matters, and
Section~\ref{sec:repair} shows that the two can point in different directions. The
partition that best repairs the deficit on the speed-hour ridge is the one defined on the
ridge's own coordinates, not the one that minimizes width overall. The repair has a price
in width. Under the anchor, the motorcyclist stratum's mean width rises from 2.64 to 4.22
categories and the unrestrained stratum's from 2.62 to 4.58, and those sets were narrow
precisely because they often missed the reported label.

\subsection{Coverage deficit and calibration result}
\label{sec:repair}

Figure~\ref{fig:repair} shows the heterogeneity result on the plane of hour and posted
speed limit, using the gradient boosting anchor and the three partitions described in
Section~\ref{sec:splitting}. Under marginal calibration the deficit is concentrated on the
severity ridge, and 90 of the 254 supported cells fall below 0.80; coverage falls to 0.338
in the 75-mph cell at 03:00. The KMeans partition raises 75-mph coverage from 0.611 to
0.692 but leaves the night deficit in place, because it does not encode hour. The
speed-hour partition raises the deepest cell to 0.847 and leaves 7 cells below 0.80.

The speed-hour bands are round engineering values, but they were chosen after the
descriptive analysis of Figure~\ref{fig:landscape}, which uses all records, so they should
not be read as prespecified. Their effect also has a cost at the band edges. Four of the
seven remaining cells below 0.80 lie at the day-band edge hours of 06:00 and 21:00, and were
covered at 0.87 to 0.95 under marginal calibration; for example, the 30-mph cell at 06:00
falls from 0.953 to 0.780. Coverage is lower than under marginal calibration in 105 of the
254 supported cells, most of which were over-covered before. A partition therefore moves
coverage toward the structure it encodes, and the guarantee reaches no further than the
declared cells.

\begin{figure}[pos=htbp]
\centering
\includegraphics[width=\linewidth,height=.58\textheight,keepaspectratio]{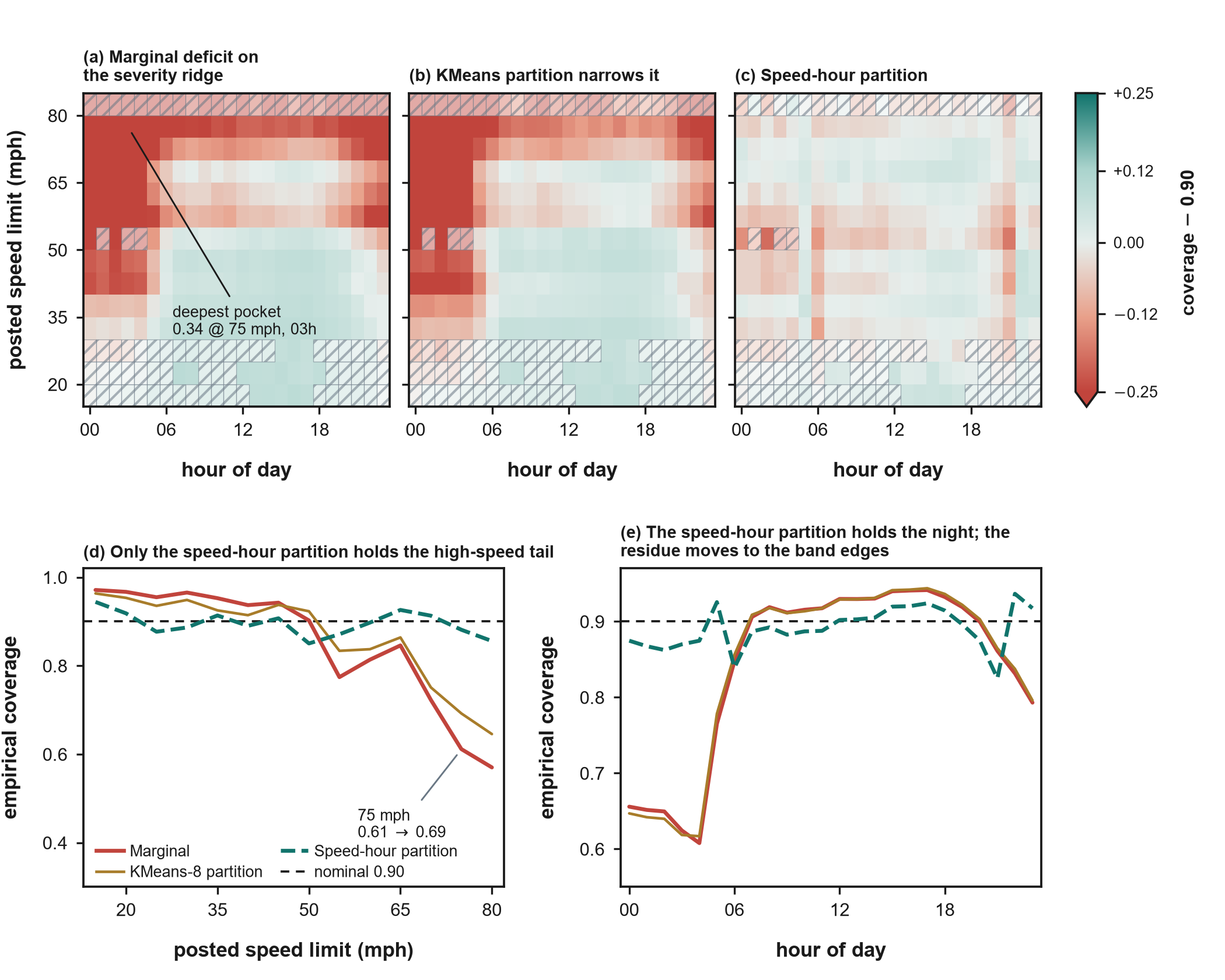}
\caption{Coverage deficit and calibration result. Map color gives the signed difference
between empirical coverage and nominal 0.90 over the plane of Figure~\ref{fig:landscape},
for the 754{,}186 test records with posted limits between 15 and 85 mph and the gradient
boosting base model; cells with fewer than 400 records (82 of 336) are hatched. The
guarantee applies to the final cells of each partition, not to plotting cells or to this
subset, so pooled values here are descriptive. (a) Under marginal calibration, coverage falls to
0.338 in the 75-mph, 03:00 cell. (b) Under the KMeans partition, 75-mph coverage rises to 0.692,
while the night deficit remains (Proposition~\ref{prop:impossible}). (c) Under the speed-hour
partition, the deepest cell reaches 0.847 (0.8979 pooled over the subset); the remaining
cells below 0.80 lie at the band edges, with an hour-wall minimum of 0.825. (d,e) The same
comparison along each axis.}
\label{fig:repair}
\end{figure}

\subsection{Declared compatibility-expansion sensitivity audit}
\label{sec:e4}

The audit uses pre-injection KABCO as a proxy outcome. One common recordwise kernel
produces synthetic calibration and test reports, and coverage is evaluated against
test-fold proxy KABCO. The analysis is an implementation check under a known injected
process; it is not evidence about the Texas reporting process.

At $\alpha=0.10$ and declared $\delta_a=0.02$, reported-label coverage ranges from 0.8955
to 0.9030 across cells, outside-map shares are 0.0012 to 0.0137, expanded coverage is
0.9963 to 1.0000, above the 0.88 floor, and mean width is 4.26 to 5.00 categories
(Table~\ref{tab:semisynthetic}).

\begin{table}[pos=htbp]
\caption{Semi-synthetic sensitivity audit by final cell under the declared compatibility relation ($\delta_a=0.02$).}
\label{tab:semisynthetic}
\centering\footnotesize
\begin{threeparttable}
\begin{tabular}{lrrrrr}
\toprule
Final cell & $n$ & Report coverage & Outside-map share & Expanded coverage & Mean width \\
\midrule
Baseline & 699,412 & 0.8996 & 0.0012 & 0.9973 & 4.26 \\
Motorcycle & 8,530 & 0.8955 & 0.0137 & 1.0000 & 5.00 \\
Rural high-speed & 91,805 & 0.8969 & 0.0025 & 0.9963 & 4.99 \\
Unrestrained & 10,335 & 0.9030 & 0.0108 & 0.9998 & 5.00 \\
Overall & 810,082 & 0.8993 & 0.0016 & 0.9972 & 4.36 \\
\bottomrule
\end{tabular}
\begin{tablenotes}[flushleft]\footnotesize
\item[] \textit{Note:} The fixed-score implementation stress test uses $\alpha=0.10$, $\delta_a=0.02$, and seed 20260708. The theorem floor is 0.88 in every cell. The observed CRIS category is the pre-injection proxy label. Cells use the priority rule stated in Table~\ref{tab:maincells}.
\end{tablenotes}
\end{threeparttable}
\end{table}

Table~\ref{tbl:noise} reports the sensitivity grid. The grid varies the declared
beyond-band mass $\delta$ of Assumption~\ref{ass:noise} and compares the KABCO map, a
constant one-category map, and the unexpanded interval.

\begin{table}[pos=htbp]
\caption{Proxy-outcome coverage under declared compatibility maps in the semi-synthetic sensitivity audit.}\label{tbl:noise}
\centering\footnotesize
\begin{threeparttable}
\begin{tabular*}{\tblwidth}{@{\extracolsep{\fill}}lccccc@{}}
\toprule
 & \multicolumn{2}{c}{Declared floor} & \multicolumn{3}{c}{Empirical proxy-outcome coverage} \\
\cmidrule(lr){2-3}\cmidrule(lr){4-6}
$\delta$ & Banded & Generic\tnote{a} & KABCO map & Constant map & No expansion\tnote{b} \\
\midrule
0.00 & 0.90 & 0.7645 & 0.9972 & 0.9972 & 0.8974 \\
0.01 & 0.89 & 0.7642 & 0.9972 & 0.9972 & 0.8974 \\
0.02 & 0.88 & 0.7640 & 0.9972 & 0.9972 & 0.8974 \\
0.05 & 0.85 & 0.7623 & 0.9973 & 0.9972 & 0.8981 \\
\bottomrule
\end{tabular*}
\begin{tablenotes}[flushleft]\footnotesize
\item[a] The generic floor is $1-\alpha-\varepsilon_{\mathrm{tot}}$. The realized total error rate is 0.135--0.138 across the grid.
\item[b] The unexpanded arm is an empirical comparator under the same proxy-outcome process.
\item[] At $\delta=0$ and $0.01$ the injected process exceeds the declared beyond-band mass in some cells (Table~\ref{tab:semisynthetic}), so those rows show the floor as declared rather than as satisfied by the injected process.
\end{tablenotes}
\end{threeparttable}
\end{table}

Three patterns summarize the grid. First, the banded floor decreases from 0.90 to 0.85 as
$\delta$ increases from 0 to 0.05, while the generic floor based on total error remains near
0.763 because the realized total error is 0.135 to 0.138. The injected process is much
milder than the linkage evidence, where total disagreement is near 0.49, and its realized
beyond-band mass stays below 0.02 overall. The banded floor depends only on $\delta$, so a
larger adjacent error would leave it unchanged, while the generic floor would fall to about
0.41. Because the grid never approaches the declared band, it does not stress the band
assumption itself.

Second, expanded proxy-outcome coverage remains near 0.997 for both maps, with an overall
mean width near 4.36 categories. This width is the observable price of the declared floor
under this process.

Third, the unexpanded interval already covers the proxy outcome at 0.8974 to 0.8981 with
mean width near 3.35, above every banded floor. The expansion converts this realized
performance into a bound stated from the declared compatibility relation, and the grid
connects that bound to a directly measured width cost.

\subsection{Temporal and spatial completed-record stress tests}
\label{sec:e5e6}

The temporal stress test trains on 2017--2021, calibrates on 2022--2023, and evaluates
2024--2025. Its eight final-cell estimates range from 0.8919 to 0.9094
(Table~\ref{tab:temporalcells}). Five simultaneous 95\% Clopper--Pearson intervals contain
0.90, while three lie entirely above 0.90. Mean width is 3.19 to 4.66 categories. Intervals
that lie above nominal are themselves evidence that the later years are not exchangeable
with the calibration years, even though coverage stays close to the target.

\begin{table}[pos=htbp]
\caption{Final-cell temporal stress test for histogram gradient boosting.}
\label{tab:temporalcells}
\centering\footnotesize
\begin{threeparttable}
\begin{tabular}{llrrrr}
\toprule
Year & Final cell & $n$ & Coverage & Simultaneous 95\% interval & Mean width \\
\midrule
2024 & Baseline & 509,230 & 0.9001 & [0.8989, 0.9012] & 3.19 \\
2024 & Motorcycle & 6,481 & 0.9094 & [0.8993, 0.9189] & 4.23 \\
2024 & Rural high-speed & 69,843 & 0.9049 & [0.9018, 0.9079] & 4.65 \\
2024 & Unrestrained & 8,131 & 0.8950 & [0.8854, 0.9041] & 4.60 \\
\addlinespace
2025 & Baseline & 495,823 & 0.9028 & [0.9016, 0.9039] & 3.19 \\
2025 & Motorcycle & 6,422 & 0.9011 & [0.8905, 0.9110] & 4.25 \\
2025 & Rural high-speed & 67,181 & 0.9061 & [0.9030, 0.9092] & 4.66 \\
2025 & Unrestrained & 8,300 & 0.8919 & [0.8823, 0.9011] & 4.61 \\
\bottomrule
\end{tabular}
\begin{tablenotes}[flushleft]\footnotesize
\item[] \textit{Note:} The model is trained on 2017--2021 records and calibrated on 2022--2023 records. The eight two-sided Clopper--Pearson intervals use Bonferroni correction for at least 95\% simultaneous coverage. This table is a descriptive temporal stress test.
\end{tablenotes}
\end{threeparttable}
\end{table}

Table~\ref{tbl:temporal} reports the statewide transfer analysis. The unweighted and
County-Mondrian rows use the full target years. For the density-ratio row, the calibration
fold is split into a reference half, used only to fit the ratio, and a scored half, whose
scores are weighted; the target year is split into a fit half and an evaluation half. The
weighted coverage is 0.9002 in 2024 and 0.9053 in 2025, against 0.8990 and 0.9006 for the
unweighted threshold on the same rows. In these years the unweighted threshold was already
close to nominal, so the comparison shows that weighting does no harm here rather than that
it repairs a shift. The classifier mismatch statistic of Theorem~\ref{thm:transfer}(ii) is
0.000 and 0.008, so a linear discriminator finds little covariate mismatch between the
target years and the weighted calibration data.

\begin{table}[pos=htbp]
\caption{Temporal completed-record stress test and covariate-mismatch statistic ($\alpha=0.10$).}
\label{tbl:temporal}
\centering\footnotesize
\begin{threeparttable}
\begin{tabular*}{\tblwidth}{@{\extracolsep{\fill}}lrrrrr@{}}
\toprule
Method & Cov. 2024 & Cov. 2025 & Width 2024 & Width 2025 & Mismatch statistic \\
\midrule
Unweighted split conformal & 0.8993 & 0.9008 & 3.439 & 3.438 & -- \\
County-Mondrian & 0.9003 & 0.9024 & 3.411 & 3.414 & -- \\
Density-ratio weighted & 0.9002 & 0.9053 & 3.445 & 3.461 & 0.000 / 0.008 \\
\bottomrule
\end{tabular*}
\begin{tablenotes}[flushleft]\footnotesize
\item[] \textit{Note:} County-Mondrian calibrates within each county that has at least 1,000 training records and pools the rest into one statewide cell. The unweighted and County-Mondrian rows use 593,685 records in 2024 and 577,726 in 2025. For the density-ratio row, the 2022--2023 calibration fold (1,184,540 records) is split at random into a reference half, used with a target-year fit half to estimate $dP_{\mathrm{target},X}/dP_{\mathrm{calibration},X}$, and a scored half (592,270 records) whose scores are weighted. Coverage is evaluated on the other target-year half (296,843 and 288,863 records); same-row unweighted coverage is 0.8990 and 0.9006. The mismatch statistic is the classifier lower bound of Theorem~\ref{thm:transfer}(ii); it is not a bound on the population transfer penalty.
\end{tablenotes}
\end{threeparttable}
\end{table}

Figure~\ref{fig:county} shows the county stress test, using the S3 design and the gradient
boosting anchor of Section~\ref{sec:splitting}. Each held-out county is evaluated with one
statewide threshold calibrated on the other 200 counties, so this test uses marginal
calibration and does not isolate county shift from the stratum deficit shown above. The
failure is widespread. Pooled over all 750{,}655 held-out records coverage is 0.880, but
the median county is covered at 0.663, 42 of the 54 held-out counties fall below 0.80,
including 18 of the 29 counties with at least 2{,}000 test records, and raw coverage ranges
from 0.372 to 0.936. The lowest values occur in small rural counties, such as Roberts
County (0.372 on 183 test records) and Kinney County (0.450 on 258), while the largest
counties are close to or above nominal, such as Dallas (0.917 on 371{,}868). For the 29
larger counties, density-ratio weighting gives coverage from 0.872 to 0.922 at a mean width
about 0.43 categories larger than the unweighted sets.

\begin{figure}[pos=htbp]
\centering
\includegraphics[width=\linewidth,height=.58\textheight,keepaspectratio]{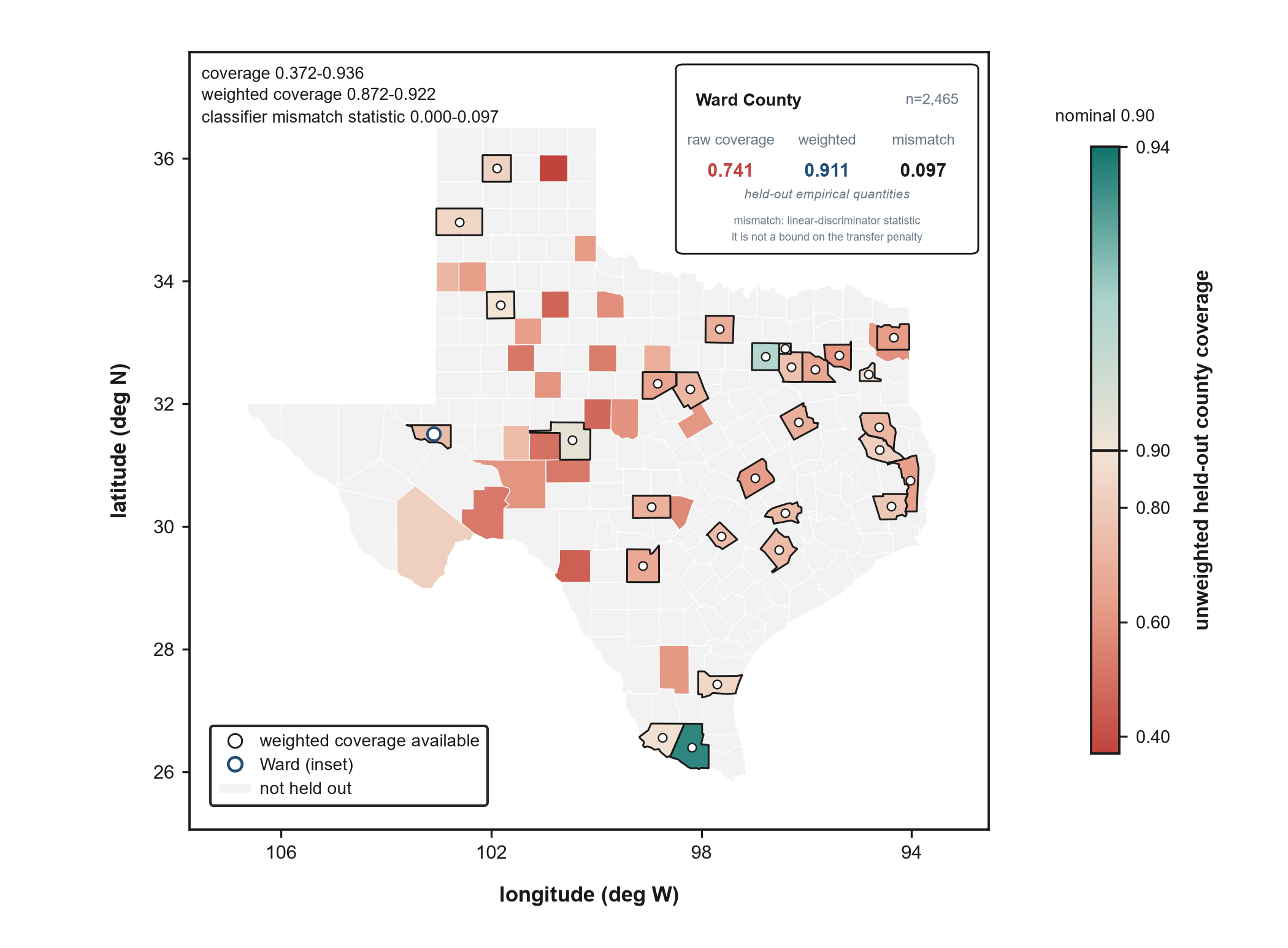}
\caption{Spatial completed-record stress display. Each of the 54 held-out counties is filled
by its raw unweighted coverage under one frozen statewide threshold (0.372 to 0.936 at
nominal 0.90); gray counties are used for training and calibration. The 29 ringed counties,
those with at least 2{,}000 test records, also carry density-ratio weighted coverage (0.872
to 0.922) and a classifier mismatch statistic (0.000 to 0.097); the raw range covers all 54
counties and the weighted range only these 29. The inset shows Ward County (raw 0.741,
weighted 0.911).}
\label{fig:county}
\end{figure}

\subsection{Completed-record monitoring audit}
\label{sec:e7fig}

Figure~\ref{fig:dashboard} summarizes two years of held-out records. Monthly coverage stays
within 0.009 of nominal with a two-year mean of 0.900. Several months nevertheless lie
outside the $\pm2$ standard-error band, and the quantiles of the conformal p-values depart
from uniformity by up to 0.021, more than sampling error alone would explain at these
monthly sizes. The stream is therefore not exchangeable with the calibration data, even
though average coverage stays near nominal.

\begin{figure}[pos=htbp]
\centering
\includegraphics[width=\linewidth,height=.58\textheight,keepaspectratio]{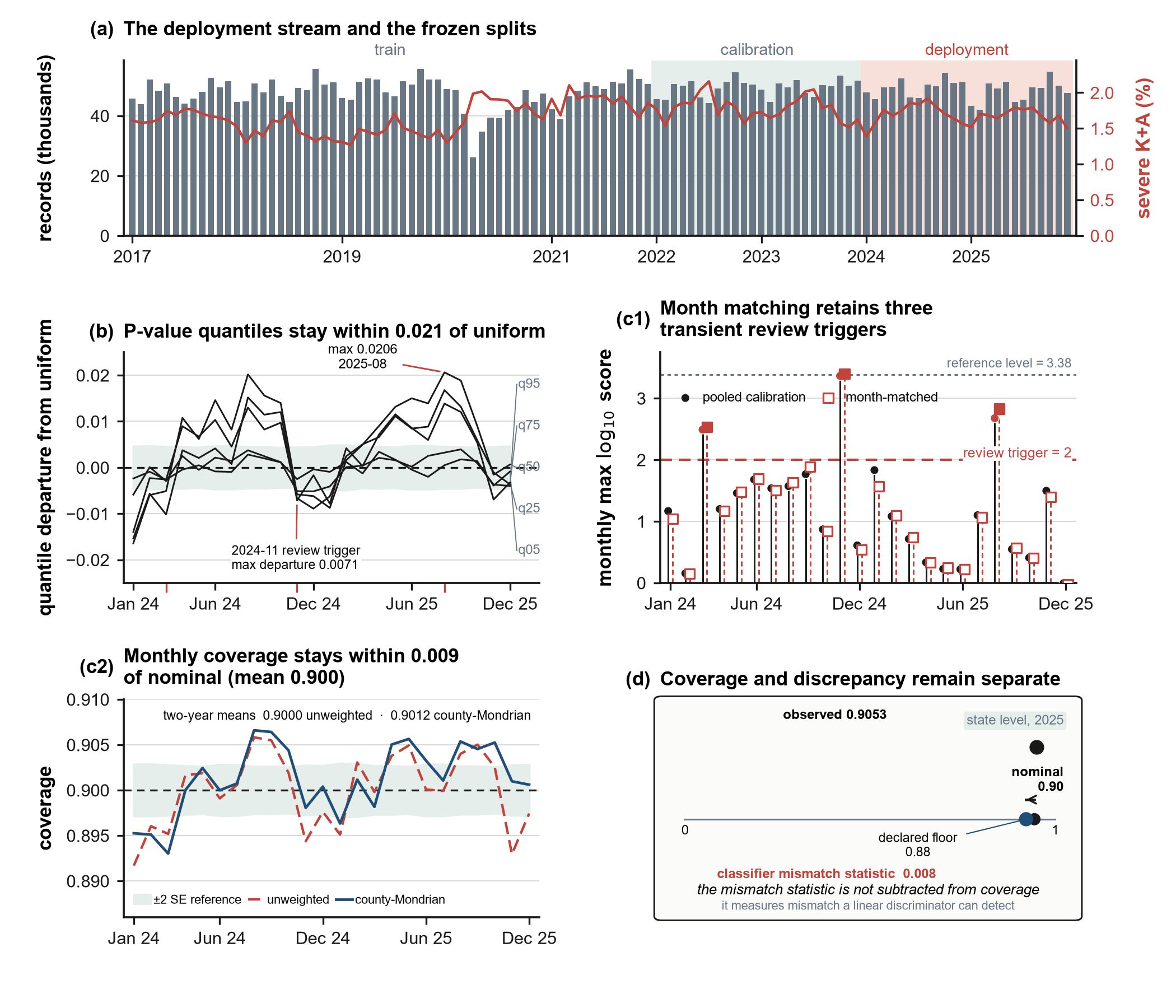}
\caption{Completed-record monitoring visualization. (a) Monthly volume and severe share over
2017--2025 with the frozen temporal splits. (b) Monthly quantiles of the conformal p-values
and their departure from uniformity ($\pm 2$ SE shading). (c1) Monthly maximum of the
frozen-reference review score of Algorithm~\ref{alg:monitor}, restarted each calendar month
with $\varepsilon=0.5$; the dashed line is the review threshold $\log_{10}b=2$, and the dotted
line marks $\log_{10}(24/0.01)=3.38$, the level that would control a 1\% family-wise false
alarm rate over 24 months if the score were a valid test martingale, shown for reference
only. (c2) Monthly coverage under the unweighted and County-Mondrian thresholds. (d) The
2025 state-level panel reports the density-ratio weighted coverage (0.9053,
Table~\ref{tbl:temporal}), the 0.88 declared-map floor, and the classifier mismatch
statistic (0.008) as separate quantities; the statistic is not subtracted from coverage.}
\label{fig:dashboard}
\end{figure}

The review score of Algorithm~\ref{alg:monitor}, restarted each month with $\varepsilon=0.5$
and threshold $b=100$, exceeds its threshold in three isolated months. Seasonality is a
natural explanation, since the calibration pools all months. After the triggers fired, we
re-tested them against a month-matched calibration drawn from the same calendar months of
the calibration years. This follow-up check was not prespecified and is exploratory. The
triggers exceed the threshold again under it, so pooled seasonal composition does not
explain them. Coverage in those months deviates from nominal by at most 0.006. The score has
no calibrated false-alarm rate, so the three months are flags for review, not detected
shifts in a statistical sense.

\subsection{Completed-record endpoint and risk audit}
\label{sec:e7}

The endpoint audit in Table~\ref{tab:endpoint} uses police-reported KABCO on all
810{,}082 test records of the primary sample. A record enters the retrospective review
queue when its upper endpoint is A or K. Workload and joint omission use all test records;
conditional omission uses the 13{,}641 observed A-or-K records. The two denominators are
not interchangeable.

\begin{table}[pos=htbp]
\caption{Completed-record endpoint audit on 810,082 test records.}
\label{tab:endpoint}
\centering\footnotesize
\begin{threeparttable}
\begin{tabular*}{\tblwidth}{@{\extracolsep{\fill}}llrrrrr@{}}
\toprule
Base model & Interval & Mean width & Workload & Severe omissions & Conditional omission & Joint omission \\
\midrule
Gradient boosting & Reported-label & 3.36 & 0.4122 & 1,956 & 0.1434 & 0.00241 \\
 & Expanded sensitivity & 4.36 & 0.9591 & 29 & 0.0021 & 0.00004 \\
Ordered logit & Reported-label & 1.54 & 0.0200 & 8,907 & 0.6530 & 0.01100 \\
 & Expanded sensitivity & 2.55 & 0.0539 & 7,192 & 0.5272 & 0.00888 \\
DLCON & Reported-label & 1.52 & 0.0188 & 8,828 & 0.6472 & 0.01090 \\
 & Expanded sensitivity & 2.52 & 0.0643 & 6,397 & 0.4690 & 0.00790 \\
\bottomrule
\end{tabular*}
\begin{tablenotes}[flushleft]\footnotesize
\item[] \textit{Note:} Intervals use calibration within the four final cells and the observed police-reported KABCO outcome. A record enters the retrospective review queue when the interval upper endpoint is A or K. Workload and joint omission use all test records; conditional omission uses the 13,641 observed A or K test outcomes.
\end{tablenotes}
\end{threeparttable}
\end{table}

The endpoint rule behaves very differently across base models. Under the wide gradient
boosting sets it routes 41.2\% of records to review and misses 14.3\% of observed A-or-K
outcomes; the expanded sets route 95.9\% of records, which is not a usable screen. Under the
narrow ordered logit and DLCON sets the rule routes only about 2\% of records but misses
about 65\% of observed A-or-K outcomes, and the expanded sets still miss about half of
them. A 90\% coverage guarantee over all records does not make the upper endpoint a good
screen for severe injury, because most covered records are uninjured. Severe-outcome
screening needs the risk-control layer below or a dedicated rule, and its operating point
has to be chosen with the omission rate in view.

Figure~\ref{fig:risk} reports the risk analysis for the gradient boosting anchor. Joint
fatal omission uses the full test-fold population as its denominator, and the conditional
rate among the 2{,}475 fatal records follows from \eqref{eq:fatal-conditional}. The
distinction matters in practice. The four budgets of the initial design ($\beta$ from 0.01
to 0.10) all exceed the fatal share of the test fold (0.0031), so the fatal constraint is
slack there and the calibrated threshold falls to zero; 1{,}544 fatal records, or 62.4\% of
them, fall outside the risk-calibrated set, although the joint rate of 0.0019 meets the
bound. The same holds for the cost loss, whose expected normalized value is about 0.005.
The informative range therefore lies below these budgets, and Figure~\ref{fig:risk} extends
the grid down to $\beta=0.0005$. There the fatal constraint binds, the joint rate is
0.000495, and the conditional rate falls to 16.2\%. An agency that needs protection for
fatal records must choose $\beta$ well below the fatal share and should report the
conditional rate next to the joint bound. Panel (c) uses a distinct escalation rule that
counts records whose risk-calibrated lower endpoint is B, A, or K.

\begin{figure}[pos=htbp]
\centering
\includegraphics[width=\linewidth,height=.58\textheight,keepaspectratio]{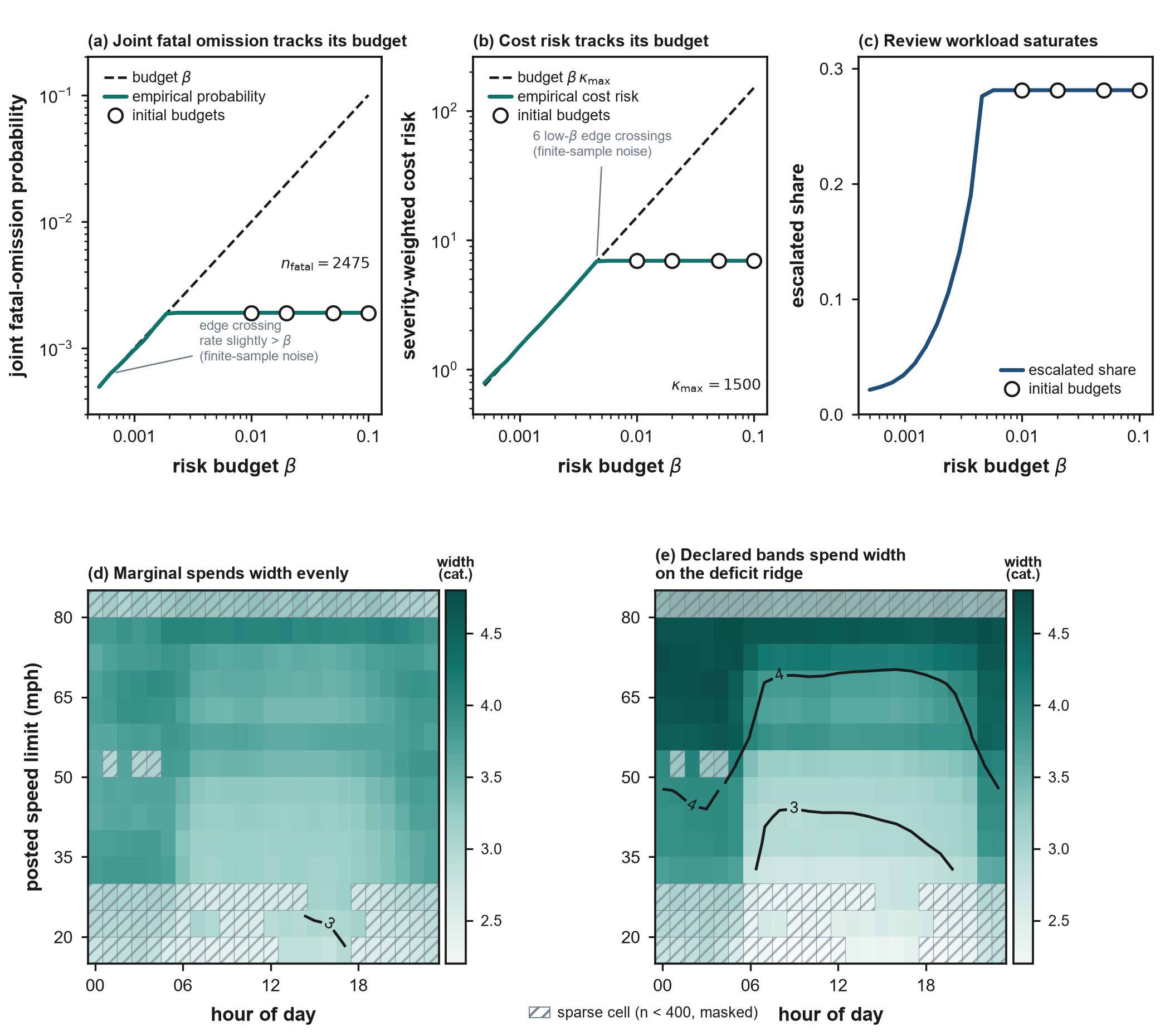}
\caption{Completed-record operating analysis for the gradient boosting anchor. (a) Joint
fatal-omission probability over all 810{,}082 test records, with 2{,}475 fatal records,
against the nominal bound. (b) Cost-weighted false-omission risk against its nominal bound
over a log-spaced $\beta$ grid, with the four budgets of the initial design (0.01, 0.02,
0.05, 0.10) marked; both constraints are slack above about 0.005. (c) The share with a
risk-calibrated lower endpoint of B, A, or K, saturating at 0.281 once the constraint is
slack. (d,e) Mean set width over the plane of Figure~\ref{fig:landscape} under marginal and
speed-hour calibration, where certification moves width into the cells in which
Figure~\ref{fig:repair}(a) showed coverage falling short. The curve values for panels
(a)--(c) are included in the reproducibility supplement.}
\label{fig:risk}
\end{figure}

On this single test-set realization the empirical risk exceeds its bound at a few budgets.
The cost bound is exceeded at the five smallest budgets, by at most 4.8\% relative, and at
one budget near 0.0046, by 0.4\%; the fatal bound is exceeded at one budget, by 1\%. The
guarantees control an expectation over calibration draws rather than each realization. In
100 repeated splits of the kind described at the start of this section, the mean joint fatal
omission matches its bound at every budget. It exceeds the bound by at most 0.000003 (0.5\%
relative, less than one Monte Carlo standard error) and lies below it elsewhere. The
single-split excursions are therefore consistent with sampling variation around a valid
guarantee.

Panels (d) and (e) show the width cost of the guarantee across the operating plane.
Certification reallocates width toward cells with coverage deficits. In cells where marginal
coverage fell below 0.80 the speed-hour partition spends an additional 0.61 categories, and
in cells where marginal coverage was already adequate it spends 0.18 categories fewer.

\subsection{Composition and cost}
\label{sec:e8e9}

The composition audit is a semi-synthetic sensitivity calculation. Observed CRIS KABCO is
the proxy outcome on both folds. Unlike the sensitivity grid of Section~\ref{sec:e4},
synthetic reported labels are generated in the calibration fold only, and test coverage is
evaluated against test-fold KABCO. Table~\ref{tbl:compose} reports one row for each of the
seven final cells of the product partition in Section~\ref{sec:splitting}, for the gradient
boosting anchor and for DLCON, a narrow model. Sets are expanded with the category-dependent
map of Remark~\ref{rem:catdep}, for which the declared $\delta_a=0.02$ is supported in every
cell. Weighting the cited beyond-band rates by each cell's reported composition gives
$\delta$ of 0.006 (baseline), 0.007 (rural high-speed), 0.009 (unrestrained), and 0.011
(motorcycle). Under the constant one-category band the same calculation gives 0.103 for
motorcyclists and 0.069 for unrestrained drivers, because the band then leaves two-level
over-reporting from A outside; the declared 0.02 is not supported there, and the floor would
fall to about 0.80 and 0.83. Both maps give the same composed sets to four decimals.

\begin{table}[pos=htbp]
\caption{Semi-synthetic composition audit by final calibration cell for a wide and a narrow base model ($\alpha=0.10$, $\delta=0.02$).}
\label{tbl:compose}
\centering\footnotesize
\begin{threeparttable}
\begin{tabular*}{\tblwidth}{@{\extracolsep{\fill}}lrrrrrrr@{}}
\toprule
 & & \multicolumn{3}{c}{Gradient boosting (balanced)} & \multicolumn{3}{c}{DLCON} \\
\cmidrule(lr){3-5}\cmidrule(lr){6-8}
Final calibration cell & $n_{\rm cal}$ & Coverage & Width & Full (\%) & Coverage & Width & Full (\%) \\
\midrule
Baseline, rural & 106,245 & 0.9950 & 4.32 & 39.7 & 0.9834 & 2.65 & 0.3 \\
Baseline, urban & 591,520 & 0.9976 & 4.24 & 29.4 & 0.9851 & 2.71 & 0.0 \\
Motorcycle, rural & 2,787 & 1.0000 & 5.00 & 100.0 & 0.9986 & 4.97 & 97.5 \\
Motorcycle, urban & 5,599 & 1.0000 & 5.00 & 100.0 & 0.9995 & 4.94 & 93.6 \\
Rural high-speed, rural & 92,187 & 0.9963 & 4.99 & 99.3 & 0.9759 & 3.13 & 0.6 \\
Unrestrained, rural & 4,673 & 0.9998 & 5.00 & 99.9 & 0.9941 & 4.77 & 82.6 \\
Unrestrained, urban & 5,583 & 0.9998 & 5.00 & 99.9 & 0.9855 & 4.35 & 49.9 \\
Overall & 808,594 & 0.9971 & 4.35 & 40.3 & 0.9840 & 2.80 & 2.0 \\
\bottomrule
\end{tabular*}
\begin{tablenotes}[flushleft]\footnotesize
\item[] \textit{Note:} Rows are the seven non-empty cells of the product of the four declared strata and the rural or urban indicator; no cell required rollup. Sets are expanded with the category-dependent map of Remark~\ref{rem:catdep}; the constant one-category band gives the same values to four decimals. Observed CRIS KABCO is the proxy outcome; synthetic reported labels are generated on calibration rows only, and test coverage is evaluated against test-fold KABCO. Full (\%) is the share of records whose set spans all five categories. Every cell meets the declared 0.88 floor.
\end{tablenotes}
\end{threeparttable}
\end{table}

The table separates what the base model contributes from what the strata contribute. With
DLCON, the baseline and rural high-speed cells receive composed sets of 2.65 to 3.13
categories, and fewer than 1\% of those records receive the full scale. With the gradient
boosting anchor the same cells receive 4.24 to 4.99 categories, and the rural high-speed cell
is almost entirely full scale. On the motorcyclist and unrestrained cells both models give
wide composed sets, with 94 to 100\% of motorcyclists and 50 to 100\% of unrestrained drivers
at full scale. Near-vacuous composed sets on the rural high-speed cell are therefore a
property of the anchor model, whereas on the motorcyclist and unrestrained cells they occur
for both models, because the base sets there are already wide (3.55 to 4.58 categories in
Table~\ref{tab:maincells}) and a unit expansion in each direction reaches the full scale.
Coverage of 1.000 in those cells reflects slack in the expansion rather than a useful
trade-off. Narrower composed sets would need a smaller declared band where the linkage
evidence permits it (Remark~\ref{rem:bandminimal}), a lower nominal level, or a base model
that separates these strata better.

How much of the width follows from a declared reporting channel is a quantitative question.
For any fixed channel and latent composition, Theorem~\ref{thm:channel} lower-bounds the
mean width of every valid reported-label predictor on stratum $A$. The empirical analysis
evaluates this quantity on the finite channel grid defined in
Remark~\ref{rem:channelscope}. For each fixed scenario, the latent composition varies over
the simplex while reproducing each reported KABCO share within its stated rounding
interval. The exact floor is obtained from 25 linear programs that enumerate the possible
fractional-knapsack boundary cells. Table~\ref{tbl:channelfloor} reports, for each stratum and
selection index $\varrho$, the smallest exact floor across the feasible scenarios and the number
of feasible scenarios, and Figure~\ref{fig:channelfloor} shows the corresponding conditional
decomposition. These values are exact for the declared finite scenario set. They are not
lower endpoints for a continuous class of channels or selection functions.

\begin{table}[pos=htbp]
\centering
\caption{Scenario-implied channel floor $\mathcal N^{\ast}_A$ against achieved width by
stratum at $\alpha=0.10$.}
\label{tbl:channelfloor}
\footnotesize
\begin{threeparttable}
\begin{tabular}{lcccc}
\toprule
Stratum & Achieved base width & $\varrho=1$ & $\varrho=2$ & $\varrho=5$ \\
\midrule
Baseline          & $1.40$ & $0.97$ (131) & $0.93$ (246) & $0.91$ (270) \\
Rural high-speed  & $1.95$ & $1.00$ (168) & $0.95$ (270) & $0.92$ (540) \\
Unrestrained      & $3.58$ & $1.36$ (634) & $1.20$ (810) & $0.96$ (705) \\
Motorcyclist      & $3.55$ & $1.64$ (755) & $1.48$ (735) & $1.03$ (705) \\
\bottomrule
\end{tabular}
\begin{tablenotes}[flushleft]\footnotesize
\item[] \textit{Note:} Each entry is the exact minimum floor over the feasible members of the
declared grid of 900 channel scenarios, with the number of feasible scenarios in
parentheses, at three reported-label selection indices $\varrho$; $\varrho=1$ represents
selection invariance. The achieved width is the smallest certified width among the seven
models (Table~\ref{tab:maincells}). Floors at or below one are vacuous for non-empty sets.
\end{tablenotes}
\end{threeparttable}
\end{table}

\begin{figure}[pos=htbp]
\centering
\includegraphics[width=.92\linewidth,height=.58\textheight,keepaspectratio]{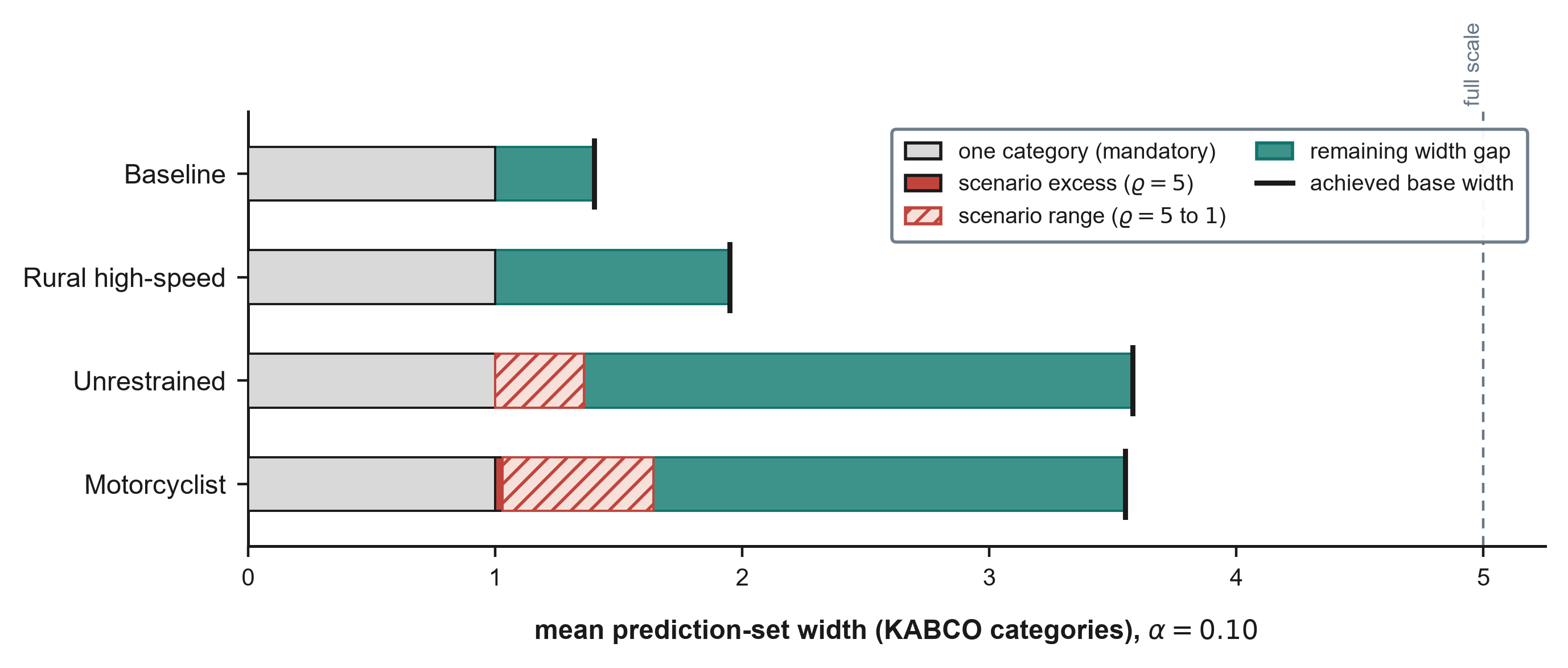}
\caption{Conditional width decomposition for the declared finite reporting-channel scenarios at
$\alpha=0.10$. Gray is the single category carried by any nonempty set. Coral shows the
scenario-implied excess $\max(0,\mathcal N^{\ast}_A-1)$, with the solid segment for
$\varrho=5$ and the hatched segment extending to the selection-invariance scenario. Teal is
the remaining gap to the smallest achieved base width. Each decomposition is conditional on
its declared channel and latent composition. Each floor is the exact minimum over feasible
members of the declared finite grid and is not a continuous-channel endpoint.}
\label{fig:channelfloor}
\end{figure}

The first reading is a scenario specification check. Every achieved width lies above the
smallest floor in its row, so at least one feasible scenario survives in every stratum.
Because the table reports minima, this comparison cannot show that every scenario survives;
scenarios with higher floors than those reported could still be refuted, and a refutation
would also need to allow for sampling error in the achieved width. Floors at or below one,
which include every baseline and rural high-speed entry, are vacuous for non-empty sets.
Achieved motorcyclist coverage across the seven base models runs from 0.8984 to 0.9015, and
recomputing the floor at the realized level instead of $1-\alpha$ changes it by about
$10^{-3}$ of a category.

The second reading is the conditional decomposition in
Corollary~\ref{cor:channelvacuity}. Under a retained scenario, one category is required by
any nonempty set, $\max(0,\mathcal N^{\ast}_A-1)$ is the scenario-implied channel excess, and
$w^{\mathrm{ach}}_A-\max(1,\mathcal N^{\ast}_A)$ is an upper bound on the width that a better
model or richer record could recover. On the motorcyclist stratum, whose achieved width is
about $3.55$, the scenario-implied excess is $0.64$ under selection invariance and $0.03$
under the evaluated fivefold-selection scenario, so the declared channels explain at most
0.64 of the roughly 2.5 categories above one. The corresponding conditional upper bound on
recoverable width ranges from $1.91$ to $2.52$. The unrestrained selection-invariance
scenario implies an excess of $0.36$. These decompositions require the fixed channel,
latent composition, and within-stratum nondifferentiality stated by
Theorem~\ref{thm:channel}; the last is doubtful for the unrestrained stratum, whose defining
field is coded by the same officer as the outcome (Section~\ref{sec:features}).

The law-dependent floor of Theorem~\ref{thm:levelset} can also be estimated by plugging in
each model's fitted probabilities. At $\alpha=0.10$ the plug-in level-set size ranges from
1.25 to 1.51 categories in the baseline cell and from 3.38 to 3.76 in the motorcyclist cell
across the six narrow models. These plug-in values depend on the fitted model and are not
certificates (Remark~\ref{rem:levelsetscope}). Along the informativeness frontier of
Corollary~\ref{cor:frontier}, computed for the gradient boosting anchor, only the baseline cell reaches an even chance of a composed
set that excludes the fatal category, at $\alpha=0.05$; no safety cell reaches it for any
$\alpha$ up to 0.20 under either band.

The width contribution from finalized-record features is also measurable. Configuration T
adds ejection, airbag deployment, vehicle damage severity, the first harmful event, and the
manner of collision. Re-running class-conditional calibration narrows the intervals by a
stratum-dependent amount while preserving nominal cell coverage. On the unrestrained
stratum, these fields reduce ordered-logit mean width from 3.77 to 3.19 because ejection
separates severity strongly. On the motorcyclist stratum, mean width remains near 3.6 because
enclosed-vehicle occupant fields provide little added separation. These fields are coded by
the same officer who assigns KABCO, so part of the gain may reflect shared reporting rather
than injury information. The contrast suggests helmet use as a field worth adding in future
analyses of motorcycle injury records; it was not tested here.

Computational cost is the last question the experiments address. The certification step is
a sort and a quantile lookup, so it is $O(n\log n)$ with a small constant. Measured on
synthetic conditional distributions at increasing scale, on a 24-core desktop processor
(Intel Core i9-13900K, 3.0\,GHz base, 32\,GB) with the machine otherwise idle, this step
costs \RUNTIMEONEM{} at one million rows and \RUNTIMEFIVEM{} at five million, and the
measured times scale approximately linearly from $10^5$ to $5\times10^6$ rows, with an
endpoint departure of about 9\% from the linear projection. Base models take seconds to
24 minutes to fit here (Table~\ref{tbl:models}), and the foundation model's prediction pass
takes about 25 minutes. The figure of about one second covers the threshold step only; the
density-ratio fit, the risk-control search, monitoring, and the channel-floor programs add
further cost, each of which is small relative to model fitting.

\section{Discussion}
\label{sec:discussion}

The Texas CRIS evaluation shows why an aggregate coverage result is not enough for
crash-injury analysis. Every base model reaches about 0.90 overall under one pooled
threshold, yet every model covers motorcyclists or unrestrained drivers at 0.868 or lower,
and the class-balanced gradient boosting model covers unrestrained drivers at 0.374. These
groups carry a large share of the severe outcomes. Calibration within prespecified safety
cells applies an established groupwise conformal tool and restores coverage near 0.90 in
every cell for every model, which the alternative conformal methods in
Table~\ref{tab:baselines} do not achieve. The audit therefore moves from a single statewide
average to coverage and width results for each safety population.

The seven-model comparison separates model quality from certificate validity. Coverage is
set by the calibration step, while the fitted probability distribution determines most of
the efficiency. After certification six of the seven models give mean widths between 1.52
and 1.55 categories and similar ranked probability scores, and a classical ordered logit is
within 0.03 categories of the narrowest model and narrower than the tabular foundation
model, most clearly for motorcyclists and unrestrained drivers. On these records, then, an
agency does not need an opaque model to obtain a guarantee, and it can compare injury models
on width after applying a common coverage requirement. The comparison has limits, because the models
were fitted with different subsample sizes, and TabPFN used a fixed stratified context, so
the ranking among the six narrow models is not settled here.

The guarantee is informative for most drivers and much less so where it matters most.
Certified sets for motorcyclists and unrestrained drivers span 3.6 to 4.6 of the five
categories, and after the declared compatibility expansion most of them cover the whole
scale for both the wide and the narrow base model. Contiguous sets still provide lower and
upper severity endpoints, and the expansion turns a statement about police-reported KABCO
into a sensitivity statement about the underlying injury. But near-full-scale sets for these
groups mean that the recorded fields separate their severity poorly. The channel-floor
analysis shows that a declared reporting process explains only a small part of this width,
so the larger part is a feature-information problem rather than a calibration problem.

The empirical design supports a reproducible workflow for retrospective crash-record
audits. The agency first defines the record version and the sampled-driver unit, then
freezes the training, calibration, and test split at the crash level. It next declares the
reporting configuration and the final-cell hierarchy before any calibration outcome is
used. The resulting certificate reports coverage, simultaneous uncertainty intervals, width,
and full-scale frequency for every final cell. The choice of cells determines which safety
populations receive a separate guarantee. A generic learned partition restores coverage in
its own cells but does not resolve the speed-hour ridge that it does not encode, and the
speed-hour partition repairs the ridge while moving part of the deficit to its band edges.
A safety population that needs separate protection must therefore appear in the frozen
partition or in its rollup hierarchy, and cell boundaries deserve the same scrutiny as the
model.

The endpoint and risk audits give three practical cautions for agencies that want to use
the sets for review. First, a coverage guarantee over all records does not make the upper
endpoint a good severe-injury screen. Under narrow models the ``upper endpoint A or K'' rule
routes about 2\% of records to review but misses about 65\% of observed A-or-K outcomes;
under the wide anchor it routes 41\% and misses 14\%, and the expanded sets route 96\% of
records. Second, the fatal-omission bound controls a joint probability. Budgets above the
fatal share of the population (0.0031 here) leave the constraint slack, and at the planned
budgets 62\% of fatal records fell outside the risk-calibrated sets. A useful fatal budget
must lie well below the fatal share, and the conditional omission rate should be reported
next to the joint bound. Third, the risk-workload curve in Figure~\ref{fig:risk} lets an
agency choose an operating point and keep the chosen threshold as an auditable record, but
that choice should be made with the omission rate in view.

The compatibility map is a declared sensitivity input. External KABCO linkage evidence
supplies plausible reporting patterns \citep{burdett2015,burdett2022,taylor2024}, and the
semi-synthetic grid shows how the declared inputs change the coverage floor and set width
under a known injected process; it is not evidence about the Texas reporting process. The
per-cell arithmetic in Section~\ref{sec:e8e9} also shows that the declared beyond-band mass
depends on the map. Under the category-dependent map, $\delta=0.02$ is supported in every
cell, whereas under the constant band it is not for motorcyclists and unrestrained drivers.
The map therefore has to be declared per cell with the cell's reported composition in view.

The temporal and county analyses examine how a frozen certificate behaves when the record
distribution changes. In the temporal stress test, the eight final-cell estimates range from
0.8919 to 0.9094, and three intervals lie entirely above 0.90; the monthly series and the
p-value quantiles also depart from the calibration law by more than sampling error, even
though average coverage stays near nominal. With the ratio fit on a separate half of the
calibration fold, density-ratio weighting gives 0.9002 and 0.9053, close to the unweighted
coverage in these years. The county analysis shows much greater variation. A single
statewide threshold covers 42 of the 54 held-out counties at less than 0.80, including 18 of
the 29 larger ones, and the lowest values occur in small rural counties. For the larger
counties density-ratio weighting brings coverage to 0.872 to 0.922 at about 0.43 more
categories of width. The small rural counties, which fail most, have too few records for
either weighting or a separate calibration cell. A statewide threshold should therefore not
be applied to them without local calibration data, for example by pooling several years or
neighboring counties into a declared cell.

The monitoring score flagged three months for review, and the flags survived a post hoc
month-matched check. Because the score has no calibrated false-alarm rate, these flags start
review and recalibration rather than establish a shift. Together with the coverage series,
they give an auditable record of how the frozen certificate behaved after the split.

General-purpose conformal libraries provide a strong marginal-coverage baseline.
Table~\ref{tbl:generic} compares CHOIR with MAPIE and crepes on the synthetic demonstration
data shipped with the package; all three reach similar marginal coverage and set size, and
the CHOIR value of 0.895 lies within one binomial standard error of the target. Some of these
libraries also support Mondrian calibration, so the difference that matters is not grouping
as such but the combination of ordinal contiguous sets, declared strata frozen on training
data, the true-label sensitivity transfer, and fatal-omission control, each tied to its
stated assumption. The comparison on the real data is Table~\ref{tab:baselines}.

\begin{table}[pos=htbp]
\caption{Guarantees available from generic conformal tooling on synthetic demonstration data.}
\label{tbl:generic}
\centering\footnotesize
\begin{threeparttable}
\begin{tabular*}{\tblwidth}{@{\extracolsep{\fill}}lccccc@{}}
\toprule
Method & Coverage & Set size & Contiguous & True-label & Fatal-omission \\
\midrule
CHOIR  & 0.895 & 2.19 & by construction & declared map & fatal premise \\
MAPIE  & 0.899 & 2.22 & 99\% & no & no \\
crepes & 0.899 & 2.22 & 99\% & no & no \\
\bottomrule
\end{tabular*}
\begin{tablenotes}[flushleft]\footnotesize
\item[] \textit{Note:} The comparison uses the synthetic demonstration data shipped with the package; the aggregate output is in the \texttt{benchmarks} folder of the public repository. The binomial standard error is about 0.007 at this sample size. The last three columns describe properties of each method by design rather than tested results. True-label transfer requires the declared compatibility relation in Assumption~\ref{ass:noise}. Fatal-omission control requires $Y=5\Rightarrow\yt=5$ as stated in Corollary~\ref{cor:fatal}. The Texas CRIS evaluation in Section~\ref{sec:results} uses the 4.04 million crashes of the primary sample.
\end{tablenotes}
\end{threeparttable}
\end{table}

The contribution is a certification framework that combines established conformal
components with the structure of crash-injury records. Split conformal calibration, Mondrian
conditioning, weighted conformal prediction, and conformal risk control keep their
established attribution \citep{vovk2005,lei2018,tibshirani2019,angelopoulos2024}. The
specialized and derived results add ordinal reporting-map transfer and its sharpness, the
reporting-channel width floor, heterogeneity-specific efficiency, and severity-weighted risk
under the declared reporting structure.

For transportation safety practice, the results support three conclusions. Aggregate
predictive performance should be reported together with coverage within prespecified safety
strata, because the most safety-relevant groups can behave differently from the statewide
average. Uncertainty methods should report their information cost through set width and
full-scale frequency, and omission rates for severe outcomes with the correct denominator.
Transfer to other counties and years should be checked with local data rather than assumed.
The extended record configuration also suggests where information is missing. Ejection,
airbag deployment, vehicle damage severity, first harmful event, and manner of collision
reduce the ordered-logit width for unrestrained drivers from 3.77 to 3.19 categories, though
these fields share the officer's reporting process, and motorcyclist width stays near 3.6.
Helmet use and motorcycle-specific injury information are therefore worth testing in future
analyses. The certification step itself adds little computation relative to model fitting,
with measured times of \RUNTIMEONEM{} at one million rows and \RUNTIMEFIVEM{} at five
million rows, so repeated audits across model libraries and safety strata are practical.

\section{Conclusion}
\label{sec:conclusion}

This study presents CHOIR, a certification layer for police-reported KABCO injury
prediction at the level of the sampled driver. The layer uses established split-conformal
calibration, groupwise calibration, weighted conformal prediction, and conformal risk
control, and it adds the crash-specific structure needed to produce contiguous ordinal
intervals, use a declared reporting band for sensitivity analysis on medical injury, keep
observed-cell validity separate from transfer diagnostics, and bound injury-weighted
omission risk. The mathematical results characterize the certified statements and, under a
declared reporting channel and within-stratum nondifferential reporting, a lower bound on the
width of any valid predictor.

The Texas evaluation, with 4.04 million crashes and seven base models from the ordered
logit to a tabular foundation model, answers the three research questions. For RQ1, one
pooled threshold reaches about 0.90 overall for every model but covers motorcyclists or
unrestrained drivers at 0.868 or lower, whereas calibration within four prespecified strata
places all 28 model-by-stratum estimates between 0.8983 and 0.9074, with every
Bonferroni-adjusted 95\% interval containing 0.90 and stable results across 200 repeated
calibration splits. The guarantee costs width, since certified sets span 3.6 to 4.6 of five
categories for motorcyclists and unrestrained drivers, and once certified a classical
ordered logit is within 0.03 categories of the narrowest model. For RQ2, a declared ordinal
band transfers coverage to the medically referenced injury at a sharp additive cost, and the
declared reporting channels explain only a small part of the width on the high-risk strata,
so most of that width reflects limited information in the recorded fields. For RQ3, a
frozen statewide threshold covers most small rural held-out counties at less than 0.80, and
fatal-omission control is informative only when the risk budget lies below the fatal share
of the population.

This study has several limitations. The underlying medical injury is not observed in CRIS, so
the label-noise and composition analyses are semi-synthetic stress tests based on declared
reporting bands and injected noise, and the injected noise is milder than the disagreement
reported in linkage studies. The base models are fit once, so the repeated-split analysis measures variation across
calibration draws but not across training draws. The certified population is
crash-weighted drivers with a recorded KABCO category; passengers, non-motorists, and drivers
without a valid code are excluded. The unrestrained cell and the Configuration T fields are
coded by the same officer who records the outcome. The exact-K premise, the declared
compatibility maps, and within-stratum nondifferential reporting are assumptions that the
Texas records do not validate. Several downstream audits use a single wide base model, and
the base models were fitted with different training sizes. The temporal, county, and monitoring analyses do not
satisfy the exchangeability assumption and are descriptive, and the guarantee does not cover
use at the time of a crash report. The channel-floor values are exact over the declared finite
scenario set rather than over a continuous class of channels. Finally, the records represent
one state and one reporting regime. Within this scope, and under the stated exchangeability
assumption, the finite-sample coverage results hold.

Each limitation also points to an extension. A Texas linkage study connecting CRIS records
with trauma-registry outcomes could measure the reporting band directly and provide
category-dependent sensitivity inputs. Refitting the base models on repeated training draws would extend the repeated-split
analysis to model-fitting variation. Labeled samples from target counties could support local calibration and
two-sided transfer studies, which matter most for small rural counties. Features available at
the time of a crash report would allow a forward-use evaluation. Evaluation in other
jurisdictions and on other ordered safety outcomes recorded by human observers would test
how well the certification layer carries over beyond a single state and reporting regime.

\section*{Appendix}
\appendix
\setcounter{equation}{0}
\renewcommand{\theequation}{A.\arabic{equation}}
\section{Deferred proofs}
\label{app:proofs}

\begin{proof}[Proof of Theorem~\ref{thm:oracle}]
(i) Class-$c$ raw-set coverage at threshold $\lambda_c$ is $G_c(\lambda_c)$; by (R1), $G_c(\lambda_c)\ge 1-\alpha$ if and only if $\lambda_c\ge q_c$. Expected raw-set size $\E[|C_\lambda(X)|\mid\hat c=c]=\sum_k\Prob(s(X,k)\le\lambda\mid\hat c=c)$ is non-decreasing in $\lambda$. Thus $q_c$ is the componentwise-smallest valid threshold and is an expected-size minimizer. The statement does not concern the deployed fallback family.
(ii) $\hat q_c$ is the $\lceil(1-\alpha)(n_c+1)\rceil$-th order statistic of the within-class sample (i.i.d.\ in the i.i.d.\ case we invoke); convergence and the rate are classical \citep[Cor.~21.5]{vandervaart1998}. Size convergence follows from
\begin{equation}\label{eq:size-conv}
\left|\E[|C_{\hat q_c}(X)|\mid\hat c=c,\hat q_c]-\E[|C_{q_c}(X)|\mid\hat c=c]\right|
\le\sum_k\Prob\bigl(q_c\wedge\hat q_c<s(X,k)\le q_c\vee\hat q_c\,\bigm|\,\hat c=c,\hat q_c\bigr)\ \xrightarrow{p}\ 0
\end{equation}
by (R2), independence of a fresh $X$ from the calibration quantile, and $\hat q_c\xrightarrow{p}q_c$. Summing over the finite class set gives convergence in probability of the conditional-on-calibration expected raw-set size to $\mathcal S^\star$.
(iii) The marginal empirical quantile converges to $q_{\mathrm{mix}}$ by (R1); class-$c$ coverage of the fixed threshold $q_{\mathrm{mix}}$ is $G_c(q_{\mathrm{mix}})$; monotonicity gives the dichotomy, and $\sum_c p_cG_c(q_{\mathrm{mix}})=1-\alpha$ with all summands $\ge 1-\alpha$ forces equality iff all $q_c$ coincide.
\end{proof}

\begin{proof}[Proof of Theorem~\ref{thm:sharp}]
(i) Take $X$ degenerate and $m=1$, so that the band around $m$ is $[1,1+b^-]$ and $K\ge b^-+2$ leaves at least one category above it. For $0<\varepsilon<\min\{\alpha,\eta\}$ let $\yt$ put mass $1-\alpha+\varepsilon$ on $m$ and the remaining mass $\alpha-\varepsilon$ on $K$ (outside the band $[1,1+b^-]$), with $\hat F$ chosen so that $s(x,m)<s(x,k)$ for all $k\ne m$. Then $\Prob(S=s(x,m))=1-\alpha+\varepsilon$, so for large $n$ the calibration quantile $\hat q$ equals $s(x,m)$ with probability $\to 1$ and $\Ct=\{m\}$, whence $\Cop=[1,\,m+b^-]$. Couple $Y$ to $\yt$ so that $Y=\yt$ except on an event of probability $\delta$ contained in $\{\yt=m\}$, where $Y=m+b^-+1$ (possible since $\delta<1-\alpha<1-\alpha+\varepsilon$ and $m+b^-+1\le K$). N($b^-,b^+,\delta$) holds by construction. Coverage then decomposes, in that the unperturbed part of $\{\yt=m\}$ (mass $1-\alpha+\varepsilon-\delta$) is covered; the perturbed event (mass $\delta$) has $Y=m+b^-+1\notin\Cop$; and $\{\yt=K\}$ (mass $\alpha-\varepsilon$) has $Y=K\notin\Cop$. Hence $\Prob(Y\in\Cop)\to 1-\alpha+\varepsilon-\delta<1-\alpha-\delta+\eta$. Since $\eta>0$ is arbitrary, no constant larger than $1-\alpha-\delta$ can hold uniformly over laws satisfying the assumptions.
(ii) \emph{Upper trim.} Take the construction of (i) with $m=1$ and $\delta=0$, and set $Y=\yt+b^-$ a.s.\ on $\{\yt=m\}$ (maximal under-reporting, allowed by N since $Y-\yt=b^-$; feasible as $m+b^-\le K$). Whenever $\Ct=\{m\}$ the trimmed set is $[1,\ m+b^--1]$, and $Y=m+b^-\notin[1,m+b^--1]$. So on $\{\yt=m\}$ (probability $\to 1-\alpha$) the trimmed rule misses the truth; its true-label coverage is at most $\alpha+o(1)$.
\emph{Lower trim.} Symmetrically anchor at $m=K$ (feasible when $K\ge b^++1$, so that $m-b^+\ge 1$) and set $Y=\yt-b^+$ a.s.\ on $\{\yt=m\}$ (maximal over-reporting); the trimmed set $[m-b^++1,\ K]$ never contains $Y=m-b^+$. In both cases only the miss on $\{\yt=m\}$ is needed for the bound, so the placement of the residual mass is immaterial.
\end{proof}

\begin{proof}[Proof of Theorem~\ref{thm:transfer}]
(i) Two steps. \emph{Step 1 (exactness under the estimated tilt).} Condition on the independent calibration-reference and target ratio-fitting samples, so $\hat w$ is fixed and independent of the scored calibration fold. By weighted conformal \citep{tibshirani2019}, calibration with this weight function is valid if the test covariate is drawn from $Q_{\hat w}$ and the conditional law of $\yt\mid X$ matches calibration, giving $P_{X\sim Q_{\hat w}}(\yt\in C_{\hat q_{\hat w}}(X))\ge 1-\alpha$.
\emph{Step 2 (change of test measure).} The coverage event is $E=\{\yt_{n+1}\in C_{\hat q_{\hat w}}(X_{n+1})\}$. Conditionally on calibration data, its probability under the two test-point laws $P^\ast=P^\ast_X\otimes P_{\mathrm{cal}}(\yt\mid X)$ and $Q=Q_{\hat w}\otimes P_{\mathrm{cal}}(\yt\mid X)$ differs by at most $\dtv(P^\ast,Q)$; since the joints share the conditional,
\begin{equation}\label{eq:tv-identity}
\dtv(P^\ast,Q)=\tfrac12\int_{\X}\bigl|\mathrm dP^\ast_X-\mathrm dQ_{\hat w}\bigr|=\dtv(P^\ast_X,Q_{\hat w}).
\end{equation}
Combine and average over calibration data.
(ii) For any fixed $Q$, $2\,\mathrm{ba}(\phi)-1=P^\ast_X(\phi{=}1)-Q(\phi{=}1)\le\sup_A|P^\ast_X(A)-Q(A)|=\dtv(P^\ast_X,Q)$. Independent held-out observations from each class give independent binomial counts conditional on the fitted $\phi$. One-sided Clopper--Pearson bounds and a Bonferroni split therefore give the stated lower confidence bound. Weighted resampling with replacement from the finite calibration reference samples exactly from $\widehat Q_{\hat w,n}$ conditional on that reference, so the implemented bound targets $\dtv(P^\ast_X,\widehat Q_{\hat w,n})$ rather than $\dtv(P^\ast_X,Q_{\hat w})$.
(iii) Under realizability $Q_{w_{\theta_0}}=P^\ast_X$ and
\begin{equation}\label{eq:tv-rate}
\dtv\bigl(Q_{w_{\hat\theta}},Q_{w_{\theta_0}}\bigr)\ \le\ \tfrac12\,\E_{P_{\mathrm{cal}}}\Bigl|\tfrac{w_{\hat\theta}}{\E w_{\hat\theta}}-\tfrac{w_{\theta_0}}{\E w_{\theta_0}}\Bigr|,
\end{equation}
The local Lipschitz envelope and the positive normalizing expectation make the right side $O_p(\|\hat\theta-\theta_0\|)$. The stated likelihood conditions give $\hat\theta-\theta_0=O_p((n_{\mathrm{fit}}\wedge n_{\mathrm{ref}})^{-1/2})$, which proves the rate.
\end{proof}

\section*{Declaration of competing interest}
\addcontentsline{toc}{section}{Declaration of competing interest}

The authors declare that they have no known competing financial interests or personal
relationships that could have appeared to influence the work reported in this paper.

\section*{Data availability}
\addcontentsline{toc}{section}{Data availability}

The police-reported sampled-driver records analyzed in this study are not public. They were
obtained from the Texas Crash Records Information System, maintained by the Texas
Department of Transportation, which distributes record-level extracts to researchers on request through
the CRIS Query portal at \url{https://cris.dot.state.tx.us/}. Access requires a data
request and acceptance of the department's terms of use. The authors' access agreement
excludes redistribution of the records. The vehicle
attributes were decoded from reported vehicle identification numbers using the National
Highway Traffic Safety Administration's public vPIC service
(\url{https://vpic.nhtsa.dot.gov/api/}), and the geographic reference files are public
Census products. All personally identifying material, including vehicle identification
numbers, license and registration identifiers, dates of birth, and free-text investigator
narratives, was excluded at ingest and appears in no released artifact.

The submission reproducibility supplement contains the data-construction scripts and analysis
code used from the departmental files onward. A reader holding an equivalent CRIS extract can
rebuild the sampled-driver snapshot and rerun the empirical pipeline. A synthetic demonstration
dataset ships with the public software package so that the examples, theorem-level tests, and
comparison in the discussion run without restricted data.

\section*{Code availability}
\addcontentsline{toc}{section}{Code availability}

The certification framework is released as an open-source package under the MIT license.
It is distributed as \texttt{choircert} on the Python Package Index and imported as
\texttt{choir}. The public source, documentation, examples, benchmarks, and theorem-level
test suite are available at \url{https://github.com/pozapas/choircert}. The submission
reproducibility supplement contains the empirical experiment harness, aggregate result
artifacts, and scripts that generate the tables and figures in this paper.
It excludes CRIS records and direct identifiers. All reported randomness derives from the
declared seeds recorded in the experiment scripts and audit artifacts.

\section*{Reproducibility statement}
\addcontentsline{toc}{section}{Reproducibility statement}

Results were computed on a frozen analysis snapshot built by the submission pipeline, whose
row order is deterministic. Every cached model artifact carries a fingerprint of the split
it was built from, and the loader refuses any artifact whose fingerprint does not match
the current snapshot, so a stale cache raises an error rather than silently misaligning
with the data. Each number reported in this paper traces to a stored result file in the reproducibility
supplement rather than to a transcribed value, and the tables are generated from those
files by script.
Five statistical and gradient-boosting models were fit on the stated 24-core desktop
processor. The DLCON network and the tabular foundation model were fit and applied on a GPU. The
certification layer itself is CPU-only and its cost is reported in
Section~\ref{sec:e8e9}.

\printcredits

\bibliographystyle{cas-model2-names}
\bibliography{references}

\end{document}